%% file: arxiv.tex
\documentclass[10pt,journal,compsoc]{IEEEtran}

\ifCLASSOPTIONcompsoc
  \usepackage[nocompress]{cite}
\else
  \usepackage{cite}
\fi

\ifCLASSINFOpdf
\else
\fi

\usepackage[T1]{fontenc}
\usepackage[utf8]{inputenc}
\usepackage{times}
\usepackage{textcomp}

\usepackage{array}
\usepackage[caption=false,font=footnotesize]{subfig}
\usepackage{textcomp}
\usepackage{stfloats}
\usepackage{url}

\usepackage{amsmath,amssymb,amsfonts,amsthm,bm,mathtools}

\usepackage[table,dvipsnames]{xcolor}

\usepackage{graphicx}
\usepackage{booktabs,multirow,makecell,adjustbox,tabularx}
\usepackage{placeins}
\usepackage{enumerate}
\usepackage{siunitx}
\usepackage{enumitem}
\newcommand{\na}{\textemdash}
\definecolor{oursrow}{RGB}{235,235,252}

\newcommand{\mc}[1]{\multicolumn{1}{c}{#1}} 

\usepackage[ruled,vlined,linesnumbered]{algorithm2e}

\usepackage{enumitem}
\usepackage{ragged2e}

\usepackage{hyperref}
\usepackage{cleveref}

\definecolor{AcademicBlue}{RGB}{39,76,119}
\definecolor{AcademicTeal}{RGB}{53,122,120}
\definecolor{AcademicGray}{RGB}{90,90,90}
\definecolor{SoftLine}{RGB}{210,214,220}
\definecolor{SoftBlueBg}{RGB}{245,248,252}
\definecolor{SoftTealBg}{RGB}{244,249,248}
\definecolor{SoftGrayBg}{RGB}{247,247,247}

\hypersetup{
    colorlinks=true,
    linkcolor=AcademicBlue,
    citecolor=AcademicTeal,
    urlcolor=AcademicBlue,
    pdftitle={PhaseWin: An Efficient Search Algorithm for Faithful Visual Attribution},
    pdfauthor={Zihan Gu, Junchi Zhang, Li Liu, Xiaochun Cao, Hua Zhang},
    pdfsubject={arXiv Submission},
    pdfkeywords={visual attribution, efficient search, interpretability, submodular optimization}
}

\newtheorem{theorem}{Theorem}
\newtheorem{proposition}{Proposition}
\newtheorem{lemma}{Lemma}
\newtheorem{corollary}{Corollary}
\theoremstyle{definition}
\newtheorem{definition}{Definition}
\newtheorem{assumption}{Assumption}
\theoremstyle{remark}
\newtheorem{remark}{Remark}

\newcommand{\score}{F}

\renewcommand{\arraystretch}{1.12}
\setlist[itemize]{leftmargin=1.5em,itemsep=2pt,topsep=2pt}
\setlist[enumerate]{leftmargin=1.8em,itemsep=2pt,topsep=2pt}

\title{PhaseWin: An Efficient Search Algorithm for Faithful Visual Attribution}

\author{Zihan Gu,
    Junchi Zhang, 
    Li Liu, 
    Xiaochun Cao, 
    and Hua Zhang
\thanks{
    Zihan Gu, and Hua Zhang are with the Institute of Information Engineering, Chinese Academy of Sciences, Beijing 100093, China, and also with the School of Cyber Security, University of Chinese Academy of Sciences, Beijing 100049, China (Email: 
    \href{mailto:guzihan@iie.ac.cn}{guzihan@iie.ac.cn}, \href{mailto:chenruoyu@iie.ac.cn}{chenruoyu@iie.ac.cn}, \href{mailto:zhanghua@iie.ac.cn}{zhanghua@iie.ac.cn}).\\
    Junchi Zhang is with the Shanghai Center for Mathematical Sciences, Fudan University, Shanghai 200438, China (Email: \href{mailto:jczhang24@m.fudan.edu.cn}{jczhang24@m.fudan.edu.cn}).\\
    Li Liu is with the College of Electronic Science and Technology, National University of Defense Technology, Changsha 410073, China (Email: \href{mailto:li.liu@oulu.fi}{li.liu@oulu.fi}). \\
    Xiaochun Cao is with the School of Cyber Science and Technology, Shenzhen Campus of Sun Yat-sen University, Shenzhen 518107, China (Email: \href{mailto:caoxiaochun@mail.sysu.edu.cn}{caoxiaochun@mail.sysu.edu.cn}).}
}

\begin{document}

\IEEEtitleabstractindextext{%
\begin{abstract}
\justifying
Visual attribution is a fundamental tool for interpreting modern vision and vision-language models, particularly when their decisions must be inspected, diagnosed, or audited. Its goal is to explain how a model's decision depends on local regions of the visual input, typically by assigning an importance ordering over candidate image regions. Given an image partitioned into $n$ regions, faithful attribution can be cast as an ordered subset-search problem, in which progressively inserting the selected regions should recover the target model response as early as possible. Exhaustive search over region subsets incurs exponential cost, while the widely used greedy search still requires a quadratic number of model evaluations, because every selection step rescores all remaining candidates.
We propose PhaseWin, an efficient subset-search algorithm for faithful visual attribution. PhaseWin reorganizes greedy region selection into a phased window-search procedure: rather than re-evaluating the full candidate set at every step, it alternates between global candidate screening, adaptive pruning, and localized window refinement, while preserving the essential region-ranking behavior of greedy search.
We analyze PhaseWin under monotone evidence-accumulation conditions and show that, under feature-level structural assumptions, it attains controllable linear evaluation complexity together with near-greedy faithfulness guarantees. Extensive experiments on image classification, object detection, visual grounding, and image captioning show that, among all compared attribution methods, PhaseWin reaches high faithfulness with the fewest forward passes, empirically realizing the predicted reduction from $O(n^2)$ to $O(n)$. The code is available at \url{https://github.com/Qihuai27/phasewin-va}.
\end{abstract}

}

\maketitle


\section{Introduction}\label{sec:intro}

\IEEEPARstart{V}{ision} and vision--language models have grown increasingly capable, yet their predictions are produced from complex visual inputs while the underlying evidence remains implicit. A model may correctly classify an object, localize a referred region, or generate a caption, but the image regions that actually support the response cannot be read off from the output alone. Visual attribution addresses this gap by identifying the parts of an input responsible for a target model response~\cite{deng2026attribution}. It has consequently become an important tool for interpreting modern visual systems~\cite{dwivedi2023explainable,nazir2025survey,feng2021review}: diagnosing model behavior and analyzing failure cases~\cite{gao2024going,yang2026can}, detecting spurious correlations and inspecting bias~\cite{chen2026not}, and supporting safetyauditing~\cite{wilson2023safe,liang2025safemobile}. Because these models are increasingly deployed in high-stakes and large-scale settings, attribution methods must be not only faithful to the model's decision process but also computationally practical.

\begin{figure*}[t]
    \centering
    \vspace{-2pt}
    \includegraphics[width=\textwidth]{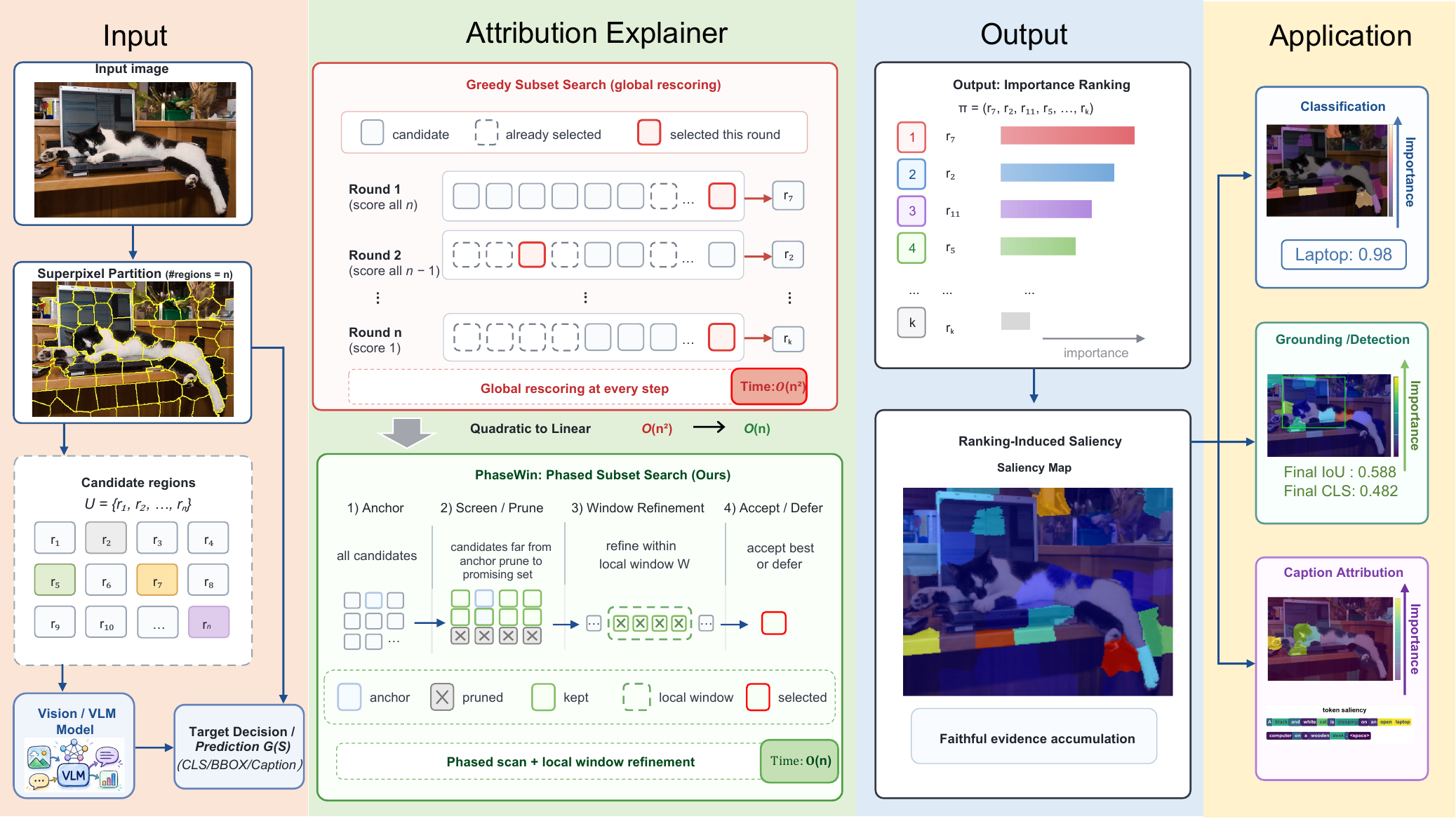}
    \caption{\textbf{Overview of PhaseWin for efficient high-faithfulness visual attribution.}
    From left to right, an input image is partitioned into candidate regions, and a 
    target response is obtained from a vision or vision--language model. Conventional 
    greedy subset search repeatedly rescores all remaining candidates and therefore 
    incurs quadratic evaluation cost. PhaseWin replaces exhaustive global rescoring 
    with phased subset search, including anchoring, screening and pruning, local 
    window refinement, and accept-or-defer decisions. The resulting importance ranking 
    induces a saliency map that can be used for classification, grounding or detection, 
    and caption attribution.}
    \label{fig:intro}
    \vspace{-14pt}
\end{figure*}

A broad class of faithful attribution methods can be cast as a search over visual 
evidence. Given an image partitioned into candidate regions, the goal is to order 
these regions so that progressively inserting highly ranked regions rapidly recovers 
the target model response, while removing them causes the response to drop. This 
ordered-region perspective underlies many perturbation-based~\cite{petsiuk2018rise,
petsiuk2021black} and search-based~\cite{chen2025less} attribution methods. Unlike 
purely gradient-based explanations~\cite{yamauchi2022spatial,yamauchi2024spatial}, 
such methods probe the model directly under controlled visual perturbations and often 
yield stronger faithfulness, especially for black-box or weakly accessible models. 
Their central drawback is cost: faithful evidence search typically demands a large 
number of model evaluations.

Recent greedy-style attribution methods have shown strong empirical performance under 
this search formulation. Methods such as LIMA~\cite{chen2024less,chen2025less}, VPS~
\cite{chen2025vps} rank image regions by optimizing a 
sufficiency--necessity proxy~\cite{sun2023explain}: a selected region set should be sufficient to recover 
the target response, and it should also be necessary in the sense that removing it 
from the image reduces the response. This criterion is natural for attribution because 
it captures two complementary aspects of visual evidence. A truly important region 
should not only support the prediction when present, but also affect the prediction 
when absent. Greedy search is therefore an appealing strategy, since it builds an 
explanation by repeatedly selecting the region that appears most informative under 
the current partial evidence set.

However, previous work~\cite{sun2023explain,chen2025less,chen2025vps} has two fundamental limitations. The first limitation is 
theoretical. Greedy attribution search is often motivated by analogies to submodular 
maximization, where classical greedy algorithms enjoy approximation guarantees. We 
show that this justification does not apply to the standard sufficiency--necessity 
proxy used in visual attribution. In particular, except for degenerate modular cases, 
this proxy cannot be a submodular set function. Intuitively, the proxy couples two 
complementary requirements---recovering the response from the selected set and 
destroying the response by removing that set. Imposing global submodularity on such a 
proxy would force a structure that is too restrictive to model nonlinear visual 
evidence interactions. Consequently, the classical submodular-greedy guarantee cannot 
serve as the theoretical foundation for existing greedy-style attribution methods. 
This does not imply that greedy attribution is empirically ineffective; rather, it 
shows that its effectiveness requires a different and more appropriate explanation.

The second limitation is computational. Standard greedy search performs global 
rescoring at every selection step. After one region is selected, all remaining 
regions are evaluated again to determine the next region. If an image is partitioned 
into \(n\) candidate regions, this procedure requires a quadratic number of expensive 
model evaluations. The cost becomes especially restrictive for large vision--language 
models, black-box APIs, high-resolution images, and dense superpixel partitions. As a 
result, faithful attribution faces a persistent trade-off: exhaustive search can 
produce high-quality explanations, but it is expensive; cheaper alternatives are more 
scalable, but often lose part of the faithfulness that makes perturbation-based 
attribution useful.

These two limitations suggest that a satisfactory solution should do more than simply 
speed up an existing implementation. It should first provide a theory that matches the 
actual attribution proxy being optimized, and then design an algorithm whose efficiency 
does not come at the expense of the evidence ordering that makes greedy search 
faithful. To solve this, we reformulate visual attribution 
as ordered evidence accumulation over an image partition. Instead of assuming a global 
submodular objective, we analyze the sufficiency--necessity proxy under explicit 
evidence-accumulation conditions. This yields a theoretical basis for understanding 
when greedy-style region ordering is meaningful and when a faster search procedure can 
preserve its key faithfulness properties.

Built on this formulation, we propose \emph{PhaseWin}, a phased window search 
algorithm for efficient faithful visual attribution. PhaseWin avoids exhaustive global 
rescoring by organizing the search into repeated phases. Each phase first performs a 
coarse global screening to identify promising candidates, prunes clearly uninformative 
regions, refines the most relevant candidates within a local window, and then either 
accepts the best region or defers ambiguous candidates to later phases. In this way, 
the algorithm concentrates expensive evaluations on regions that are most likely to 
change the target response, rather than repeatedly comparing every remaining region 
against every other region. Figure~\ref{fig:intro} summarizes this pipeline: PhaseWin 
takes the same input partition and target response as greedy attribution, but replaces 
quadratic global rescoring with phased subset search and produces an importance 
ranking that can be converted into task-level saliency maps.

The resulting algorithm has linear evaluation complexity with respect to the number 
of candidate regions under fixed phase and window settings. More importantly, this 
speedup is obtained under a theory that is tailored to the attribution problem. We model a visual understanding framework for evidence accumulation and prove that, under monotone evidence accumulation and additional feature-level 
diminishing-gain conditions, PhaseWin preserves the key ranking behavior of greedy 
search and achieves near-greedy faithfulness guarantees. We further relate these 
guarantees back to the original task response, including response recovery and 
insertion-based evaluation. Thus, PhaseWin is not merely a heuristic acceleration of 
greedy attribution; it provides a validity-preserving route from quadratic greedy 
search to linear-complexity attribution.

Extensive experiments support the proposed formulation and algorithm. For image
classification, we evaluate PhaseWin on ImageNet~\cite{deng2009imagenet} with CLIP ViT-L/14~\cite{radford2021learning}, CLIP ResNet-101~\cite{radford2021learning},
and ResNet-101~\cite{he2016resnet} under both correctly classified and failure-case settings. For
object-level interpretation, we evaluate object detection and referring expression
comprehension on MS COCO~\cite{lin2014microsoft}, LVIS~\cite{gupta2019lvis}, and RefCOCO~\cite{kazemzadeh2014referitgame} with Grounding DINO~\cite{liu2023grounding} and Florence-2~\cite{xiao2024florence}. For
generation, we further test image-caption attribution with Qwen-2.5-VL~\cite{bai2025qwen2} on COCO-style
captioning~\cite{lin2014microsoft} tasks. Across all these scenarios, PhaseWin consistently behaves as an
effective substitute for exhaustive greedy search. It preserves almost the same
faithfulness as greedy search, with only a small metric gap, while using substantially
fewer model evaluations. At the same time, it substantially outperforms all
non-greedy attribution baselines on standard faithfulness metrics, including
Insertion AUC, Deletion AUC, average highest confidence, and early-area recovery
metrics. Thus, PhaseWin occupies a distinct position among attribution methods: it
achieves the high-faithfulness behavior previously associated with greedy search,
but with a much lower evaluation cost. These results indicate that the phase-window
principle is not tied to a specific model, task, or scoring function, but provides a
general acceleration strategy for replacing greedy search in high-faithfulness visual
attribution.

This work makes the following contributions:
\begin{itemize}
    \item We provide a theoretical foundation for greedy-style visual attribution. 
    We show that the standard sufficiency--necessity proxy is not a non-degenerate 
    submodular function, so classical submodular-greedy guarantees cannot justify its 
    use; motivated by this, we reformulate visual attribution as ordered evidence 
    accumulation over image partitions, a model that applies uniformly to 
    classification, detection, grounding, and caption attribution.

    \item We propose PhaseWin, a phased window search algorithm that replaces 
    exhaustive global rescoring with anchor selection, screening and pruning, local 
    window refinement, and accept-or-defer decisions. Under fixed phase and window 
    settings it attains linear evaluation complexity, and under explicit 
    evidence-accumulation conditions we prove that it preserves the key ranking 
    behavior of greedy search with near-greedy faithfulness guarantees. To the best of our knowledge, PhaseWin is the most evaluation-efficient method that attains this level of faithfulness.

    \item We conduct extensive experiments across image classification, object 
    detection, visual grounding, and image captioning, showing that PhaseWin preserves 
    near-greedy faithfulness with only a small metric gap while using substantially 
    fewer model evaluations, and outperforms all non-greedy baselines on standard 
    faithfulness metrics.
\end{itemize}

The present article substantially extends our preliminary conference version~
\cite{gu2026phasewin}. The earlier version mainly focused on empirical acceleration 
for object-level interpretation. This journal version broadens the work in three 
directions. First, it introduces a unified formulation of visual attribution as 
ordered subset search over image partitions. Second, it replaces the previous 
submodular-style narrative with a theory centered on the actual sufficiency--necessity 
proxy, including the impossibility of non-degenerate submodularity, linear evaluation 
complexity, and near-greedy faithfulness guarantees. Third, it expands the empirical 
scope from object-level interpretation to a wider range of attribution settings, 
including image classification, object detection, visual grounding, and image 
captioning.

\section{Related Work}\label{sec:related}

\subsection{Post-hoc Visual Attribution}

Post-hoc visual attribution aims to identify input regions that support a model prediction. Existing methods can be broadly grouped into white-box attribution, perturbation-based attribution, Shapley-style estimation, and search-based attribution. White-box methods construct saliency maps from internal model signals, including relevance propagation~\cite{bach2015pixel}, input gradients~\cite{hassani2017gradient}, activation-gradient maps such as Grad-CAM and Grad-CAM++~\cite{selvaraju2020grad,chattopadhay2018grad}, score-based variants such as Score-CAM~\cite{wang2020score}, and recent architectures-specific methods such as ViT-CX and Grad-ECLIP~\cite{xie2023vit,zhao2024gradient}. Path-integral approaches, including Integrated Gradients and IGOS++, accumulate gradients along a prescribed path~\cite{sundararajan2017axiomatic,khorram2021igos++}. These methods are efficient when model internals are accessible, but their performance is often sensitive to layer choice, baseline design, and architectural details~\cite{novello2022making}.

Perturbation-based methods instead treat the model as a black box and estimate importance from output changes under masked or corrupted inputs. Representative examples include LIME~\cite{ribeiro2016should}, RISE~\cite{petsiuk2018rise}, and dependence-based attribution such as D-HSIC~\cite{novello2022making}. These methods are broadly applicable, but usually require many model evaluations and are sensitive to perturbation granularity, mask resolution, and sampling variance. Shapley-style methods estimate coalition contributions through sampled subsets~\cite{shapley1953value,lundberg2017unified}, sometimes using image regions or structural priors to reduce cost~\cite{sun2023explain,chen2023harsanyinet}. However, coalition-based attribution remains expensive for high-dimensional visual inputs and may dilute the importance of correlated regions~\cite{kumar2020problems}. In contrast, search-based attribution directly constructs an ordered subset of visual regions, making it closely aligned with insertion/deletion faithfulness evaluation~\cite{shitole2021one,chen2024less,chen2025vps}. Our work follows this search-based direction, but focuses on accelerating the ordered region search rather than designing a new attribution score.

\subsection{Attribution Beyond Image Classification}

Compared with image classification, object-level attribution is more challenging because detector outputs couple category prediction, localization, confidence scoring, and post-processing. Existing methods extend gradient-based attribution to detector architectures~\cite{gudovskiy2018explain,selvaraju2020grad,zhao2024gradient_detector}, refine Grad-CAM variants for spatial sensitivity~\cite{yamauchi2022spatial,yamauchi2024spatial,chattopadhay2018grad}, or adapt randomized perturbation to object-level outputs~\cite{petsiuk2018rise,petsiuk2021black}. Other studies analyze detector explanations from complementary perspectives, including diverse rationales~\cite{jiang2023diverse}, architecture comparison~\cite{jiang2024comparing}, representation decomposition~\cite{gandelsman2024interpreting}, and collective pixel contribution~\cite{yamauchi2024explaining}. Recent object-level methods improve faithfulness by searching for compact visual evidence that recovers a target detection~\cite{chen2025vps}, but their greedy region selection requires a quadratic number of model evaluations with respect to the number of candidate regions.

Attribution for multimodal generation introduces another layer of difficulty: the target is no longer a single class or detection score, but a generated sequence whose visual grounding can vary across tokens. Existing work visualizes cross-modal attention, adapts activation maps to token probabilities, or identifies visually grounded tokens in large vision-language models~\cite{ben2024lvlm,omeiza2019smooth,zhang2025redundancy,xing2025large,li2025token,zhang2025mllms,chen2026where}. These methods provide useful grounding evidence, but often depend on internal activations, gradients, attention maps, or token-specific designs. Greedy subset-search attribution offers a more general alternative: given a task-specific evaluation metric, it tests whether a compact set of image regions can recover the target output. This makes the same attribution principle applicable across classification, detection, and caption-level generation.

\subsection{Efficient Subset Search for Attribution}

Greedy subset search is attractive for attribution because it directly builds an ordered insertion trajectory: at each step, the method selects the region that most improves the current response. This procedure often produces highly faithful explanations, but its exhaustive rescoring of all remaining candidates leads to quadratic cost. This bottleneck becomes severe for large models, fine image partitions, or large-scale failure analysis.

A broad algorithmic literature has studied ways to accelerate greedy-style search, including lazy evaluation, multi-stage selection, stochastic candidate reduction, and pruning strategies~\cite{minoux1978lazy,leskovec2007lazy,mirzasoleiman2015lazier,wei2014fast,breuer2020fast}. These methods show that exhaustive rescoring is often unnecessary when many candidates are clearly unpromising. However, visual attribution differs from classical subset selection in two aspects: the goal is not only a high final subset score but also a high-quality ordered response curve, and the scoring function is induced by task-specific model behavior rather than a fixed analytic objective. PhaseWin is designed for this setting. It organizes search into phases, uses anchor regions to prune low-potential candidates, and applies windowed fine-grained selection to promising subsets. As a result, it preserves the empirical strength of greedy attribution while making high-faithfulness subset search practical for large visual and multimodal models.
\section{Method}
\label{sec:method}
To facilitate reading, we provide a notation table in Appendix~\ref{sec:notation}.
\subsection{Problem Setup}
\label{sec:problem_setup}

Let $U$ be a finite ground set with $|U|=n$, and let $G:2^U\to\mathbb R$ denote the base set function of interest. For any set function $H$ and any $X\subseteq U$, $e\in U\setminus X$, define the marginal gain
\begin{equation}
\Delta_H(e\mid X):=H(X\cup\{e\})-H(X).
\end{equation}
Instead of optimizing $G$ directly, we consider the symmetrized search objective
\begin{equation}
\label{eq:search_objective}
F(X):=G(X)+G(U)-G(U\setminus X),\; X\subseteq U.
\end{equation}

PhaseWin searches $F$ and outputs an ordering of elements in $U$. For any ordering $\pi=(\pi_1,\dots,\pi_n)$, define its prefix sets by
\begin{equation}
P_t^\pi:=\{\pi_1,\dots,\pi_t\},\; P_0^\pi:=\varnothing.
\end{equation}
We evaluate an ordering through two prefix-wise quantities induced by $G$:
\begin{itemize}
\item \textbf{Prefix maximum}
\begin{equation}
M_G(\pi):=\max_{0\le t\le n} G(P_t^\pi);
\end{equation}
\item \textbf{Full-cardinality AUC}
\begin{equation}
\operatorname{AUC}_G(\pi):=\sum_{t=1}^n \frac{a_{\pi_t}}{A}\,G(P_t^\pi),
\;
A:=\sum_{e\in U} a_e,
\end{equation}
where $a_e>0$ denotes the area or weight of element $e$.
\end{itemize}

\begin{remark}[Choice of the search objective]
Many perturbation-based attribution methods evaluate subsets directly through $G(X)$, which corresponds to the unsymmetrized choice $F\equiv G$. In contrast, LIMA and VPS adopt the symmetrized objective in Eq.~\eqref{eq:search_objective} and empirically obtain better attribution quality. Intuitively, the complement term compensates for higher-order interactions and makes the objective more amenable to ranking-based selection. To ensure a fair comparison with greedy-style baselines, we keep the same objective and focus exclusively on improving the search procedure.
\end{remark}

To motivate the design of PhaseWin, we introduce a partition-based view of the search objective. Let
\(
\mathcal H=\{H_1,\dots,H_q\}
\)
be a partition of $U$, where $H_i\cap H_j=\varnothing$ for $i\neq j$ and $\bigcup_{j=1}^q H_j=U$. Define the activation signature
\begin{equation}
\chi_{\mathcal H}(X)
:=
\bigl(
\mathbf 1[X\cap H_1\neq\varnothing],\dots,\mathbf 1[X\cap H_q\neq\varnothing]
\bigr),
\end{equation}
and the activated block set
\begin{equation}
B(X):=\{j\in [q]: X\cap H_j\neq\varnothing\}.
\end{equation}
If $\phi:\{0,1\}^q\to \mathbb R$ is a block-activation function, we write its set-valued counterpart as
\begin{equation}
\Phi(J):=\phi(\mathbf 1_J),\; J\subseteq [q].
\end{equation}

The central intuition behind PhaseWin is that the dominant gain in $F$ comes from activating previously uncovered semantic blocks, whereas additional elements selected within an already activated block contribute only small residual improvements. The next subsection turns this structural view into an efficient search algorithm; the formal assumptions required for the guarantees are stated in Sec.~\ref{sec:theory}.

\begin{figure*}[t]
    \centering
    \includegraphics[width=\textwidth]{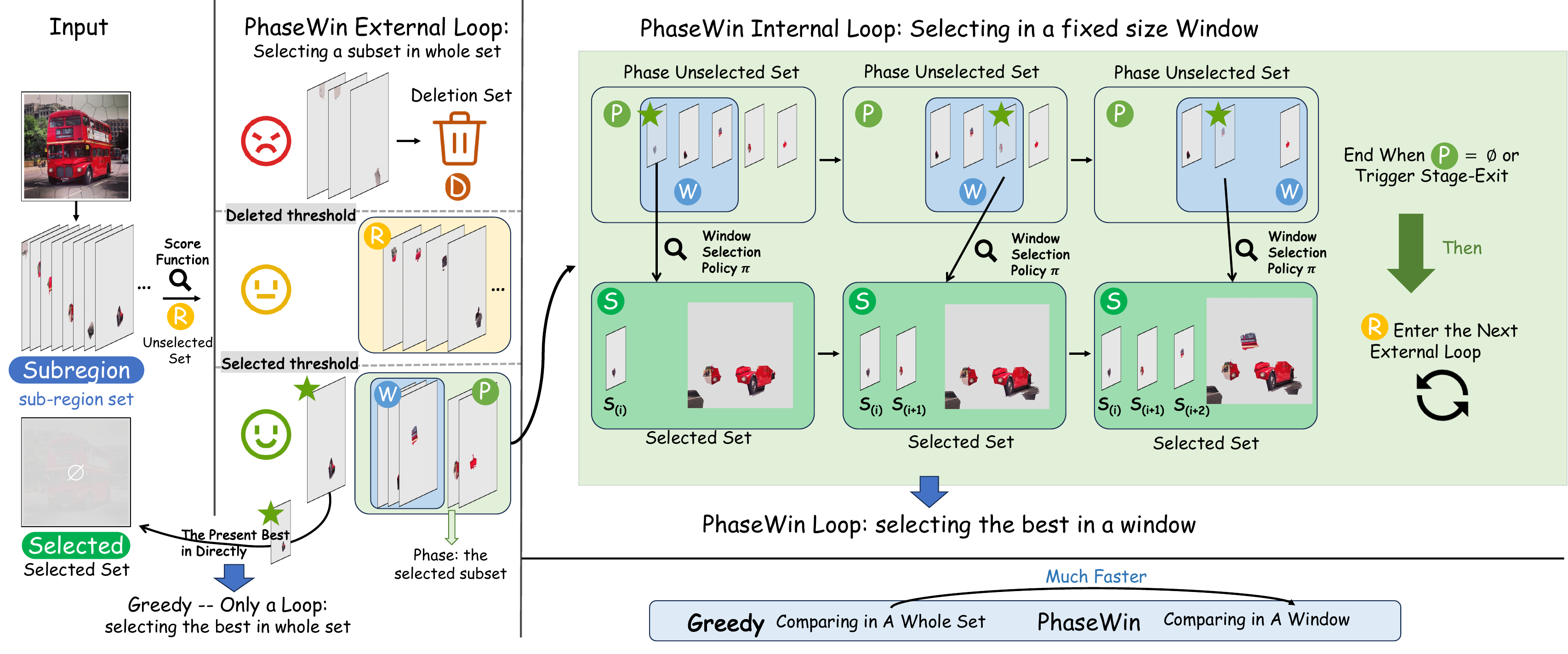}
    \caption{\textbf{PhaseWin workflow.} The algorithm alternates between (i) selecting a high-confidence anchor, (ii) pruning candidates via fixed-ratio thresholds, and (iii) performing windowed local refinement with dynamic supervision.}
    \label{fig:phasewin_workflow}
    \vspace{-12pt}
\end{figure*}

\subsection{Phase-Window Accelerated Search}
\label{sec:phasewin_alg}

A naive greedy procedure recomputes $\Delta_F(e\mid S)$ for every remaining candidate at every step, which leads to $\Theta(n^2)$ evaluations when $|U|=n$. PhaseWin preserves the ranking tendency of greedy search while sharply reducing repeated evaluations through a phased coarse-to-fine strategy.

Given a target cardinality $k$, PhaseWin returns an ordered list
\(
\Pi=(\Pi_1,\dots,\Pi_k),
\)
together with its selected set
\(
S=\{\Pi_1,\dots,\Pi_k\}.
\)
When $k=n$, the output is the full ordering used in the theoretical analysis, which we denote by $\pi^{\mathrm{PW}}$; when $k<n$, the output is a truncated attribution sequence. The overall workflow is shown in Fig.~\ref{fig:phasewin_workflow}.

At the beginning of each phase, the algorithm performs one global scan over the current remaining pool $\mathcal R$ and computes cached gains $\{g_r\}_{r\in\mathcal R}$ with respect to the current set $S$, where
\(
g_r:=\Delta_F(r\mid S).
\)
The highest-gain candidate is accepted as the phase anchor,
\[
\alpha^\star\in\arg\max_{r\in\mathcal R} g_r,
\]
and its gain becomes the phase reference
\begin{equation}
\Delta_{\mathrm{ref}}:=g_{\alpha^\star}.
\end{equation}
Using this reference, PhaseWin constructs two fixed-ratio thresholds,
\begin{equation}
\tau_{\mathrm{sel}}=\rho_{\mathrm{sel}}\Delta_{\mathrm{ref}},
\;
\tau_{\mathrm{del}}=\rho_{\mathrm{del}}\Delta_{\mathrm{ref}},
\;
0<\rho_{\mathrm{del}}<\rho_{\mathrm{sel}}<1.
\end{equation}
The cached gains from the same global scan are then used to partition the remaining candidates into three groups: a high-potential pool $\mathcal P$, a deferred set for the next phase, and a discarded set. This three-way split prevents obviously weak candidates from being repeatedly rescored.

\begin{algorithm}[t]
\caption{PhaseWin: Phase-Window Accelerated Search}
\label{alg:phasewin}
\renewcommand{\gets}{\leftarrow}

\KwIn{Ground set $U$, target size $k$, scoring function $F(\cdot)$, window size $\omega$, window policy $\psi$}
\KwOut{Ordered list $\Pi$}

$\Pi \gets [\,]$; \quad $S \gets \emptyset$; \quad $\mathcal R \gets U$; \quad $\Delta_{\mathrm{ref}} \gets +\infty$\;

\While{$|\Pi| < k$ \textbf{and} $\mathcal R \neq \emptyset$}{
    \tcp{Global anchor selection}
    $g_r \gets \Delta_F(r\mid S)=F(S\cup\{r\})-F(S)$ for all $r\in \mathcal R$\;
    $\alpha^\star \gets \arg\max_{r\in\mathcal R} g_r$\;
    \texttt{append} $\alpha^\star$ to $\Pi$; \quad $S \gets S\cup\{\alpha^\star\}$\;
    $\Delta_{\mathrm{ref}} \gets g_{\alpha^\star}$; \quad $\mathcal R \gets \mathcal R\setminus\{\alpha^\star\}$\;

    \tcp{Fixed-ratio pruning}
    $\tau_{\mathrm{sel}} \gets \rho_{\mathrm{sel}}\Delta_{\mathrm{ref}}$\;
    $\tau_{\mathrm{del}} \gets \rho_{\mathrm{del}}\Delta_{\mathrm{ref}}$\;

    $\mathcal P \gets \emptyset$; \quad $\mathcal R_{\mathrm{next}} \gets \emptyset$\;
    \For{$r\in\mathcal R$}{
        \lIf{$g_r \ge \tau_{\mathrm{sel}}$}{$\mathcal P \gets \mathcal P\cup\{r\}$}
        \lElseIf{$g_r \le \tau_{\mathrm{del}}$}{discard $r$}
        \lElse{$\mathcal R_{\mathrm{next}} \gets \mathcal R_{\mathrm{next}}\cup\{r\}$}
    }
    $\mathcal R \gets \mathcal R_{\mathrm{next}}$\;

    \tcp{Windowed local refinement}
    $(\Pi_{\mathrm{phase}},S_{\mathrm{phase}},\Delta_{\mathrm{ref}})
    \gets
    \texttt{WindowSelection}(\mathcal P,S,k-|\Pi|,F,\Delta_{\mathrm{ref}},\omega,\psi)$\;
    \texttt{append} $\Pi_{\mathrm{phase}}$ to $\Pi$; \quad $S \gets S\cup S_{\mathrm{phase}}$\;
}

\KwRet{$\Pi$}\;
\end{algorithm}

The \texttt{WindowSelection} subroutine performs a fine-grained search only within the pruned pool $\mathcal P$. We first sort $\mathcal P$ by the cached gains $g_r$, place the top $\omega$ candidates into a sliding window $W$, and store the remaining candidates in a queue $Q$. A window policy $\psi(\cdot)$ is then used to choose a subset $A\subseteq W$ for exact reevaluation. We reserve $\psi$ for window policies and keep $\pi$ for orderings to avoid notation conflict.

Table~\ref{tab:phasewin_policies} summarizes the window policies considered in this work. The complexity discussion below is written for a generic policy $\psi$; in particular, the effective local-search factor is $f(\omega)=\omega$ for $\psi_{\mathrm{LG}}$ and $f(\omega)=\log(\omega)$ for $\psi_{\mathrm{BA}}$.

\begin{table}[t]
\centering
\small
\caption{Window-selection policies $\psi(\cdot)$ used within \texttt{WindowSelection}.}
\label{tab:phasewin_policies}
\setlength{\tabcolsep}{4pt}
\begin{tabular}{l p{0.64\columnwidth}}
\toprule
\textbf{Policy} & \textbf{Description} \\
\midrule
$\psi_{\mathrm{LG}}$ &
\textbf{Local-Greedy:} Reevaluates only the highest-ranked candidate in the current window. \\

$\psi_{\mathrm{BA}}$ &
\textbf{Beta-Adaptive:} Reevaluates all candidates whose cached gains exceed an adaptive fraction of the maximum cached gain in the window. \\

$\psi_{\mathrm{T2}}$ &
\textbf{Top-2:} Jointly reevaluates the top two candidates when both appear competitive and their relative gap is sufficiently small. \\

$\psi_{\mathrm{BAF\mbox{-}B}}$ &
\textbf{Batched Best-Above with Forward Checking:} Processes the window in short batches and terminates early when the remaining cached gains are no longer competitive. \\
\bottomrule
\end{tabular}
\end{table}

For each candidate $\alpha\in A$, the algorithm recomputes its true gain $\Delta_F(\alpha\mid S)$. A \emph{phase-exit} rule then compares this value with the current reference $\Delta_{\mathrm{ref}}$: if
\begin{equation}
\Delta_F(\alpha\mid S)<\theta\,\Delta_{\mathrm{ref}},
\end{equation}
the current phase terminates early, since the remaining candidates are unlikely to alter the prefix order substantially. Otherwise, the candidate is processed by an \emph{annealing delay} mechanism, which either accepts it immediately or postpones it to encourage local exploration. Accepted candidates are appended to $\Pi$, inserted into $S$, and used to update $\Delta_{\mathrm{ref}}$. The window is then replenished from $Q$ until either $|\Pi|=k$ or no promising candidates remain.

This design directly matches the partition-based view introduced in Sec.~\ref{sec:problem_setup}. Under a decomposition of the form
\(
F(X)=\Phi(B(X))+R(X),
\)
the global anchor step captures high-value block activations, while the windowed refinement resolves the smaller within-block differences contributed by the residual term $R(X)$. PhaseWin therefore separates the expensive global search for new informative blocks from the cheaper local ranking within the current high-potential pool.

As shown in Theorem~\ref{thm:linear}, if the number of effective phase transitions is bounded independently of $n$, then the total evaluation cost is
\[
O\!\bigl(n(f(\omega)+1)\bigr),
\]
which becomes linear in $n$ for fixed $\omega$.

\subsection{Theoretical Guarantees}
\label{sec:theory}

Then we state the assumptions and guarantees used to analyze PhaseWin.  The analysis distinguishes between a target-cardinality guarantee, which applies to the first $k$ accepted regions, and full-order prefix/AUC guarantees, which apply when PhaseWin is run to produce a complete ordering.  For a PhaseWin output $\Pi=(v_1,\ldots,v_k)$, define
\begin{equation}
S_i^{\mathrm{PW}}:=\{v_1,\ldots,v_i\},
\;
S_0^{\mathrm{PW}}:=\varnothing .
\end{equation}

We first assume that $F$ is a monotone set function and has properties as follows.

\begin{assumption}[Partition-dominant structure]
\label{ass:partition}
There exist a monotone submodular function $\Phi:2^{[q]}\to\mathbb R_+$ with $\Phi(\varnothing)=0$ and a residual term $R:2^U\to\mathbb R_+$ such that
\begin{equation}
F(X)=\Phi(B(X))+R(X),\; \forall X\subseteq U,
\end{equation}
where $R(\varnothing)=0$ and
\begin{equation}
0\le \Delta_R(e\mid X)\le \varepsilon_R,
\;
\forall X\subseteq U,\ e\in U\setminus X.
\end{equation}
Further, for approximate window policies, we further assume
\begin{equation}
\label{eq:kappa-gap}
\varepsilon_R<\kappa_1\,\underline{\Delta}_\Phi,
\end{equation}
where
\begin{equation}
\label{eq:block-gap}
\begin{aligned}
\Delta_\Phi(j\mid J)
&:= \Phi(J\cup\{j\})-\Phi(J),\\
\underline{\Delta}_\Phi
&:=
\min_{J\subsetneq [q],\, j\notin J}
\Delta_\Phi(j\mid J).
\end{aligned}
\end{equation}

\end{assumption}

Let $\mathcal R_i$ be the live candidate pool immediately before $v_i$ is accepted, and let $\mathcal D_i$ be the set of candidates hard-deleted at the same decision point; for ordinary window steps without deletion, $\mathcal D_i=\varnothing$. Define
\begin{equation}
a_i:=\max_{e\in\mathcal R_i}\Delta_F(e\mid S_{i-1}^{\mathrm{PW}}).
\end{equation}
To accelerate the algorithm, we may adopt different selection threshold ratios, typically increasing step by step and thus there exist increasing policy-dependent constants $\beta_i^\psi\in(0,1]$ such that, with
\begin{equation}
\kappa_i:=\rho_{\mathrm{sel}}\beta_i^\psi,
\end{equation}
the accepted element satisfies
\begin{equation}
\Delta_F(v_i\mid S_{i-1}^{\mathrm{PW}})
\ge
\kappa_i\,a_i,
\; i=1,\ldots,k.
\end{equation}
Moreover, every deleted candidate satisfies
\begin{equation}
\Delta_F(e\mid S_{i-1}^{\mathrm{PW}})
\le
\rho_{\mathrm{del}}\,
\Delta_F(v_i\mid S_{i-1}^{\mathrm{PW}}),
\;
\forall e\in\mathcal D_i.
\end{equation}

For the full-order prefix and AUC guarantees, we additionally use a block-safe live-pool condition:

\begin{assumption}[Window-faithful selection]
\label{ass:windowfaithful}
    \begin{equation}
    \label{eq:block-safe}
    \mathcal R_i\cap H_j\neq\varnothing,
    \;
    \forall i\le q,\ \forall j\notin B(S_{i-1}^{\mathrm{PW}}).
    \end{equation} 
\end{assumption}
This condition states that before all semantic blocks have been activated, every inactive block still has at least one live representative.

\begin{remark}
The condition in Eq.~\eqref{eq:block-safe} is mainly introduced to facilitate theoretical analysis,
ensuring that each unactivated semantic block remains observable during the search process.

In practical implementations of PhaseWin, this condition is typically satisfied implicitly.
In particular, we adopt a very small deletion ratio $\rho_{\mathrm{del}}$, so that the deletion
operation is rarely triggered. As a result, most candidate regions are preserved across phases,
and the live pool $\mathcal R_i$ continues to contain representatives from nearly all semantic blocks.

Therefore, although Eq.~\eqref{eq:block-safe} appears as a structural assumption in the analysis,
it does not impose a restrictive requirement in practice.
\end{remark}

Since we optimize $F$ instead of $G$, a two-sided alignment of $F$ and $G$ should be added. A counterexample is given in \Cref{app:proofs} stating that without this assumption, the algorithm fails.

\begin{assumption}[Two-sided alignment]
\label{ass:align}
There exist constants $0\leq \lambda_1\leq \lambda_2$ and $b_1,b_2\in\mathbb{R}$ such that for any $X$
\begin{equation}
\lambda_1G(X)+b_1\leq F(X)\leq \lambda_2G(X)+b_2.
\end{equation}
\end{assumption}
\begin{remark}
    Usually $b_1$ and $b_2$ are related to $\epsilon_R$ and $k$, and $\lambda_1, \lambda_2$ are related to the submodular ration of $G$. Typically $\lambda_1\leq 1\leq \lambda_2$.
\end{remark}

Under the above assumptions, we state the main theorems as following:

\textbf{Target-cardinality approximation.}
Let
\begin{equation}
\begin{aligned}
&S_{F,k}^\star
\in \arg\max_{|X|\le k} F(X),\\
&S_{G,k}^\star
\in \arg\max_{|X|\le k} G(X),\\
&\mathrm{OPT}_{G,k}
:= G(S_{G,k}^\star).
\end{aligned}
\end{equation}
Define
\begin{equation}
\label{eq:Ck}
C_k(\kappa,\rho_{\mathrm{del}})
:=
\frac{
1-\left(
1-\frac{\kappa(1+k\rho_{\mathrm{del}})}{k}
\right)^k
}{
1+k\rho_{\mathrm{del}}
}.
\end{equation}

\begin{theorem}[Cardinality-$k$ search-objective guarantee]
\label{thm:finiteF}
Under Assumptions~\ref{ass:partition} and \ref{ass:windowfaithful}, suppose
\begin{equation}
0\le
1-\frac{\kappa_1(1+k\rho_{\mathrm{del}})}{k}
\le 1.
\end{equation}
Then the first $k$ elements returned by PhaseWin satisfy
\begin{equation}
F(S_k^{\mathrm{PW}})
\ge
C_k(\kappa_1,\rho_{\mathrm{del}})
\left(
F(S_{F,k}^\star)-k\varepsilon_R
\right).
\end{equation}
and under Assumption~\ref{ass:align}
\begin{equation}
\begin{aligned}
G(S_k^{\mathrm{PW}})
\;\ge\;&
\frac{C_k(\kappa_1,\rho_{\mathrm{del}})\lambda_1}{\lambda_2}\,\mathrm{OPT}_{G,k}\\
&+\frac{C_k(\kappa_1,\rho_{\mathrm{del}}) b_1-
C_k(\kappa_1,\rho_{\mathrm{del}}) k\varepsilon_R-b_2}{\lambda_2}.
\end{aligned}
\end{equation}
In particular, if $\kappa_1=1-o(1)$ and $\rho_{\mathrm{del}}=o(1/k)$, then
\begin{equation}
C_k(\kappa_1,\rho_{\mathrm{del}})
=
1-\frac{1}{e}-o(1),
\end{equation}
\end{theorem}
\textbf{Complexity.}
Let $f_\psi(\omega)$ denote the number of true reevaluations induced by the window policy $\psi$ per window scan. For example, $f_{\psi_{\mathrm{LG}}}(\omega)=\omega$ and $f_{\psi_{\mathrm{BA}}}(\omega)=\log\omega$ in the policies considered in this work.

\begin{theorem}[Near-linear complexity]
\label{thm:linear}
Under Assumptions~\ref{ass:partition}, every effective non-terminal phase whose live pool still contains an inactive block activates at least one previously inactive block. Consequently, the number of effective non-terminal phases is at most $q$, and the total number of evaluations of $F$ is
\begin{equation}
O\!\bigl((q+1)n(f_\psi(\omega)+1)\bigr).
\end{equation}
When $q$ is independent of $n$, this reduces to $O(n(f_\psi(\omega)+1))$, and to $O(n)$ for fixed $\omega$.
\end{theorem}

\textbf{Full-order prefix and AUC guarantees.}
For the following two guarantees, which are direct results of \ref{thm:finiteF}, PhaseWin is run in the full-order regime $k=n$ and the resulting ordering is denoted by $\pi^{\mathrm{PW}}$. 

\begin{corollary}[Prefix-maximum guarantee]
\label{thm:max}
Under Assumptions~\ref{ass:partition}--\ref{ass:align}, let
\begin{equation}
\begin{aligned}
X^\star
&\in \arg\max_{X\subseteq U} G(X),\\
J^\star
&:= B(X^\star),\\
r^\star
&:= |J^\star|,\\
\mathrm{OPT}&_G
:= G(X^\star).
\end{aligned}
\end{equation}
Then
\begin{equation}
\begin{aligned}
M_G(\pi^{\mathrm{PW}})
\ge & \frac{\lambda_1}{\lambda_2}\cdot C_{r^*}(\kappa_1,\rho_{\mathrm{del}}) \mathrm{OPT}_G\\
&+\frac{C_{r^*}(\kappa_1,\rho_{\mathrm{del}}) b_1-
C_{r^*}(\kappa_1,\rho_{\mathrm{del}}) r^*\varepsilon_R-b_2}{\lambda_2}.
\end{aligned}
\end{equation}
\end{corollary}

For the AUC guarantee, we specialize to the equal-area setting $a_e\equiv 1$, under which
\begin{equation}
\operatorname{AUC}_G(\pi)
=
\frac{1}{n}\sum_{t=1}^nG(P_t^\pi),
\;
\operatorname{AUC}_G^\star
:=
\max_{\pi}\operatorname{AUC}_G(\pi).
\end{equation}

\begin{corollary}[AUC guarantee]
\label{thm:auc}
Under Assumptions~\ref{ass:partition}--\ref{ass:align}, 
let
\begin{equation}
    C_{\min}:=\min_{1\le t\le n}   C_t(\kappa_1,\rho_{\mathrm{del}}),\quad  \Gamma_{\min}:=
\min_{1\le t\le n}
\bigl(c_t b_1-c_t t\varepsilon_R\bigr).
\end{equation}
Then the PhaseWin ordering satisfies
\begin{equation}
\operatorname{AUC}_G(\pi^{\mathrm{PW}})
\ge
\frac{C_{\min}\lambda_1}{\lambda_2}
\operatorname{AUC}_G^\star
+
\frac{\Gamma_{\min}-b_2}{\lambda_2}.
\end{equation}
\end{corollary}

We can also give some sufficient conditions to obtain a more concrete bounds for the approximation.
Proofs are deferred to Appendix~\ref{app:proofs}.

\begin{table*}[t]
\centering
\small
\setlength{\tabcolsep}{4.2pt}
\renewcommand\arraystretch{1.12}

\caption{
Evaluation scope across classification, detection, grounding, and MLLM attribution tasks.
}

\label{tab:eval_scope}

\resizebox{\textwidth}{!}{%
\begin{tabular}{
l
l
l
!{\color{gray!35}\vrule}
l
}
\toprule

\textbf{Setting} 
& \textbf{Dataset} 
& \textbf{Model} 
& \textbf{Methods} \\

\midrule

\multirow{3}{*}{\textbf{Classification}}
& \multirow{3}{*}{ImageNet-1K (val)~\cite{deng2009imagenet}}
& CLIP ViT-L/14~\cite{radford2021learning}
& Gradient~\cite{simonyan2014deep}, Gradient Integral~\cite{sundararajan2017axiomatic}, Gradient ECLIP~\cite{zhao2024gradient}, IGOS++~\cite{khorram2021igos++}, RISE~\cite{petsiuk2018rise}, HSIC~\cite{novello2022making}, Greedy~\cite{chen2024less}, PhaseWin (Ours) \\

& 
& CLIP RN101~\cite{radford2021learning}
& Gradient~\cite{simonyan2014deep}, Gradient Integral~\cite{sundararajan2017axiomatic}, IGOS++~\cite{khorram2021igos++}, RISE~\cite{petsiuk2018rise}, HSIC~\cite{novello2022making}, Greedy~\cite{chen2024less}, PhaseWin (Ours) \\

& 
& ResNet-101~\cite{he2016resnet}
& Gradient~\cite{simonyan2014deep}, Gradient Integral~\cite{sundararajan2017axiomatic}, IGOS++~\cite{khorram2021igos++}, RISE~\cite{petsiuk2018rise}, HSIC~\cite{novello2022making}, Greedy~\cite{chen2024less}, PhaseWin (Ours)\\

\midrule

\multirow{2}{*}{\textbf{Detection}}
& MS COCO~\cite{lin2014microsoft}
& \multirow{2}{*}{Grounding DINO~\cite{liu2023grounding}, Florence-2~\cite{xiao2024florence}}
& Grad-CAM~\cite{selvaraju2020grad}, SSGrad-CAM++~\cite{yamauchi2022spatial}, ODAM~\cite{zhao2024gradient}, RISE~\cite{petsiuk2021black}, HSIC~\cite{novello2022making}, Greedy~\cite{chen2025vps}, PhaseWin (Ours)\\

& LVIS v1~\cite{gupta2019lvis}
& 
& Grad-CAM~\cite{selvaraju2020grad}, SSGrad-CAM++~\cite{yamauchi2022spatial}, ODAM~\cite{zhao2024gradient}, RISE~\cite{petsiuk2021black}, HSIC~\cite{novello2022making}, Greedy~\cite{chen2025vps}, PhaseWin (Ours)\\

\midrule

\textbf{Grounding (REC)}
& RefCOCO~\cite{kazemzadeh2014referitgame}
& Grounding DINO~\cite{liu2023grounding}, Florence-2~\cite{xiao2024florence}
& Grad-CAM~\cite{selvaraju2020grad}, SSGrad-CAM++~\cite{yamauchi2022spatial}, ODAM~\cite{zhao2024gradient}, RISE~\cite{petsiuk2021black}, HSIC~\cite{novello2022making}, Greedy~\cite{chen2025vps}, PhaseWin (Ours) \\

\midrule

\multirow{2}{*}{\textbf{MLLM Attribution}}
& \multirow{2}{*}{MS COCO Captions~\cite{lin2014microsoft}}
& Qwen2.5-VL-3B-Instruct~\cite{bai2025qwen2}
& Gradient~\cite{simonyan2014deep}, LLaVA-CAM~\cite{zhang2025redundancy}, RISE~\cite{petsiuk2018rise}, IGOS++~\cite{khorram2021igos++}, Greedy~\cite{chen2026where}, PhaseWin (Ours)\\

& 
& Qwen2.5-VL-7B-Instruct~\cite{bai2025qwen2}
& Gradient~\cite{simonyan2014deep}, LLaVA-CAM~\cite{zhang2025redundancy}, RISE~\cite{petsiuk2018rise}, IGOS++~\cite{khorram2021igos++}, Greedy~\cite{chen2026where}, PhaseWin (Ours)\\

\bottomrule
\end{tabular}
}
\vspace{-4pt}
\end{table*}

\section{Experiments}
\label{sec:exp}
We evaluate PhaseWin as a general accelerator for region-based attribution in
three progressively more challenging settings: standard post-hoc image
classification explanation, object-level attribution for detection and referring
expression comprehension (REC), and token-level caption attribution for
multimodal large language models (MLLMs). This expanded evaluation is designed
to test not only whether PhaseWin preserves the faithfulness of greedy search in its
original object-level regime, but also whether the same search principle
transfers to classical discriminative explanations and generative multimodal
attribution. Table~\ref{tab:eval_scope} summarizes the datasets, models,
and attribution baselines used in all experiments.

\textbf{Unified naming and mechanism-level grouping.}
To emphasize underlying algorithmic mechanisms rather than task-specific
implementations, we adopt a unified naming scheme across all experiments.

For perturbation-based methods, RISE and HSIC denote the general attribution
families, abstracting away task-specific instantiations (e.g., RISE and
HSIC in detection). These variants share the same perturbation principle
while differing in task-specific modeling and evaluation protocols.

For search-based methods, Greedy refers to the standard subset selection
procedure used in subset attribution. While prior work instantiates this
procedure under different task settings with tailored objectives and modeling
choices, the underlying search mechanism remains the same. Our focus is to
isolate and improve this shared search primitive.

For gradient-based methods, we distinguish approaches based on how gradient
information is utilized, as different usages correspond to different attribution
mechanisms. Specifically, Gradient denotes single-step local sensitivity
(saliency), while Gradient Integral denotes path-integrated variants.
When different path construction strategies are employed (e.g., linear
interpolation versus representation-constrained paths), we treat them as
distinct methods.

\subsection{Unified Setup and Evaluation Protocol}
\label{sec:exp_setup}

\textbf{Common implementation details.}
All tasks share the same region-based attribution runtime. Given an image and a task-specific target, we first partition the image into disjoint regions, then run an attribution method to obtain a ranking over regions, and finally evaluate the ordered regions with a common insertion/deletion replay protocol. The target score depends on the task: class probability for classification, object-level confidence for detection and REC, and mean selected-token probability for caption attribution. Search-based methods directly output ordered regions, whereas gradient- and perturbation-based baselines first produce pixel-level saliency maps, which we aggregate to region scores and convert into the same ordered-mask representation for fair comparison. Unless otherwise noted, we use SLIC~\cite{achanta2012slic} or SLICO~\cite{achanta2012slic} superpixels and set $\lambda_1=\lambda_2=1$ in the shared gain function
\begin{equation}
    G(S)=\lambda_1 s_{\mathrm{ins}}(S)+\lambda_2\bigl(1-s_{\mathrm{del}}(S)\bigr).
\end{equation}
We use 50 regions for classification, 64 regions for caption, and 50 or 100 regions for detection and grounding. Unless otherwise stated, PhaseWin uses a window size of 16 for the 50- and 64-region settings and 32 for the 100-region setting. For PhaseWin, we use the ratio-based early-exit criterion
\[
\frac{S_{k-2}}{S_{k-1}}-\frac{S_{k-1}}{S_k}\le \tau,
\]
with $\tau=0.025$ for 50 subregions and $\tau=0.01$ for 100 subregions.

\begin{table*}[t]
\centering
\small
\setlength{\tabcolsep}{3.8pt}
\renewcommand\arraystretch{1.10}

\caption{Classification attribution on ImageNet (correct samples).}
\label{tab:cls_main_correct}

\resizebox{\textwidth}{!}{%
\begin{tabular}{
l
!{\color{gray!35}\vrule}
cccccc
!{\color{gray!35}\vrule}
cccccc
}
\toprule

\multirow{2}{*}{\textbf{Method}}
& \multicolumn{6}{c!{\color{gray!35}\vrule}}{\textbf{CLIP ViT-L/14}}
& \multicolumn{6}{c}{\textbf{CLIP RN101}} \\

\cmidrule(lr){2-7}
\cmidrule(lr){8-13}

& \shortstack{Ins.\\(↑)}
& \shortstack{Del.\\(↓)}
& \shortstack{Ave.~high.\\(↑)}
& \shortstack{$\mu$-fid\\(↑)}
& \shortstack{MEC\\(↓)}
& \shortstack{A--C\\(↑)}

& \shortstack{Ins.\\(↑)}
& \shortstack{Del.\\(↓)}
& \shortstack{Ave.~high.\\(↑)}
& \shortstack{$\mu$-fid\\(↑)}
& \shortstack{MEC\\(↓)}
& \shortstack{A--C\\(↑)} \\

\midrule

Gradient
& 0.4404 & 0.4783 & 0.9081 & 0.1881 & \na & --
& 0.3495 & 0.2193 & 0.7677 & 0.2076 & \na & -- \\

Gradient Integral
& 0.4213 & 0.5012 & 0.9081 & 0.1887 & \na & --
& 0.3540 & 0.2271 & 0.7694 & 0.2054 & \na & -- \\

Grad-ECLIP
& 0.6488 & 0.2791 & 0.9273 & 0.1595 & \na & --
& \multicolumn{6}{c}{\textit{not supported}} \\

IGOS++
& 0.5224 & 0.4149 & 0.9122 & 0.1792 & \na & --
& 0.2871 & 0.2171 & 0.7655 & 0.2219 & \na & -- \\

RISE
& 0.6364 & 0.3161 & 0.9282 & 0.1509 & 5000.00 & 1.27
& 0.4627 & 0.1232 & 0.7917 & 0.1733 & 5000.00 & 0.93 \\

HSIC
& 0.6755 & 0.2617 & 0.9191 & 0.1600 & 1536.00 & 4.39
& 0.4405 & 0.1314 & 0.7768 & 0.1835 & 1536.00 & 2.87 \\

\midrule

Greedy
& 0.8239 & 0.1388 & 0.9707 & 0.1743 & 1735.90 & 4.74
& 0.6525 & 0.0650 & 0.8943 & 0.2101 & 1736.68 & 3.76 \\

\rowcolor{oursrow}
PhaseWin
& 0.7990 & 0.1625 & 0.9653 & 0.1717 & 871.84 & 9.16
& 0.5981 & 0.0674 & 0.8783 & 0.2046 & 951.16 & 6.29 \\

\bottomrule
\end{tabular}}
\vspace{-6pt}
\end{table*}

\textbf{Evaluation metrics.}

Faithfulness is the primary criterion across all tasks. We report
Insertion AUC (higher is better), Deletion AUC (lower is better),
Average Highest, and early-area recovery metrics such as
Highest@30\% and Highest@50\% when available.

Efficiency is measured by the average number of model evaluations,
denoted as $\mathrm{MEC}_{\mathrm{ave}}$, where one unit corresponds
to a single forward pass.

To jointly reflect faithfulness and efficiency, we also report the
accuracy–cost ratio (A–C ratio), defined as the primary faithfulness
metric (scaled by 10000) divided by the number of forward passes.
We note that this metric can be misleading in low-faithfulness regimes,
as it may favor methods that achieve low scores with extremely few
evaluations. Therefore, A–C ratio is only meaningful when methods
operate at sufficiently high faithfulness levels.

For methods whose implementations do not expose directly comparable
evaluation counts, we leave the efficiency entries blank.

For classification tasks, we additionally report $\mu$-fidelity~\cite{fel2021sobol}.
For detection and REC, we further report class-specific insertion
and deletion AUC~\cite{petsiuk2021black}, Point Game~\cite{zhang2018top}, Energy Point Game~\cite{wang2020score} for localization,
and ESR~\cite{chen2025vps} on failure cases.

For caption attribution, we report sensitivity-aware insertion
and deletion AUC~\cite{zhang2025mllms}, computed on the subset of visually sensitive
generated tokens.

Overall, this unified evaluation protocol ensures that comparisons
primarily reflect the quality of region ordering, rather than
differences in raw heatmap appearance.

\subsection{Image Classification Attribution}
\label{sec:cls_exp}

We first evaluate PhaseWin in the standard post-hoc image classification setting. Given an image and a target class, the goal is to identify an ordered set of image regions whose progressive insertion most efficiently recovers the model response to the target. We conduct experiments on ImageNet validation images with three representative classifiers: CLIP ViT-L/14, CLIP RN101, and ResNet-101. All images are resized to \(224\times224\), and each image is partitioned into 50 SLIC superpixels. For each method, the attribution order is evaluated under the common insertion/deletion replay protocol.

For each backbone, we construct three evaluation splits. The first split contains 5,000 correctly classified images, where the model prediction matches the ground-truth label. The other two splits are derived from 2,000 misclassified images, for which both the model-predicted label and the ground-truth label are recorded. This design allows us to evaluate not only standard attribution on successful predictions, but also failure attribution under two complementary targets: the wrong class selected by the model and the true class missed by the model. The samples are collected by randomly shuffling the ImageNet validation set and running model inference sequentially until the required number of correct and incorrect samples is reached for each backbone.

We compare PhaseWin with gradient-based methods, perturbation-based methods, and greedy subset search. For CLIP ViT-L/14, we additionally include Grad-ECLIP, which is specifically designed for CLIP-style vision transformers and is therefore not applicable to CLIP RN101 or ResNet-101. We report Insertion AUC, Deletion AUC, Average Highest, \(\mu\)-fidelity on correct predictions, early recovery at \(50\%\) revealed area for failure cases, Model Evaluation Count (MEC), and the accuracy--cost ratio (A--C).

\begin{table}[t]
\centering
\small
\setlength{\tabcolsep}{3.8pt}
\renewcommand\arraystretch{1.10}

\caption{Classification attribution on ImageNet (correct samples) with ResNet-101.}
\label{tab:cls_resnet101}

\resizebox{\columnwidth}{!}{%
\begin{tabular}{
l
!{\color{gray!35}\vrule}
ccccc
}
\toprule

\multirow{2}{*}{\textbf{Method}}
& \multicolumn{5}{c}{\textbf{ResNet-101}} \\

\cmidrule(lr){2-6}

& \shortstack{Ins.\\(↑)}
& \shortstack{Del.\\(↓)}
& \shortstack{Ave.~high.\\(↑)}
& \shortstack{MEC\\(↓)}
& \shortstack{A--C\\(↑)} \\

\midrule

Gradient
& 0.4649 & 0.3637 & 0.8025 & \na & -- \\

Gradient Integral
& 0.5062 & 0.3347 & 0.8064 & \na & -- \\

IGOS++
& 0.4740 & 0.3536 & 0.8136 & \na & -- \\

RISE
& 0.6083 & 0.2429 & 0.8254 & 5000.00 & 1.22 \\

HSIC
& 0.6128 & 0.2233 & 0.8213 & 1536.00 & 3.99 \\

\midrule

Greedy
& 0.7926 & 0.1449 & 0.9349 & 1744.07 & 4.54 \\

\rowcolor{oursrow}
PhaseWin
& 0.7672 & 0.1556 & 0.9225 & 907.73 & 8.45 \\

\bottomrule
\end{tabular}}
\vspace{-6pt}
\end{table}

\subsubsection{Correctly Classified Samples}

Tables~\ref{tab:cls_main_correct} and~\ref{tab:cls_resnet101} report the results on correctly classified ImageNet samples. Across all three backbones, search-based region selection is clearly stronger than map-based attribution. Greedy subset search obtains the best raw faithfulness overall, confirming that direct region ordering is highly aligned with the insertion/deletion protocol. However, PhaseWin consistently stays close to Greedy while using substantially fewer model evaluations.

On CLIP ViT-L/14, PhaseWin achieves an Insertion AUC of \(0.7990\), compared with \(0.8239\) from Greedy, while reducing MEC from \(1735.90\) to \(871.84\). This corresponds to \(96.98\%\) of Greedy's Insertion AUC with roughly half of the model evaluations. The same pattern holds for CLIP RN101, where PhaseWin obtains \(0.5981\) Insertion AUC against \(0.6525\) from Greedy, with MEC reduced from \(1736.68\) to \(951.16\). On ResNet-101, PhaseWin reaches \(0.7672\) Insertion AUC, close to Greedy's \(0.7926\), while reducing MEC from \(1744.07\) to \(907.73\).

Averaged over the three correctly classified settings, PhaseWin preserves \(95.39\%\) of Greedy's Insertion AUC and \(98.79\%\) of its Average Highest score, while using only \(52.35\%\) of the model evaluations. This gives an average speedup of \(1.91\times\) over Greedy. The A--C ratio also improves from \(0.43\) to \(0.80\), showing that the small faithfulness loss is compensated by a much better efficiency--faithfulness trade-off.

Compared with perturbation-based baselines, PhaseWin is consistently more faithful. On CLIP ViT-L/14, it improves Insertion AUC over RISE and HSIC by \(16.26\) and \(12.35\) percentage points, respectively. On CLIP RN101, the corresponding gains are \(13.54\) and \(15.76\) percentage points. On ResNet-101, PhaseWin exceeds RISE and HSIC by \(15.89\) and \(15.44\) percentage points. The deletion results show the same trend: PhaseWin yields substantially lower Deletion AUC than the non-search baselines across all three backbones. These results indicate that PhaseWin does not merely accelerate Greedy; it remains in the high-faithfulness regime that map-based and perturbation-based methods fail to reach.

\begin{figure*}[t]
    \centering
    \includegraphics[width=\textwidth]{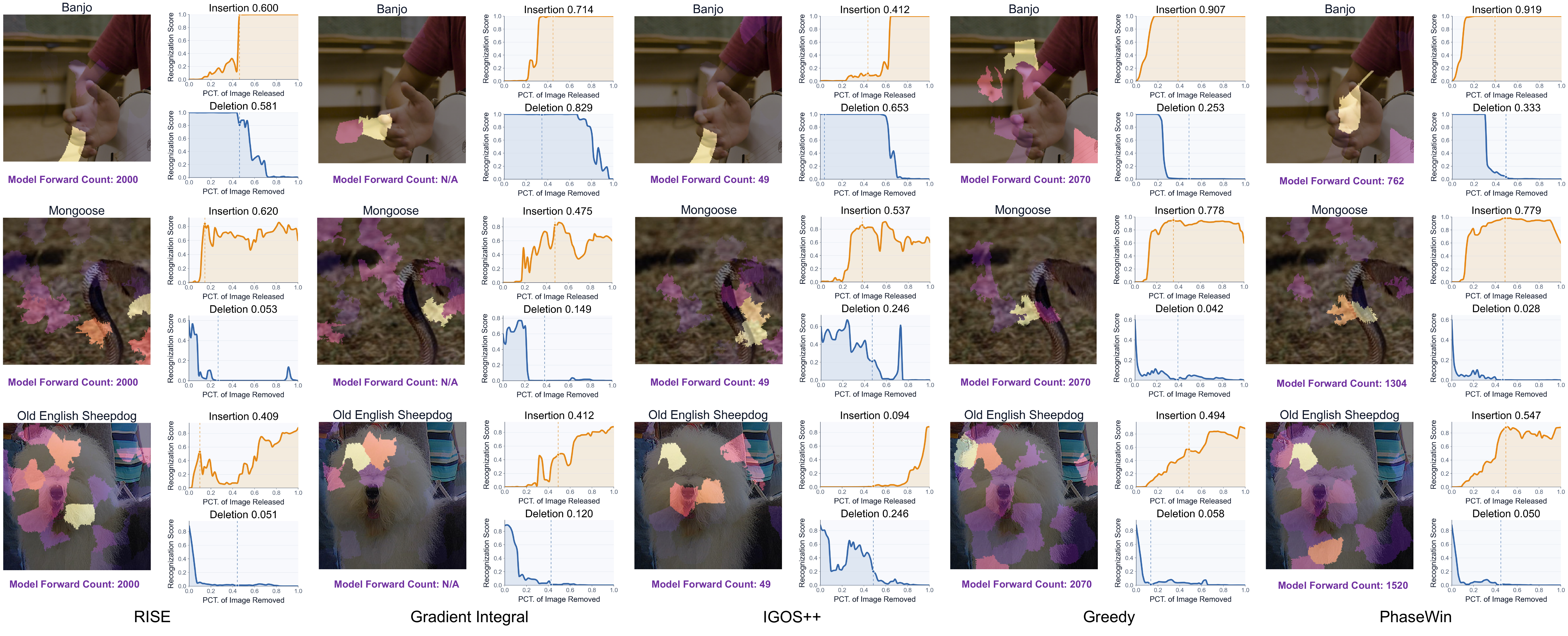}
    \vspace{-5pt}
    \caption{
\textbf{Qualitative comparison on ImageNet classification attribution with CLIP ViT-L/14.}
Each row shows one target class, and each column compares one method.
The overlays visualize the ranked superpixel regions, while the curves report the corresponding insertion and deletion trajectories.
PhaseWin produces region orderings visually close to Greedy and achieves comparable insertion/deletion behavior, but requires substantially fewer model forward evaluations.
}
    \label{fig:cls_vis}
    \vspace{-4pt}
\end{figure*}

\subsubsection{Failure Attribution Toward the Model's Wrong Prediction}

\begin{table}[t]
\centering
\small
\setlength{\tabcolsep}{3.8pt}
\renewcommand\arraystretch{1.10}

\caption{Classification attribution on ImageNet misclassified samples with CLIP ViT-L/14, using the model's wrong prediction as the attribution target.}
\label{tab:cls_vitl_wrong_pred}

\resizebox{\columnwidth}{!}{%
\begin{tabular}{
l
!{\color{gray!35}\vrule}
cccccc
}
\toprule

\multirow{2}{*}{\textbf{Method}}
& \multicolumn{6}{c}{\textbf{CLIP ViT-L/14}} \\

\cmidrule(lr){2-7}

& \shortstack{Ins.\\(↑)}
& \shortstack{Del.\\(↓)}
& \shortstack{Ave.~high.\\(↑)}
& \shortstack{@50\%\\(↑)}
& \shortstack{MEC\\(↓)}
& \shortstack{A--C\\(↑)} \\

\midrule

Gradient
& 0.2607 & 0.2831 & 0.6933 & 0.3118 & \na & -- \\

Gradient Integral
& 0.2468 & 0.2984 & 0.6899 & 0.2817 & \na & -- \\

Grad-ECLIP
& 0.3958 & 0.1616 & 0.7363 & 0.5644 & \na & -- \\

IGOS++
& 0.3035 & 0.2395 & 0.7016 & 0.3866 & \na & -- \\

RISE
& 0.4133 & 0.1605 & 0.7474 & 0.5851 & 5000.00 & 0.83 \\

HSIC
& 0.4035 & 0.1518 & 0.7124 & 0.5757 & 1536.00 & 2.63 \\

\midrule

Greedy
& 0.6837 & 0.0652 & 0.8932 & 0.8576 & 1755.77 & 3.89 \\

\rowcolor{oursrow}
PhaseWin
& 0.6421 & 0.0705 & 0.8774 & 0.8105 & 1026.52 & 6.26 \\

\bottomrule
\end{tabular}}
\vspace{-6pt}
\end{table}

\begin{table}[t]
\centering
\small
\setlength{\tabcolsep}{3.8pt}
\renewcommand\arraystretch{1.10}

\caption{Classification attribution on ImageNet misclassified samples with CLIP RN101, using the model's wrong prediction as the attribution target.}
\label{tab:cls_rn101_wrong_pred}

\resizebox{\columnwidth}{!}{%
\begin{tabular}{
l
!{\color{gray!35}\vrule}
cccccc
}
\toprule

\multirow{2}{*}{\textbf{Method}}
& \multicolumn{6}{c}{\textbf{CLIP RN101}} \\

\cmidrule(lr){2-7}

& \shortstack{Ins.\\(↑)}
& \shortstack{Del.\\(↓)}
& \shortstack{Ave.~high.\\(↑)}
& \shortstack{@50\%\\(↑)}
& \shortstack{MEC\\(↓)}
& \shortstack{A--C\\(↑)} \\

\midrule

Gradient
& 0.1891 & 0.1120 & 0.5097 & 0.2168 & \na & -- \\

Gradient Integral
& 0.1930 & 0.1150 & 0.5091 & 0.2251 & \na & -- \\

IGOS++
& 0.1530 & 0.1127 & 0.5051 & 0.1367 & 49.00 & 31.22 \\

RISE
& 0.2543 & 0.0602 & 0.5435 & 0.3443 & 5000.00 & 1.27 \\

HSIC
& 0.2365 & 0.0671 & 0.5210 & 0.3292 & 1536.00 & 4.73 \\

\midrule

Greedy
& 0.4831 & 0.0309 & 0.7440 & 0.6297 & 1748.26 & 2.81 \\

\rowcolor{oursrow}
PhaseWin
& 0.4264 & 0.0316 & 0.7175 & 0.5281 & 1014.20 & 4.20 \\

\bottomrule
\end{tabular}}
\vspace{-6pt}
\end{table}

\begin{table}[t]
\centering
\small
\setlength{\tabcolsep}{3.8pt}
\renewcommand\arraystretch{1.10}

\caption{Classification attribution on ImageNet misclassified samples with ResNet-101, using the model's wrong prediction as the attribution target.}
\label{tab:cls_resnet101_wrong_pred}

\resizebox{\columnwidth}{!}{%
\begin{tabular}{
l
!{\color{gray!35}\vrule}
cccccc
}
\toprule

\multirow{2}{*}{\textbf{Method}}
& \multicolumn{6}{c}{\textbf{ResNet-101}} \\

\cmidrule(lr){2-7}

& \shortstack{Ins.\\(↑)}
& \shortstack{Del.\\(↓)}
& \shortstack{Ave.~high.\\(↑)}
& \shortstack{@50\%\\(↑)}
& \shortstack{MEC\\(↓)}
& \shortstack{A--C\\(↑)} \\

\midrule

Gradient
& 0.2747 & 0.1737 & 0.5732 & 0.3770 & \na & -- \\

Gradient Integral
& 0.2928 & 0.1622 & 0.5761 & 0.4047 & \na & -- \\

IGOS++
& 0.2680 & 0.1695 & 0.5984 & 0.3525 & \na & -- \\

RISE
& 0.3733 & 0.1065 & 0.6315 & 0.5252 & 5000.00 & 0.74 \\

HSIC
& 0.3543 & 0.1067 & 0.6087 & 0.5203 & 1536.00 & 2.31 \\

\midrule

Greedy
& 0.6671 & 0.0559 & 0.8672 & 0.8399 & 1738.14 & 3.84 \\

\rowcolor{oursrow}
PhaseWin
& 0.6311 & 0.0575 & 0.8476 & 0.8004 & 987.42 & 6.39 \\

\bottomrule
\end{tabular}}
\vspace{-10pt}
\end{table}

Tables~\ref{tab:cls_vitl_wrong_pred}, \ref{tab:cls_rn101_wrong_pred}, and~\ref{tab:cls_resnet101_wrong_pred} evaluate misclassified samples using the model's wrong prediction as the attribution target. This setting asks which regions support the erroneous decision actually made by the model. The task is easier than attributing toward the ground-truth class, because the selected target already corresponds to the model's dominant response.

The results again show a clear hierarchy. Gradient-based methods produce weak insertion scores and limited early recovery. Perturbation-based methods improve over gradients, but remain far below search-based methods. Greedy achieves the strongest raw faithfulness, while PhaseWin provides the closest efficient approximation. On CLIP ViT-L/14, PhaseWin obtains \(0.6421\) Insertion AUC and \(0.8105\) @50\%, compared with \(0.6837\) and \(0.8576\) from Greedy. On CLIP RN101, PhaseWin reaches \(0.4264\) Insertion AUC and \(0.5281\) @50\%, compared with Greedy's \(0.4831\) and \(0.6297\). On ResNet-101, PhaseWin obtains \(0.6311\) Insertion AUC and \(0.8004\) @50\%, close to Greedy's \(0.6671\) and \(0.8399\).

Averaged across the three wrong-prediction settings, PhaseWin retains \(92.68\%\) of Greedy's Insertion AUC, \(97.53\%\) of its Average Highest score, and \(91.91\%\) of its @50\% recovery. Meanwhile, MEC decreases from \(1747.39\) to \(1009.38\), giving a \(1.73\times\) speedup. The A--C ratio improves from \(3.51\) to \(5.62\). Although HSIC can sometimes obtain a competitive A--C value due to its fixed and smaller evaluation budget, its raw faithfulness is much lower than PhaseWin. Therefore, PhaseWin offers a more favorable balance when both explanation quality and evaluation cost are considered.

These results are important for failure diagnosis. In misclassified cases, the explanation target is no longer a correct semantic decision, but the model's own erroneous class preference. PhaseWin remains close to Greedy under this setting, suggesting that the phased window search can still track the dominant evidence used by the model even when the prediction itself is wrong.

\begin{table*}[t]
\centering
\small
\setlength{\tabcolsep}{5.5pt}
\renewcommand\arraystretch{1.15}
\caption{Comparison on three datasets for correctly detected or grounded samples using Grounding DINO.}
\label{tab:correct}
\vspace{-5pt}
\resizebox{\textwidth}{!}{%
\begin{tabular}{
   l l
  !{\color{gray!40}\vrule}
  S[table-format=1.4] S[table-format=1.4] S[table-format=1.4]
  S[table-format=1.4] S[table-format=1.4] S[table-format=1.4]
  S[table-format=1.4]
  !{\color{gray!40}\vrule}
  S[table-format=1.4] S[table-format=1.4]
  !{\color{gray!40}\vrule}
  c c
}
\toprule
\multirow{3}{*}{\textbf{Datasets}} & \multirow{3}{*}{\textbf{Methods}} &
\multicolumn{7}{c!{\color{gray!40}\vrule}}{Faithfulness Metrics} &
\multicolumn{2}{c!{\color{gray!40}\vrule}}{Location Metrics} &
\multicolumn{2}{c}{\textbf{Efficiency Metrics}} \\
\cmidrule(lr){3-9}\cmidrule(lr){10-11}\cmidrule(lr){12-13}
& &
\mc{\shortstack{Ins.\\(↑)}} &
\mc{\shortstack{Del.\\(↓)}} &
\mc{\shortstack{Ins.~(class)\\(↑)}} &
\mc{\shortstack{Del.~(class)\\(↓)}} &
\mc{\shortstack{Ins.~(IoU)\\(↑)}} &
\mc{\shortstack{Del.~(IoU)\\(↓)}} &
\mc{\shortstack{Ave.~high.\\score~(↑)}} &
\mc{\shortstack{Point\\Game~(↑)}} &
\mc{\shortstack{Energy\\PG~(↑)}} &
\mc{\shortstack{$\mathrm{MEC}_{\mathrm{ave}}$\\(↓)}} &
\mc{\shortstack{\textbf{A-C ratio}\\(↑)}} \\
\midrule
\multirow{9}{*}{\shortstack[l]{MS COCO\\(Detection task)}} & Grad-CAM
& 0.2436 & 0.1526 & 0.3064 & 0.2006 & 0.6229 & 0.5324 & 0.5904 & 0.1746 & 0.1463 & \na & \na \\
& SSGrad-CAM++ & 0.2107 & 0.1778 & 0.2639 & 0.2314 & 0.5981 & 0.5511 & 0.5886 & 0.1905 & 0.1293 & \na & \na \\
& RISE & 0.4412 & 0.0402 & 0.5081 & 0.0886 & 0.8396 & 0.3642 & 0.6215 & 0.9497 & 0.1850 & 5000 & 0.88 \\
& HSIC & 0.3776 & 0.0439 & 0.4382 & 0.0903 & 0.8301 & 0.3301 & 0.5862 & 0.7328 & 0.1861 & 1536 & 2.46 \\
& ODAM & 0.3103 & 0.0519 & 0.3655 & 0.0894 & 0.7869 & 0.3984 & 0.5865 & 0.5431 & 0.2034 & \na & \na \\
\cmidrule(lr){2-13}
& Greedy-50 & 0.5195 & 0.0375 & 0.5941 & 0.0835 & 0.8480 & 0.3044 & 0.6591 & 0.9841 & 0.2046 & 2548.8 & 2.04 \\
\rowcolor{oursrow}
\cellcolor{white} & PhaseWin-50 & 0.4785 & 0.0424 & 0.5562 & 0.0898 & 0.8323 & 0.3116 & 0.6353 & 0.9894 & 0.1843 & 536.8 & 8.92 \\
\cmidrule(lr){2-13}
& Greedy-100 & 0.5459 & 0.0375 & 0.6204 & 0.0882 & 0.8581 & 0.3300 & 0.6873 & 0.9894 & 0.2046 & 10100 & 0.54 \\
\rowcolor{oursrow}
\cellcolor{white} & PhaseWin-100 & 0.5141 & 0.0410 & 0.5890 & 0.0907 & 0.8505 & 0.3400 & 0.6644 & 0.9894 & 0.1628 & 2853.4 & 1.81 \\
\midrule
\multirow{9}{*}{\shortstack[l]{RefCOCO\\(REC task)}} & Grad-CAM
& 0.3749 & 0.4237 & 0.4658 & 0.5194 & 0.7516 & 0.7685 & 0.7481 & 0.2380 & 0.2171 & \na & \na \\
& SSGrad-CAM++ & 0.4113 & 0.3925 & 0.5008 & 0.4851 & 0.7700 & 0.7588 & 0.7561 & 0.2820 & 0.2262 & \na & \na \\
& RISE & 0.6178 & 0.1605 & 0.7033 & 0.3396 & 0.8606 & 0.5164 & 0.8471 & 0.9400 & 0.2870 & 5000 & 1.24 \\
& HSIC & 0.5491 & 0.1846 & 0.6295 & 0.3509 & 0.8504 & 0.5120 & 0.7739 & 0.7900 & 0.3190 & 1536 & 3.57 \\
& ODAM & 0.4778 & 0.2718 & 0.5620 & 0.3757 & 0.8217 & 0.6641 & 0.7425 & 0.6320 & 0.3529 & \na & \na \\
\cmidrule(lr){2-13}
& Greedy-50 & 0.7278 & 0.1240 & 0.7995 & 0.2473 & 0.8961 & 0.5053 & 0.8770 & 0.9580 & 0.3738 & 2290.6 & 3.18 \\
\rowcolor{oursrow}
\cellcolor{white} & PhaseWin-50 & 0.7013 & 0.1473 & 0.7794 & 0.2747 & 0.8862 & 0.5273 & 0.8654 & 0.9580 & 0.3530 & 630.1 & 11.13 \\
\cmidrule(lr){2-13}
& Greedy-100 & 0.7419 & 0.1250 & 0.8080 & 0.2457 & 0.9050 & 0.5103 & 0.8842 & 0.9460 & 0.3566 & 10100 & 0.73 \\
\rowcolor{oursrow}
\cellcolor{white} & PhaseWin-100 & 0.7377 & 0.1529 & 0.8046 & 0.2823 & 0.9054 & 0.5466 & 0.8813 & 0.9360 & 0.3076 & 3382.5 & 2.18 \\
\midrule
\multirow{9}{*}{\shortstack[l]{LVIS V1 (rare)\\(Zero-shot det. task)}} & Grad-CAM
& 0.1253 & 0.1294 & 0.1801 & 0.1814 & 0.5657 & 0.5910 & 0.3549 & 0.1151 & 0.0941 & \na & \na \\
& SSGrad-CAM++ & 0.1253 & 0.1254 & 0.1765 & 0.1775 & 0.5800 & 0.5691 & 0.3504 & 0.1091 & 0.0931 & \na & \na \\
& RISE & 0.2808 & 0.0289 & 0.3348 & 0.0835 & 0.8303 & 0.3174 & 0.4289 & 0.9697 & 0.1462 & 5000 & 0.56 \\
& HSIC & 0.2417 & 0.0353 & 0.2912 & 0.0928 & 0.8187 & 0.3550 & 0.4044 & 0.8303 & 0.1730 & 1536 & 1.57 \\
& ODAM & 0.2009 & 0.0410 & 0.2478 & 0.0844 & 0.7770 & 0.4082 & 0.3694 & 0.6061 & 0.2050 & \na & \na \\
\cmidrule(lr){2-13}
& Greedy-50 & 0.3411 & 0.0265 & 0.3995 & 0.0805 & 0.8372 & 0.2986 & 0.4654 & 0.9939 & 0.1439 & 2544.6 & 1.34 \\
\rowcolor{oursrow}
\cellcolor{white} & PhaseWin-50 & 0.3071 & 0.0303 & 0.3645 & 0.0893 & 0.8245 & 0.3097 & 0.4325 & 0.9939 & 0.1369 & 465.9 & 6.59 \\
\cmidrule(lr){2-13}
& Greedy-100 & 0.3695 & 0.0277 & 0.4275 & 0.0799 & 0.8479 & 0.3242 & 0.4969 & 0.9758 & 0.1785 & 10100 & 0.37 \\
\rowcolor{oursrow}
\cellcolor{white} & PhaseWin-100 & 0.3363 & 0.0309 & 0.3944 & 0.0839 & 0.8379 & 0.3374 & 0.4688 & 0.9697 & 0.1175 & 2726.8 & 1.23 \\
\bottomrule
\end{tabular}}
\vspace{-10pt}
\end{table*}

\subsubsection{Failure Attribution Toward the Ground-Truth Class}
\begin{table}[t]
\centering
\small
\setlength{\tabcolsep}{3.8pt}
\renewcommand\arraystretch{1.10}

\caption{Classification attribution on ImageNet misclassified samples with CLIP ViT-L/14, using the ground-truth class as the attribution target.}
\label{tab:cls_vitl_wrong_gt}

\resizebox{\columnwidth}{!}{%
\begin{tabular}{
l
!{\color{gray!35}\vrule}
cccccc
}
\toprule

\multirow{2}{*}{\textbf{Method}}
& \multicolumn{6}{c}{\textbf{CLIP ViT-L/14}} \\

\cmidrule(lr){2-7}

& \shortstack{Ins.\\(↑)}
& \shortstack{Del.\\(↓)}
& \shortstack{Ave.~high.\\(↑)}
& \shortstack{@50\%\\(↑)}
& \shortstack{MEC\\(↓)}
& \shortstack{A--C\\(↑)} \\

\midrule

Gradient
& 0.0996 & 0.0968 & 0.3251 & 0.1852 & \na & -- \\

Gradient Integral
& 0.0931 & 0.1042 & 0.3169 & 0.1653 & \na & -- \\

Grad-ECLIP
& 0.1832 & 0.0508 & 0.4475 & 0.3898 & \na & -- \\

IGOS++
& 0.1181 & 0.0875 & 0.3361 & 0.2226 & \na & -- \\

RISE
& 0.1990 & 0.0516 & 0.4541 & 0.3774 & 5000.00 & 0.91 \\

HSIC
& 0.1447 & 0.0642 & 0.3620 & 0.3227 & 1536.00 & 2.36 \\

\midrule

Greedy
& 0.4827 & 0.0228 & 0.7425 & 0.7161 & 1755.77 & 4.23 \\

\rowcolor{oursrow}
PhaseWin
& 0.4291 & 0.0243 & 0.6989 & 0.6458 & 1046.46 & 6.68 \\

\bottomrule
\end{tabular}}
\vspace{-6pt}
\end{table}

\begin{table}[t]
\centering
\small
\setlength{\tabcolsep}{3.8pt}
\renewcommand\arraystretch{1.10}

\caption{Classification attribution on ImageNet misclassified samples with CLIP RN101, using the ground-truth class as the attribution target.}
\label{tab:cls_rn101_wrong_gt}

\resizebox{\columnwidth}{!}{%
\begin{tabular}{
l
!{\color{gray!35}\vrule}
cccccc
}
\toprule

\multirow{2}{*}{\textbf{Method}}
& \multicolumn{6}{c}{\textbf{CLIP RN101}} \\

\cmidrule(lr){2-7}

& \shortstack{Ins.\\(↑)}
& \shortstack{Del.\\(↓)}
& \shortstack{Ave.~high.\\(↑)}
& \shortstack{@50\%\\(↑)}
& \shortstack{MEC\\(↓)}
& \shortstack{A--C\\(↑)} \\

\midrule

Gradient
& 0.0605 & 0.0349 & 0.1849 & 0.1043 & \na & -- \\

Gradient Integral
& 0.0615 & 0.0361 & 0.1832 & 0.1048 & \na & -- \\

IGOS++
& 0.0528 & 0.0343 & 0.1839 & 0.0729 & \na & -- \\

RISE
& 0.0975 & 0.0177 & 0.2433 & 0.1793 & 5000.00 & 0.48 \\

HSIC
& 0.0795 & 0.0206 & 0.2074 & 0.1581 & 1536.00 & 1.35 \\

\midrule

Greedy
& 0.2747 & 0.0105 & 0.4802 & 0.4184 & 1748.26 & 2.75 \\

\rowcolor{oursrow}
PhaseWin
& 0.2193 & 0.0107 & 0.4267 & 0.3133 & 835.51 & 5.11 \\

\bottomrule
\end{tabular}}
\vspace{-6pt}
\end{table}

\begin{table}[t]
\centering
\small
\setlength{\tabcolsep}{3.8pt}
\renewcommand\arraystretch{1.10}

\caption{Classification attribution on ImageNet misclassified samples with ResNet-101, using the ground-truth class as the attribution target.}
\label{tab:cls_resnet101_wrong_gt}

\resizebox{\columnwidth}{!}{%
\begin{tabular}{
l
!{\color{gray!35}\vrule}
cccccc
}
\toprule

\multirow{2}{*}{\textbf{Method}}
& \multicolumn{6}{c}{\textbf{ResNet-101}} \\

\cmidrule(lr){2-7}

& \shortstack{Ins.\\(↑)}
& \shortstack{Del.\\(↓)}
& \shortstack{Ave.~high.\\(↑)}
& \shortstack{@50\%\\(↑)}
& \shortstack{MEC\\(↓)}
& \shortstack{A--C\\(↑)} \\

\midrule

Gradient
& 0.0898 & 0.0508 & 0.2471 & 0.2007 & \na & -- \\

Gradient Integral
& 0.0917 & 0.0488 & 0.2471 & 0.2077 & \na & -- \\

IGOS++
& 0.1020 & 0.0487 & 0.2725 & 0.1951 & \na & -- \\

RISE
& 0.1456 & 0.0318 & 0.3307 & 0.2943 & 5000.00 & 0.66 \\

HSIC
& 0.1102 & 0.0373 & 0.2751 & 0.2514 & 1536.00 & 1.79 \\

\midrule

Greedy
& 0.4358 & 0.0182 & 0.6606 & 0.6414 & 1738.14 & 3.80 \\

\rowcolor{oursrow}
PhaseWin
& 0.3869 & 0.0185 & 0.6101 & 0.5751 & 955.35 & 6.39 \\

\bottomrule
\end{tabular}}
\vspace{-6pt}
\end{table}

Tables~\ref{tab:cls_vitl_wrong_gt}, \ref{tab:cls_rn101_wrong_gt}, and~\ref{tab:cls_resnet101_wrong_gt} report the complementary failure setting, where the attribution target is the ground-truth class rather than the model's wrong prediction. This is a more difficult diagnostic task. Since the model did not select the ground-truth class, the target response is weaker, and the attribution method must recover evidence that is present in the image but insufficiently used by the model.

All methods show lower absolute scores in this setting, but the relative pattern remains stable. Greedy still provides the strongest raw faithfulness, and PhaseWin remains the best efficient approximation. On CLIP ViT-L/14, PhaseWin reaches \(0.4291\) Insertion AUC and \(0.6458\) @50\%, compared with \(0.4827\) and \(0.7161\) from Greedy. On CLIP RN101, PhaseWin obtains \(0.2193\) Insertion AUC and \(0.3133\) @50\%, while Greedy obtains \(0.2747\) and \(0.4184\). On ResNet-101, PhaseWin achieves \(0.3869\) Insertion AUC and \(0.5751\) @50\%, close to Greedy's \(0.4358\) and \(0.6414\).

Averaged over the three ground-truth failure settings, PhaseWin preserves \(86.77\%\) of Greedy's Insertion AUC, \(92.16\%\) of its Average Highest score, and \(86.39\%\) of its @50\% recovery. The retained fraction is lower than in the wrong-prediction setting, which is expected because the ground-truth class is not the model's dominant output. Nevertheless, PhaseWin reduces MEC from \(1747.39\) to \(945.77\), corresponding to a \(1.85\times\) speedup, and improves the A--C ratio from \(3.59\) to \(6.06\).

Compared with non-search baselines, PhaseWin remains substantially stronger in raw faithfulness. For example, on CLIP ViT-L/14, PhaseWin more than doubles the Insertion AUC of RISE under the ground-truth target setting. Similar margins are observed on CLIP RN101 and ResNet-101. This confirms that direct subset search is especially valuable for failure attribution, where weak target responses make dense saliency maps and random perturbation estimates less reliable.

\subsection{Detection and Grounding Attribution}
\label{sec:det_exp}
\begin{figure*}[t]
    \centering
    \includegraphics[width=\textwidth]{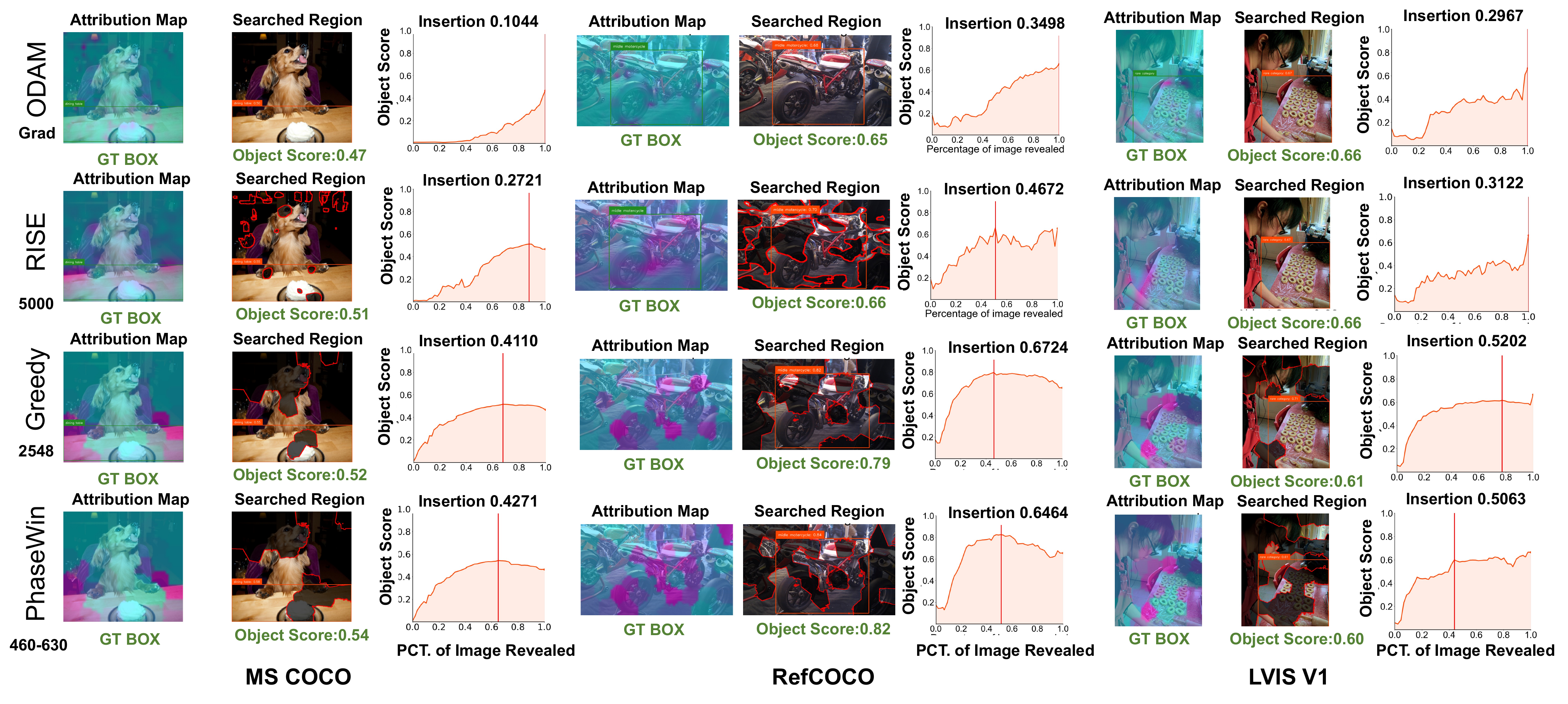}
    \vspace{-8pt}
    \caption{\textbf{Qualitative comparison of correct object-level attribution cases on MS COCO, RefCOCO, and LVIS v1.} Compared with ODAM and RISE, PhaseWin produces sharper and more faithful attributions. It matches or even exceeds Greedy in insertion AUC while requiring only a fraction of the computational budget.}
    \label{fig:vis}
    \vspace{-10pt}
\end{figure*}

\subsubsection{Summary Across Classification Settings}

Across all nine classification evaluations, covering three backbones and three target regimes, PhaseWin preserves \(92.51\%\) of Greedy's average Insertion AUC and \(96.62\%\) of its Average Highest score. At the same time, it reduces the average MEC from \(1744.55\) to \(955.13\), giving an overall \(1.83\times\) reduction in model evaluations. The average A--C ratio improves from \(2.51\) to \(4.16\).

These results support three conclusions. First, Greedy remains the strongest raw optimizer for the insertion/deletion protocol, which is consistent with its exhaustive candidate evaluation. Second, PhaseWin closely approximates Greedy across both successful and failed predictions, showing that the acceleration mechanism is not tied to a particular backbone or target class type. Third, the advantage of PhaseWin becomes more meaningful in diagnostic settings, where repeated attribution over many failure cases would make full Greedy search expensive.

\subsubsection{Qualitative Results}

Representative ImageNet visualizations are shown in Figure~\ref{fig:cls_vis}. The qualitative results follow the same trend as the quantitative tables. Gradient-based methods tend to produce diffuse maps, while perturbation-based methods often introduce noisy region responses. Greedy produces sharp and faithful region orderings but requires a large number of model evaluations. PhaseWin remains visually close to Greedy, while using substantially fewer forward passes. This confirms that the efficiency gain does not come from a qualitatively different attribution behavior, but from a more efficient approximation of the same region-selection process.

We next evaluate the original object-level setting, where attribution must explain both recognition and localization. We use Grounding DINO and Florence-2 as backbones, and consider three benchmarks: MS COCO for detection, RefCOCO for REC, and LVIS v1 rare categories for zero-shot detection. This part of the evaluation covers both correctly predicted and failure cases, making it our most comprehensive object-level benchmark. We compare against gradient-based baselines (Grad-CAM, SSGrad-CAM++, ODAM), perturbation-based baselines (RISE, HSIC), and Greedy, which is the quadratic search procedure accelerated by PhaseWin.

\subsubsection{Correct predictions on Grounding DINO}
Table~\ref{tab:correct} reports correct detection and grounding results on MS COCO, RefCOCO, and LVIS. Across all three datasets, PhaseWin preserves the ranking quality of Greedy while drastically reducing the number of model evaluations. On MS COCO with 50 regions, PhaseWin lowers $\mathrm{MEC}_{\mathrm{ave}}$ from 2548.8 to 536.8, a 4.7$\times$ reduction, while retaining an Insertion score of 0.4785 versus 0.5195 for Greedy. On RefCOCO, the same 50-region setting reduces the cost from 2290.6 to 630.1 while keeping Insertion at 0.7013 versus 0.7278. On LVIS rare categories, where all methods become weaker due to the long-tail distribution, PhaseWin still improves the A-C ratio from 1.34 to 6.59. These results show that the efficiency advantage of PhaseWin is strongest in the computationally heavy object-level regime, while the loss in faithfulness remains limited.

\begin{table*}[t]
\centering
\small
\setlength{\tabcolsep}{3.6pt}
\renewcommand\arraystretch{1.12}
\caption{Caption token attribution on COCO with Qwen2.5-VL-Instruct. PhaseWin is our method.}
\label{tab:caption_main}
\resizebox{\textwidth}{!}{%
\begin{tabular}{
   ll
  !{\color{gray!40}\vrule}
  *{6}{S[table-format=1.4]}
  !{\color{gray!40}\vrule}
  S[table-format=4.2]
  S[table-format=2.2]
}
\toprule
\multirow{2}{*}{\textbf{Model}} &
\multirow{2}{*}{\textbf{Method}} &
\multicolumn{6}{c!{\color{gray!40}\vrule}}{\textbf{Faithfulness Metrics}} &
\multicolumn{2}{c}{\textbf{Efficiency Metrics}} \\
\cmidrule(lr){3-8}\cmidrule(lr){9-10}
& &
\mc{\shortstack{Ins.\\(↑)}} &
\mc{\shortstack{Del.\\(↓)}} &
\mc{\shortstack{SensIns\\(↑)}} &
\mc{\shortstack{SensDel\\(↓)}} &
\mc{\shortstack{SensHigh\\(↑)}} &
\mc{\shortstack{Ave.~high.\\(↑)}} &
\mc{\shortstack{$\mathrm{MEC}_{\mathrm{ave}}$\\(↓)}} &
\mc{\shortstack{A-C ratio\\(↑)}} \\
\midrule

\multirow{6}{*}{Qwen2.5-VL-3B}
& Gradient        & 0.5365 & 0.5315 & 0.4397 & 0.4298 & 0.6575 & 0.6615 & \mc{\na} & \mc{\na} \\
& LLaVACAM        & 0.5248 & 0.5460 & 0.4184 & 0.4530 & 0.6547 & 0.6599 & \mc{\na} & \mc{\na} \\
& IGOS++          & 0.5376 & 0.5296 & 0.4388 & 0.4266 & 0.6574 & 0.6620 & \mc{\na}   & \mc{\na} \\
& RISE          & 0.5608 & 0.5087 & 0.4771 & 0.3893 & 0.6600 & 0.6645 & 5000.00 & 1.12 \\
& Greedy          & 0.6405 & 0.4372 & 0.5946 & 0.2858 & 0.6908 & 0.6951 & 4168.62 & 1.53 \\
\rowcolor{oursrow}
& PhaseWin (Ours) & 0.6351 & 0.4522 & 0.5736 & 0.3052 & 0.6786 & 0.6835 & 1412.71 & 4.80 \\

\midrule

\multirow{6}{*}{Qwen2.5-VL-7B}
& Gradient        & 0.5279 & 0.5248 & 0.4210 & 0.4149 & 0.6842 & 0.6791 & \mc{\na} & \mc{\na} \\
& LLaVACAM        & 0.5340 & 0.5357 & 0.4347 & 0.4362 & 0.6890 & 0.6824 & \mc{\na} & \mc{\na} \\
& IGOS++          & 0.5350 & 0.5219 & 0.4313 & 0.4097 & 0.6874 & 0.6816 & \mc{\na} & \mc{\na} \\
& RISE          & 0.5540 & 0.5019 & 0.4610 & 0.3771 & 0.6868 & 0.6821 & 5000.00 & 1.11 \\
& Greedy          & 0.6284 & 0.4350 & 0.5721 & 0.2815 & 0.7081 & 0.7064 & 2931.52 & 2.14 \\
\rowcolor{oursrow}
& PhaseWin (Ours) & 0.6155 & 0.4467 & 0.5535 & 0.2968 & 0.7015 & 0.6995 & 1401.60 & 4.39 \\

\bottomrule
\end{tabular}}
\vspace{-6pt}
\end{table*}

\begin{table}[t]
\centering
\small
\setlength{\tabcolsep}{5.5pt}
\renewcommand\arraystretch{1.15}
\caption{Evaluation of faithfulness (Insertion/Deletion AUC) and efficiency metrics on MS COCO and RefCOCO validation sets (Florence-2).}
\label{tab:vps_t2_aligned_eff}
\resizebox{\columnwidth}{!}{%
\begin{tabular}{
   l l
  !{\color{gray!40}\vrule}
  S[table-format=1.4] S[table-format=1.4]
  !{\color{gray!40}\vrule}
  c c
}
\toprule
\multirow{2}{*}{\textbf{Datasets}} & \multirow{2}{*}{\textbf{Methods}} &
\multicolumn{2}{c!{\color{gray!40}\vrule}}{Faithfulness Metrics} &
\multicolumn{2}{c}{\textbf{Efficiency Metrics}} \\
\cmidrule(lr){3-4}\cmidrule(lr){5-6}
& &
\mc{\shortstack{Insertion\\(↑)}} &
\mc{\shortstack{Deletion\\(↓)}} &
\mc{\shortstack{$\mathrm{MEC}_{\mathrm{ave}}$\\(↓)}} &
\mc{\shortstack{\textbf{A-C ratio}\\(↑)}} \\
\midrule
\multirow{4}{*}{\shortstack[l]{MS COCO\\(Detection task)}}
& RISE & 0.7477 & 0.0972 & 5000 & 1.50 \\
& HSIC & 0.5345 & 0.2730 & 1536 & 3.48 \\
\cmidrule(lr){2-6}
& Greedy-50 & 0.7678 & 0.0550 & 2548.1 & 2.98 \\
& \cellcolor{oursrow} PhaseWin-50 & \cellcolor{oursrow} 0.7615 & \cellcolor{oursrow} 0.0474 & \cellcolor{oursrow} 2184.1 & \cellcolor{oursrow} 3.49 \\
\midrule
\multirow{4}{*}{\shortstack[l]{RefCOCO\\(REC task)}}
& RISE & 0.7922 & 0.3505 & 5000 & 1.24 \\
& HSIC & 0.7639 & 0.3560 & 1536 & 3.57 \\
\cmidrule(lr){2-6}
& Greedy-50 & 0.8301 & 0.1159 & 2547.8 & 3.25 \\
& \cellcolor{oursrow} PhaseWin-50 & \cellcolor{oursrow} 0.8312 & \cellcolor{oursrow} 0.1205 & \cellcolor{oursrow} 2349.1 & \cellcolor{oursrow} 3.53 \\
\bottomrule
\end{tabular}}
\vspace{-6pt}
\end{table}

\subsubsection{Cross-backbone transfer to Florence-2}
Table~\ref{tab:vps_t2_aligned_eff} verifies that the same pattern transfers to Florence-2. PhaseWin nearly matches Greedy on both MS COCO and RefCOCO, with Insertion scores of 0.7615 vs.\ 0.7678 on COCO and 0.8312 vs.\ 0.8301 on RefCOCO. Compared with RISE, it achieves better faithfulness at lower cost, and compared with HSIC, it provides much stronger faithfulness with a modest increase in MEC. The speedup is smaller than with Grounding DINO, which is consistent with the weaker local submodularity of Florence-2, but PhaseWin still yields a very strong overall faithfulness--cost trade-off.

\subsubsection{Failure analysis}
Search efficiency is especially important when explanations are needed for debugging. We therefore evaluate mis-grounded, misclassified, and undetected cases separately.

\begin{figure*}[t]
    \centering
    \includegraphics[width=\textwidth]{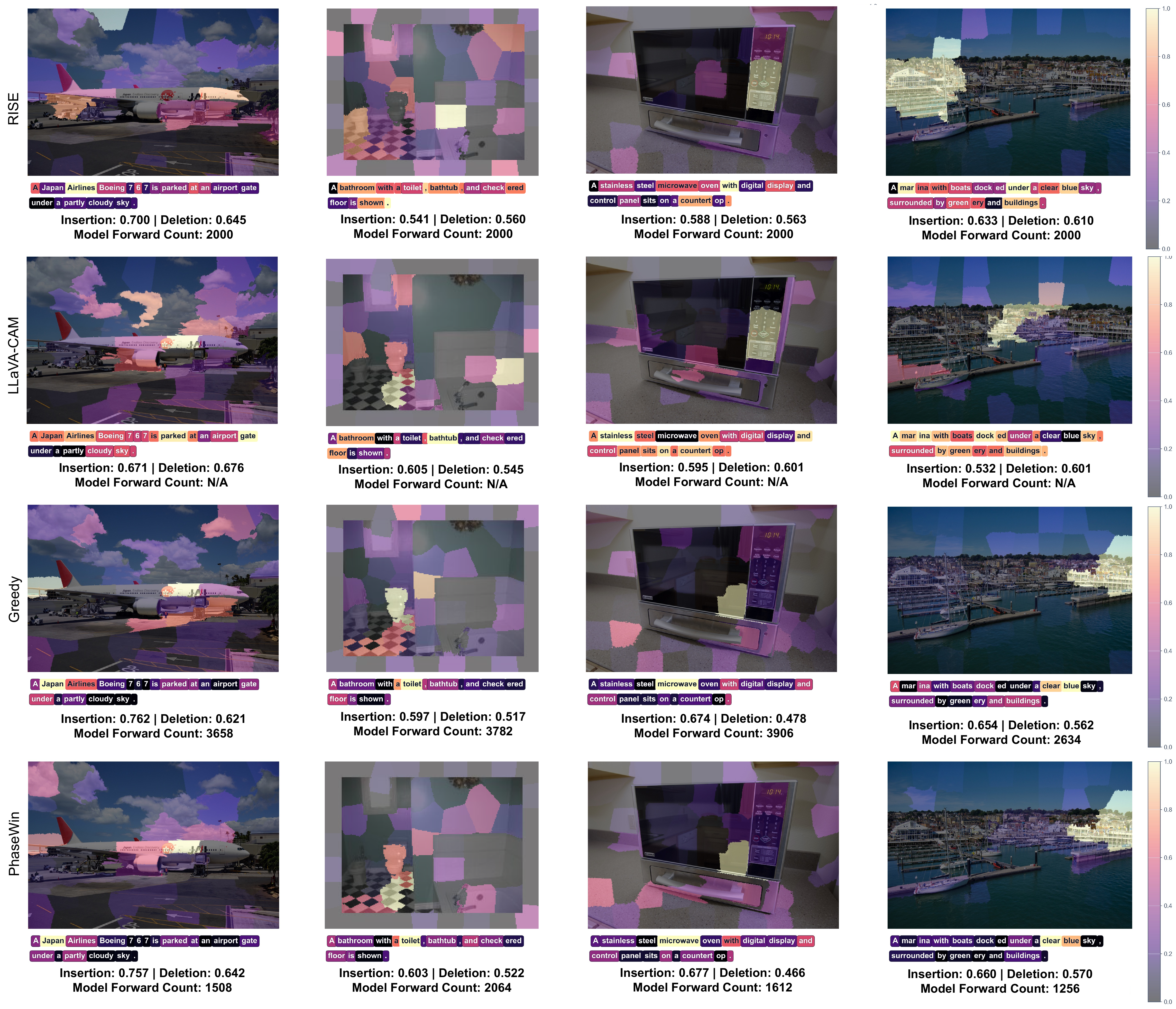}
    \vspace{-8pt}
    \caption{\textbf{Qualitative comparison on caption token attribution (Qwen2.5-VL-7B-Instruct).} Each column shows one image-caption example, and each row compares one attribution method.
The highlighted tokens define the attribution target, and the overlays show the ranked image regions supporting those tokens.
PhaseWin closely follows Greedy in token-relevant visual evidence while using far fewer model forward evaluations than exhaustive Greedy search.}
    \label{fig:caption_vis}
    \vspace{-10pt}
\end{figure*}

\textbf{REC failures.}
Table~\ref{tab:refcoco_faith_eff} shows that search-based attribution remains clearly superior when Grounding DINO grounds the wrong object. PhaseWin-100 sligtly surpasses Greedy-100 in Insertion (0.5047 vs.\ 0.4981) and class-based Insertion (0.6023 vs.\ 0.5990) while reducing $\mathrm{MEC}_{\mathrm{ave}}$ from 10100 to 3164.4. The 50-region version trades a small amount of faithfulness for a large efficiency gain, reaching an A-C ratio of 10.48.

\begin{table}[t]
\centering
\small
\setlength{\tabcolsep}{5.5pt}
\renewcommand\arraystretch{1.15}
\caption{RefCOCO (REC task): Faithfulness metrics and efficiency (Grounding DINO).}
\label{tab:refcoco_faith_eff}
\resizebox{\columnwidth}{!}{%
\begin{tabular}{
   l l
  !{\color{gray!40}\vrule}
  S[table-format=1.4] S[table-format=1.4] S[table-format=1.4]
  !{\color{gray!40}\vrule}
  c c
}
\toprule
\multirow{2}{*}{\textbf{Datasets}} & \multirow{2}{*}{\textbf{Methods}} &
\multicolumn{3}{c!{\color{gray!40}\vrule}}{Faithfulness Metrics} &
\multicolumn{2}{c}{\textbf{Efficiency Metrics}} \\
\cmidrule(lr){3-5}\cmidrule(lr){6-7}
& &
\mc{\shortstack{Ins.\\(↑)}} &
\mc{\shortstack{Ins.~(class)\\(↑)}} &
\mc{\shortstack{Ave.~high.\\score~(↑)}} &
\mc{\shortstack{$\mathrm{MEC}_{\mathrm{ave}}$\\(↓)}} &
\mc{\shortstack{\textbf{A-C ratio}\\(↑)}} \\
\midrule
\multirow{8}{*}{\shortstack[l]{RefCOCO\\(REC task)}}
& Grad-CAM & 0.1536 & 0.2794 & 0.3295 & \na & \na \\
& SSGrad-CAM++ & 0.1590 & 0.2837 & 0.3266 & \na & \na \\
& RISE & 0.3486 & 0.4787 & 0.6096 & 5000 & 1.21 \\
& HSIC & 0.2274 & 0.3488 & 0.4495 & 1536 & 2.92 \\
& ODAM & 0.1793 & 0.3001 & 0.3453 & \na & \na \\
\cmidrule(lr){2-7}
& Greedy-100 & 0.4981 & 0.5990 & 0.7007 & 10100 & 0.69 \\
& \cellcolor{oursrow} PhaseWin-50 & \cellcolor{oursrow} 0.4455 & \cellcolor{oursrow} 0.5537 & \cellcolor{oursrow} 0.6437 & \cellcolor{oursrow} 614.4 & \cellcolor{oursrow} 10.48 \\
& \cellcolor{oursrow} PhaseWin-100 & \cellcolor{oursrow} 0.5047 & \cellcolor{oursrow} 0.6023 & \cellcolor{oursrow} 0.7116 & \cellcolor{oursrow} 3164.4 & \cellcolor{oursrow} 2.25 \\
\bottomrule
\end{tabular}}
\vspace{-8pt}
\end{table}

\textbf{Misclassified detections.}
Table~\ref{tab:miscls_faith_eff} reports MS COCO and LVIS samples whose target objects are detected but assigned the wrong category. Gradient-based baselines are weak in this regime, and RISE/HSIC improve only moderately. Greedy remains the strongest raw baseline, while PhaseWin stays close with far fewer evaluations. On MS COCO, PhaseWin-50 uses only 477.3 evaluations instead of 10100 and improves the A-C ratio from 0.45 to 7.90. On LVIS rare categories, PhaseWin-50 reduces the cost by more than $20\times$ and raises the A-C ratio from 0.26 to 5.20.

\begin{table}[t]
\centering
\small
\setlength{\tabcolsep}{5.5pt}
\renewcommand\arraystretch{1.15}
\caption{MS COCO and LVIS (misclassified samples): Faithfulness metrics and efficiency (Grounding DINO).}
\label{tab:miscls_faith_eff}
\resizebox{\columnwidth}{!}{%
\begin{tabular}{
   l l
  !{\color{gray!40}\vrule}
  S[table-format=1.4] S[table-format=1.4] S[table-format=1.4] c
  !{\color{gray!40}\vrule}
  c c
}
\toprule
\multirow{2}{*}{\textbf{Datasets}} & \multirow{2}{*}{\textbf{Methods}} &
\multicolumn{4}{c!{\color{gray!40}\vrule}}{Faithfulness Metrics} &
\multicolumn{2}{c}{\textbf{Efficiency Metrics}} \\
\cmidrule(lr){3-6}\cmidrule(lr){7-8}
& &
\mc{\shortstack{Ins.\\(↑)}} &
\mc{\shortstack{Ins.~(class)\\(↑)}} &
\mc{\shortstack{Ave.~high.\\score~(↑)}} &
\mc{\shortstack{ESR\\(↑)}} &
\mc{\shortstack{$\mathrm{MEC}_{\mathrm{ave}}$\\(↓)}} &
\mc{\shortstack{\textbf{A-C ratio}\\(↑)}} \\
\midrule
\multirow{8}{*}{\shortstack[l]{MS COCO\\(Detection task)}}
& Grad-CAM & 0.1091 & 0.1478 & 0.3102 & 38.38\% & \na & \na \\
& SSGrad-CAM++ & 0.0960 & 0.1336 & 0.2952 & 33.51\% & \na & \na \\
& RISE & 0.2170 & 0.2661 & 0.3603 & 50.26\% & 5000 & 0.72 \\
& HSIC & 0.1771 & 0.2161 & 0.3143 & 34.59\% & 1536 & 2.04 \\
& ODAM & 0.1129 & 0.1486 & 0.2869 & 32.97\% & \na & \na \\
\cmidrule(lr){2-8}
& Greedy-100 & 0.3357 & 0.3967 & 0.4591 & 69.73\% & 10100 & 0.45 \\
& \cellcolor{oursrow} PhaseWin-50 & \cellcolor{oursrow} 0.2614 & \cellcolor{oursrow} 0.3198 & \cellcolor{oursrow} 0.3770 & \cellcolor{oursrow} 51.35\% & \cellcolor{oursrow} 477.3 & \cellcolor{oursrow} 7.90 \\
& \cellcolor{oursrow} PhaseWin-100 & \cellcolor{oursrow} 0.3018 & \cellcolor{oursrow} 0.3583 & \cellcolor{oursrow} 0.4289 & \cellcolor{oursrow} 63.78\% & \cellcolor{oursrow} 2595.0 & \cellcolor{oursrow} 1.65 \\
\midrule
\multirow{8}{*}{\shortstack[l]{LVIS V1 (rare)\\(Zero-shot det. task)}}
& Grad-CAM & 0.0503 & 0.0891 & 0.1564 & 12.50\% & \na & \na \\
& SSGrad-CAM++ & 0.0574 & 0.0946 & 0.1580 & 11.84\% & \na & \na \\
& RISE & 0.1245 & 0.1647 & 0.2088 & 28.95\% & 5000 & 0.41 \\
& HSIC & 0.0963 & 0.1247 & 0.1748 & 16.45\% & 1536 & 1.14 \\
& ODAM & 0.0575 & 0.0954 & 0.1520 & 9.21\% & \na & \na \\
\cmidrule(lr){2-8}
& Greedy-100 & 0.1776 & 0.2190 & 0.2606 & 43.29\% & 10100 & 0.26 \\
& \cellcolor{oursrow} PhaseWin-50 & \cellcolor{oursrow} 0.1394 & \cellcolor{oursrow} 0.1817 & \cellcolor{oursrow} 0.2119 & \cellcolor{oursrow} 36.63\% & \cellcolor{oursrow} 426.5 & \cellcolor{oursrow} 5.20 \\
& \cellcolor{oursrow} PhaseWin-100 & \cellcolor{oursrow} 0.1475 & \cellcolor{oursrow} 0.1845 & \cellcolor{oursrow} 0.2296 & \cellcolor{oursrow} 39.47\% & \cellcolor{oursrow} 2204.8 & \cellcolor{oursrow} 1.04 \\
\bottomrule
\end{tabular}}
\vspace{-8pt}
\end{table}

\textbf{Undetected instances.}
Table~\ref{tab:undetected_faith_eff} studies the hardest failure mode, where the object is not detected at all. Even in this setting, search-based attribution is substantially more informative than gradient-based or random-mask alternatives. On MS COCO, PhaseWin-100 sligtly exceeds Greedy-100 in Insertion (0.2156 vs.\ 0.2102) and ESR (44.44\% vs.\ 41.33\%) while using 4.7$\times$ fewer evaluations. On LVIS, PhaseWin again provides the best cost-effective trade-off, especially in the 50-region setting.

\begin{table}[t]
\centering
\small
\setlength{\tabcolsep}{5.5pt}
\renewcommand\arraystretch{1.15}
\caption{MS COCO and LVIS (undetected failure samples): Faithfulness metrics and efficiency (Grounding DINO).}
\label{tab:undetected_faith_eff}
\resizebox{\columnwidth}{!}{%
\begin{tabular}{
   l l
  !{\color{gray!40}\vrule}
  S[table-format=1.4] S[table-format=1.4] S[table-format=1.4] c
  !{\color{gray!40}\vrule}
  c c
}
\toprule
\multirow{2}{*}{\textbf{Datasets}} & \multirow{2}{*}{\textbf{Methods}} &
\multicolumn{4}{c!{\color{gray!40}\vrule}}{Faithfulness Metrics} &
\multicolumn{2}{c}{\textbf{Efficiency Metrics}} \\
\cmidrule(lr){3-6}\cmidrule(lr){7-8}
& &
\mc{\shortstack{Ins.\\(↑)}} &
\mc{\shortstack{Ins.~(class)\\(↑)}} &
\mc{\shortstack{Ave.~high.\\score~(↑)}} &
\mc{\shortstack{ESR\\(↑)}} &
\mc{\shortstack{$\mathrm{MEC}_{\mathrm{ave}}$\\(↓)}} &
\mc{\shortstack{\textbf{A-C ratio}\\(↑)}} \\
\midrule
\multirow{8}{*}{\shortstack[l]{MS COCO\\(Detection task)}}
& Grad-CAM & 0.0760 & 0.1321 & 0.2153 & 16.44\% & \na & \na \\
& SSGrad-CAM++ & 0.0671 & 0.1151 & 0.2124 & 16.44\% & \na & \na \\
& RISE & 0.1538 & 0.2260 & 0.2564 & 26.94\% & 5000 & 0.31 \\
& HSIC & 0.1101 & 0.1716 & 0.1945 & 13.56\% & 1536 & 1.43 \\
& ODAM & 0.0745 & 0.1350 & 0.2037 & 13.78\% & \na & \na \\
\cmidrule(lr){2-8}
& Greedy-100 & 0.2102 & 0.3011 & 0.3014 & 41.33\% & 10100 & 0.21 \\
& \cellcolor{oursrow} PhaseWin-50 & \cellcolor{oursrow} 0.1801 & \cellcolor{oursrow} 0.2641 & \cellcolor{oursrow} 0.2726 & \cellcolor{oursrow} 33.78\% & \cellcolor{oursrow} 427.8 & \cellcolor{oursrow} 6.37 \\
& \cellcolor{oursrow} PhaseWin-100 & \cellcolor{oursrow} 0.2156 & \cellcolor{oursrow} 0.3045 & \cellcolor{oursrow} 0.3289 & \cellcolor{oursrow} 44.44\% & \cellcolor{oursrow} 2160.2 & \cellcolor{oursrow} 1.52 \\
\midrule
\multirow{8}{*}{\shortstack[l]{LVIS V1 (rare)\\(Zero-shot det. task)}}
& Grad-CAM & 0.0291 & 0.0689 & 0.0901 & 5.43\% & \na & \na \\
& SSGrad-CAM++ & 0.0292 & 0.0680 & 0.0897 & 5.24\% & \na & \na \\
& RISE & 0.0703 & 0.1184 & 0.1312 & 18.73\% & 5000 & 0.26 \\
& HSIC & 0.0516 & 0.0920 & 0.1168 & 13.48\% & 1536 & 0.76 \\
& ODAM & 0.0283 & 0.0716 & 0.0851 & 4.68\% & \na & \na \\
\cmidrule(lr){2-8}
& Greedy-100 & 0.1155 & 0.1886 & 0.1784 & 30.15\% & 10100 & 0.18 \\
& \cellcolor{oursrow} PhaseWin-50 & \cellcolor{oursrow} 0.0787 & \cellcolor{oursrow} 0.1286 & \cellcolor{oursrow} 0.1309 & \cellcolor{oursrow} 17.04\% & \cellcolor{oursrow} 348.4 & \cellcolor{oursrow} 3.76 \\
& \cellcolor{oursrow} PhaseWin-100 & \cellcolor{oursrow} 0.0942 & \cellcolor{oursrow} 0.0069 & \cellcolor{oursrow} 0.1552 & \cellcolor{oursrow} 24.72\% & \cellcolor{oursrow} 1509.1 & \cellcolor{oursrow} 1.03 \\
\bottomrule
\end{tabular}}
\vspace{-8pt}
\end{table}

\subsubsection{Qualitative results}
Representative object-level cases are shown in Figure~\ref{fig:vis}. Compared with ODAM and RISE, PhaseWin produces sharper and less noisy attribution maps. It is visually close to Greedy, and in some examples even recovers a sligtly higher peak object score because the annealed search can escape locally suboptimal choices.

\subsection{Caption Token Attribution}
\label{sec:caption_exp}
We finally test PhaseWin in a generative MLLM setting. Following the token-level attribution protocol of EAGLE~\cite{chen2026where}, each example consists of an image and a pre-selected subset of generated caption tokens rather than a sentence-level metric. Concretely, given an image and the prefix preceding the selected tokens, we score a region subset by the mean probability of the selected generated tokens. We use Qwen2.5-VL-3B-Instruct and Qwen2.5-VL-7B-Instruct on 275 COCO validation images, partition each image into 64 SLICO superpixels, and compare PhaseWin with Greedy, RISE, vanilla gradients, and LLaVACAM. Sensitivity-aware AUC is computed with the default threshold of 0.2. In particular, because we were using the same samples to test two models of different sizes, we used SLICO-64 segmentation for the 3B model.

\subsubsection{Main results}
Table~\ref{tab:caption_main} shows that the search-based ordering paradigm also transfers to generative MLLM attribution. Greedy remains the strongest method, but PhaseWin is again the closest approximation while cutting the average evaluation count. It is worth noting that for both slice and slic segmentation, because slicing tends to cover fewer regions, the average number of forward passes for greedy operations surges dramatically (2931.52-4168.62) under slice, while PhaseWin remains stable (1401.60-1412.71). This also verifies our hypothesis that the number of forward passes in PhaseWin is essentially controlled by the number of features on which the model's decisions depend, and this does not change linearly with the increase in the number of segmented regions, unlike the quadratic complexity of greedy operations. 

\subsubsection{Qualitative results}
Figure~\ref{fig:caption_vis} is reserved for qualitative caption cases. This panel is particularly useful for showing whether PhaseWin preserves the token-specific evidence used by Greedy while suppressing background regions that are irrelevant to the selected caption tokens.

\subsection{Cross-task Discussion}
Across all three tasks, two consistent patterns emerge. First, explicit region search is more reliable than purely local saliency estimation when the evaluation target is defined by insertion/deletion replay. Second, PhaseWin is consistently the closest approximation to Greedy under a much smaller evaluation budget: it reduces $\mathrm{MEC}_{\mathrm{ave}}$ by about $1.9\times$ on classification, $2.8\times$ on caption attribution, and by $3.2\times$--$29.0\times$ on the most demanding Grounding DINO detection and grounding settings. The exact speedup depends on the local structure of the task-specific score function: it is largest for Grounding DINO, moderate for CLIP and Qwen2.5-VL, and smaller but still beneficial for Florence-2. Overall, these results support the main claim of this paper: PhaseWin is not a detector-specific heuristic, but a general acceleration mechanism for region-based attribution.

\subsection{Ablation Study}
\textbf{Speed--Accuracy Trade-off}
\begin{figure}[t]
    \centering
    \includegraphics[width=0.9\columnwidth]{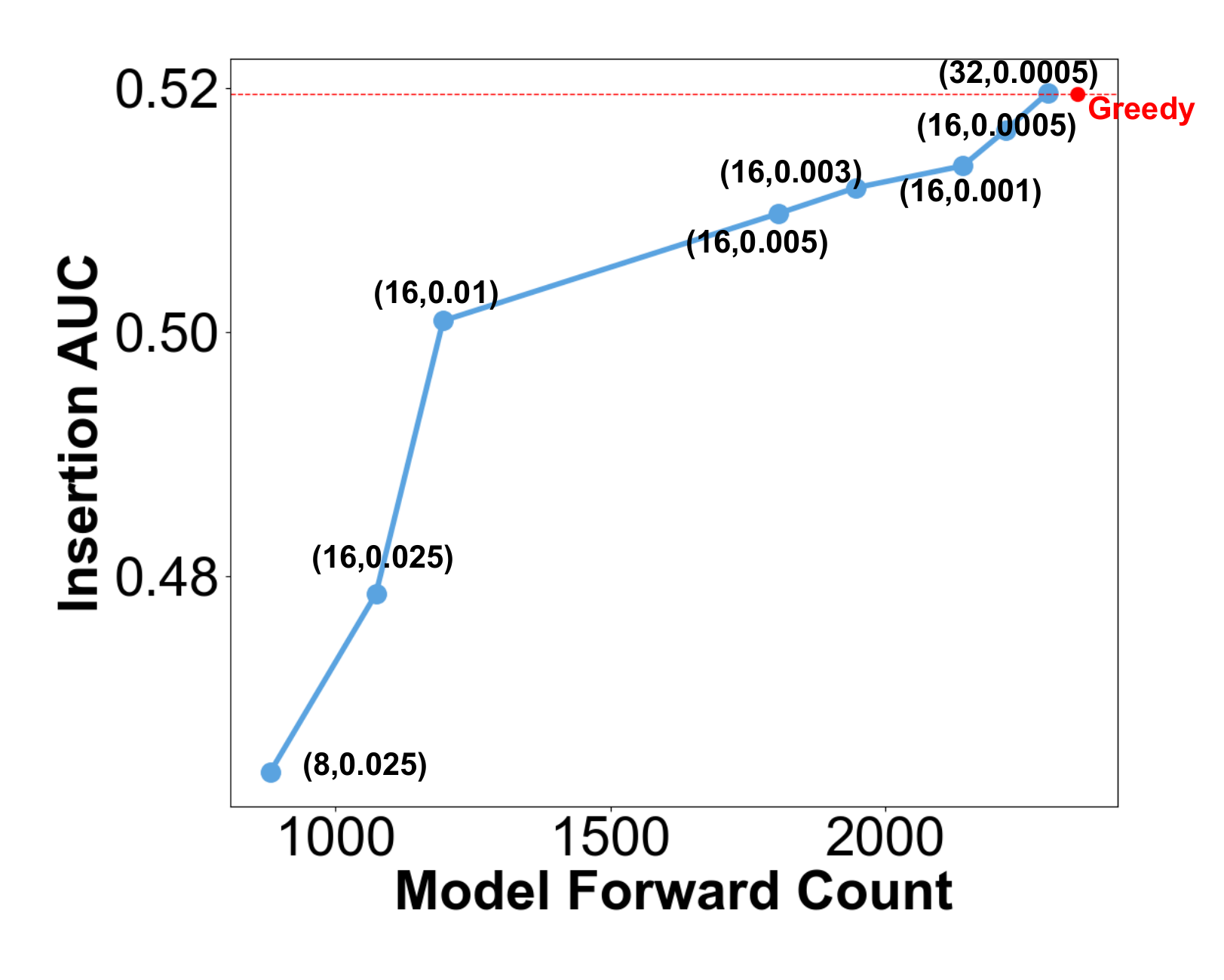}
    \vspace{-5pt}
    \caption{Trade-off between speed and precision.}
    \label{fig:af}
    \vspace{-4pt}
\end{figure}

Beyond fixed default settings, PhaseWin provides a tunable continuum between speed and faithfulness. Figure~\ref{fig:af} reports a representative object-level ablation by varying the window size and the phase-exit threshold. As the search budget increases, the insertion AUC rises monotonically and approaches the performance of Greedy. In the high-budget regime, the annealed deferral mechanism can even sligtly exceed Greedy by escaping locally suboptimal early choices. This behavior is particularly attractive in practice, because it allows users to adjust the computational budget without changing the attribution interface or the evaluation protocol.

\textbf{Boundary cases.}
We also analyze the rare cases where PhaseWin deviates noticeably from greedy search in Appendix~\ref{app:greedy_tail}. These samples typically show unstable local marginal rankings, where early region scores are not reliable indicators of later contribution. In the captioning task, we do not observe such pathological samples. In the classification task, they account for roughly \(13.1\%\) of the evaluated images; after excluding them, the faithfulness gap between PhaseWin and greedy is within \(2\%\). These results indicate that the remaining discrepancy is concentrated in a small set of ranking-unstable cases, while PhaseWin preserves greedy-level faithfulness on the overwhelming majority of samples.
\section{Conclusion}\label{sec:conclusion}
We presented PhaseWin, an efficient search algorithm for high-faithfulness visual 
attribution. Casting attribution as an ordered subset-search problem, PhaseWin 
replaces the quadratic global rescoring of greedy search with a phase-window 
strategy---anchor-based pruning, windowed refinement, dynamic supervision, and 
deferred selection---reducing model evaluations to near-linear scale while preserving 
greedy-like behavior.
Across image classification, object detection, referring expression comprehension, and 
image captioning, PhaseWin approaches the faithfulness of greedy search with 
substantially fewer forward evaluations. This advantage is not tied to a specific 
architecture, dataset, or task, but reflects a structure shared by high-faithfulness 
attribution: informative regions can be identified through a few locally reliable 
comparisons rather than repeated global rescoring. Our analysis makes this precise, 
establishing near-greedy guarantees under monotonicity and feature-level 
diminishing-return conditions and linking them to task-level objectives such as 
classification confidence, localization overlap, and response recovery.


\ifCLASSOPTIONcompsoc
\section*{Acknowledgments}
\else
\section*{Acknowledgment}
\fi



\input{arxiv.bbl}




\appendices
\setcounter{section}{0}
\renewcommand{\thesection}{\Alph{section}}
\section{Submodularity and Supermodularity}
\label{appendix:submodularity}

This appendix clarifies why the score optimized by greedy-based visual attribution should not be regarded as a globally submodular objective. The key point is not that redundancy or synergy never appears in visual attribution. Rather, the two-sided sufficiency--necessity score used in prior subset-search attribution has a structural symmetry that is incompatible with global submodularity, except in a degenerate modular case. Therefore, the empirical success of greedy search is better explained by the monotonic tendency of the attribution score and by the preference of insertion AUC for early-stage gain accumulation, rather than by a valid global submodular guarantee.

\subsection{Definitions}

Let \(V\) be a finite ground set of candidate regions and let
\(\mathcal{H}:2^V\to\mathbb{R}\) be a set function. For \(S\subseteq V\) and
\(i\in V\setminus S\), define the marginal gain
\[
    \Delta_{\mathcal{H}}(i\mid S)
    =
    \mathcal{H}(S\cup\{i\})-\mathcal{H}(S).
\]
For distinct \(i,j\in V\setminus S\), define the second-order discrete interaction
\[
    \Gamma_{\mathcal{H}}(i,j\mid S)
    =
    \mathcal{H}(S\cup\{i,j\})
    -\mathcal{H}(S\cup\{i\})
    -\mathcal{H}(S\cup\{j\})
    +\mathcal{H}(S).
\]

\begin{definition}[Submodularity]
A set function \(\mathcal{H}\) is \emph{submodular} if, for all
\(A\subseteq B\subseteq V\) and \(i\in V\setminus B\),
\[
    \Delta_{\mathcal{H}}(i\mid A)
    \ge
    \Delta_{\mathcal{H}}(i\mid B).
\]
Equivalently, for every \(S\subseteq V\setminus\{i,j\}\) and every distinct
\(i,j\in V\setminus S\),
\[
    \Gamma_{\mathcal{H}}(i,j\mid S)\le 0.
\]
Thus, submodularity corresponds to non-positive pairwise interaction, or diminishing returns.
\end{definition}

\begin{definition}[Supermodularity]
A set function \(\mathcal{H}\) is \emph{supermodular} if, for all
\(A\subseteq B\subseteq V\) and \(i\in V\setminus B\),
\[
    \Delta_{\mathcal{H}}(i\mid A)
    \le
    \Delta_{\mathcal{H}}(i\mid B).
\]
Equivalently,
\[
    \Gamma_{\mathcal{H}}(i,j\mid S)\ge 0,
    \qquad
    \forall S\subseteq V\setminus\{i,j\}.
\]
Thus, supermodularity corresponds to non-negative pairwise interaction, or increasing returns.
\end{definition}

\begin{definition}[Modularity]
A set function \(\mathcal{H}\) is \emph{modular} if both submodularity and supermodularity hold, equivalently
\[
    \Gamma_{\mathcal{H}}(i,j\mid S)=0,
    \qquad
    \forall S\subseteq V\setminus\{i,j\}.
\]
In this case, every region has a context-independent marginal contribution:
\[
    \mathcal{H}(S)
    =
    \mathcal{H}(\emptyset)
    +
    \sum_{i\in S} w_i .
\]
\end{definition}

\subsection{The Two-sided Attribution Score}

Let \(G(S)\) denote the task response obtained after inserting the region subset
\(S\) into the image. Depending on the task, \(G\) may represent a class score,
an object confidence, an IoU-aware detection score, or another target response.
The attribution score used by greedy subset search combines a sufficiency term
and a necessity term:
\[
    F(S)
    =
    G(S)
    +
    \bigl(G(V)-G(V\setminus S)\bigr).
    \label{eq:two_sided_score}
\]
The first term measures how much evidence is recovered when \(S\) is inserted.
The second term measures how much evidence is lost when \(S\) is removed from the full image.
The constant \(G(V)\) does not affect marginal comparisons, but it makes the necessity term directly comparable to the sufficiency term.

A basic property of \(F\) is its complement identity:
\[
    F(S)+F(V\setminus S)
    =
    F(V)+F(\emptyset)
    =
    2G(V),
    \qquad
    \forall S\subseteq V.
    \label{eq:complement_identity}
\]
Indeed,
\[
\begin{aligned}
    F(S)+F(V\setminus S)
    &=
    \Bigl(G(S)+G(V)-G(V\setminus S)\Bigr)
    \\&+
    \Bigl(G(V\setminus S)+G(V)-G(S)\Bigr)  \\
    &=
    2G(V).
\end{aligned}
\]

\begin{theorem}[Submodularity of \(F\) forces modularity]
Let \(F\) be the two-sided attribution score in Eq.~\eqref{eq:two_sided_score}.
If \(F\) is submodular, then \(F\) is modular.
\end{theorem}

\begin{proof}
For arbitrary \(A,B\subseteq V\), apply submodularity to the two complement sets
\(V\setminus A\) and \(V\setminus B\):
\[
    F(V\setminus A)+F(V\setminus B)
    \ge
    F\bigl(V\setminus(A\cup B)\bigr)
    +
    F\bigl(V\setminus(A\cap B)\bigr).
\]
Using the complement identity in Eq.~\eqref{eq:complement_identity}, this becomes
\[
    \bigl(C-F(A)\bigr)+\bigl(C-F(B)\bigr)
    \ge
    \bigl(C-F(A\cup B)\bigr)+\bigl(C-F(A\cap B)\bigr),
\]
where \(C=F(V)+F(\emptyset)\). After cancellation,
\[
    F(A)+F(B)
    \le
    F(A\cup B)+F(A\cap B).
\]
This is exactly supermodularity. Therefore, \(F\) is both submodular and supermodular. Hence it is modular.
\end{proof}

This theorem shows that \(F\) cannot be globally submodular unless it is in the degenerate modular case. In a modular score, every region has a fixed marginal contribution independent of the current selected context. Then subset search reduces to a one-shot sorting problem, and there is no genuine interaction among regions. This is incompatible with the empirical behavior of modern visual models, where redundant and synergistic region interactions are both observed.

The same conclusion can be expressed directly through second-order interactions. For the two-sided score in Eq.~\eqref{eq:two_sided_score}, we have
\[
    \Gamma_F(i,j\mid S)
    =
    \Gamma_G(i,j\mid S)
    -
    \Gamma_G\bigl(i,j\mid V\setminus(S\cup\{i,j\})\bigr).
    \label{eq:interaction_difference}
\]
Thus, \(F\) is modular only if every interaction of \(G\) in an insertion context is exactly matched by the corresponding interaction in the complementary deletion context:
\[
    \Gamma_G(i,j\mid S)
    =
    \Gamma_G\bigl(i,j\mid V\setminus(S\cup\{i,j\})\bigr),
    \qquad
    \forall S,i,j.
\]
This is an extremely restrictive cancellation condition. Therefore, once there exist
\(S\subseteq V\setminus\{i,j\}\) and distinct \(i,j\in V\setminus S\) such that
\[
    \Gamma_G(i,j\mid S)
    \neq
    \Gamma_G\bigl(i,j\mid V\setminus(S\cup\{i,j\})\bigr),
\]
the score \(F\) is not modular and hence, by the theorem above, cannot be submodular.

\subsection{Why \texorpdfstring{\(G\)}{G} Is Not Globally Submodular Either}

The task response \(G\) itself does not satisfy a global submodularity or supermodularity law. Global submodularity would require
\[
    \Gamma_G(i,j\mid S)\le 0,
    \qquad
    \forall S\subseteq V\setminus\{i,j\},
\]
meaning that all region interactions are redundant. Global supermodularity would require
\[
    \Gamma_G(i,j\mid S)\ge 0,
    \qquad
    \forall S\subseteq V\setminus\{i,j\},
\]
meaning that all region interactions are synergistic. Visual recognition models do not obey either uniform sign condition. Some regions provide overlapping evidence, producing negative interactions, while others only become informative when combined, producing positive interactions. Therefore, submodularity and supermodularity should be understood as local interaction modes rather than global properties of the attribution objective.

\subsection{Response-curve Diagnostics}

The insertion AUC is a scalar summary of a response curve. For an ordering
\(\pi=(s_1,\ldots,s_{|V|})\), define the prefix set
\[
    S_j^\pi=\{s_1,\ldots,s_j\}
\]
and the response value
\[
    r_j=\mathcal{H}(S_j^\pi).
\]
The insertion AUC is
\[
    \mathrm{AUC}(\pi)
    =
    \frac{1}{|V|}
    \sum_{j=1}^{|V|}
    r_j .
\]
Thus, AUC rewards orderings whose response values rise early.

For an exact greedy ordering under a monotone submodular function \(\mathcal{H}\), the response curve must be discretely concave. Let
\[
    d_j
    =
    \mathcal{H}(S_j)-\mathcal{H}(S_{j-1})
\]
be the greedy marginal gain at step \(j\). Since \(s_j\) is chosen by greedy,
\[
    d_j
    =
    \max_{i\in V\setminus S_{j-1}}
    \Delta_{\mathcal{H}}(i\mid S_{j-1}).
\]
For the next selected region \(s_{j+1}\), submodularity gives
\[
    \Delta_{\mathcal{H}}(s_{j+1}\mid S_j)
    \le
    \Delta_{\mathcal{H}}(s_{j+1}\mid S_{j-1}),
\]
and greedy optimality at step \(j\) gives
\[
    \Delta_{\mathcal{H}}(s_{j+1}\mid S_{j-1})
    \le
    \Delta_{\mathcal{H}}(s_j\mid S_{j-1}).
\]
Therefore,
\[
    d_{j+1}\le d_j.
\]
Hence, a monotone submodular objective necessarily yields a non-increasing sequence of greedy marginal gains. If the measured greedy response curve contains increasing marginal segments, then the underlying score cannot be globally submodular.

\subsection{Empirical Implications}

Figure~\ref{sub} visualizes insertion response curves under greedy subset search for Grounding DINO and Florence-2 on the same sample. Grounding DINO shows an approximately concave trend with local violations, indicating that redundancy dominates in some stages but does not define a global law. Florence-2 exhibits a strongly convex trend, suggesting substantial synergistic interactions. These two behaviors demonstrate that the attribution response is not governed by a universal submodular or supermodular structure.

Consequently, the role of greedy search in visual attribution should be interpreted carefully. Greedy is effective because insertion AUC favors early recovery of the target response, and greedy directly optimizes the immediate marginal gain at each step. However, its effectiveness does not imply that the underlying score is globally submodular. PhaseWin therefore accelerates greedy-style subset search without relying on an invalid global submodularity assumption.

\begin{figure}[ht]
    \centering
    \includegraphics[width=0.9\linewidth]{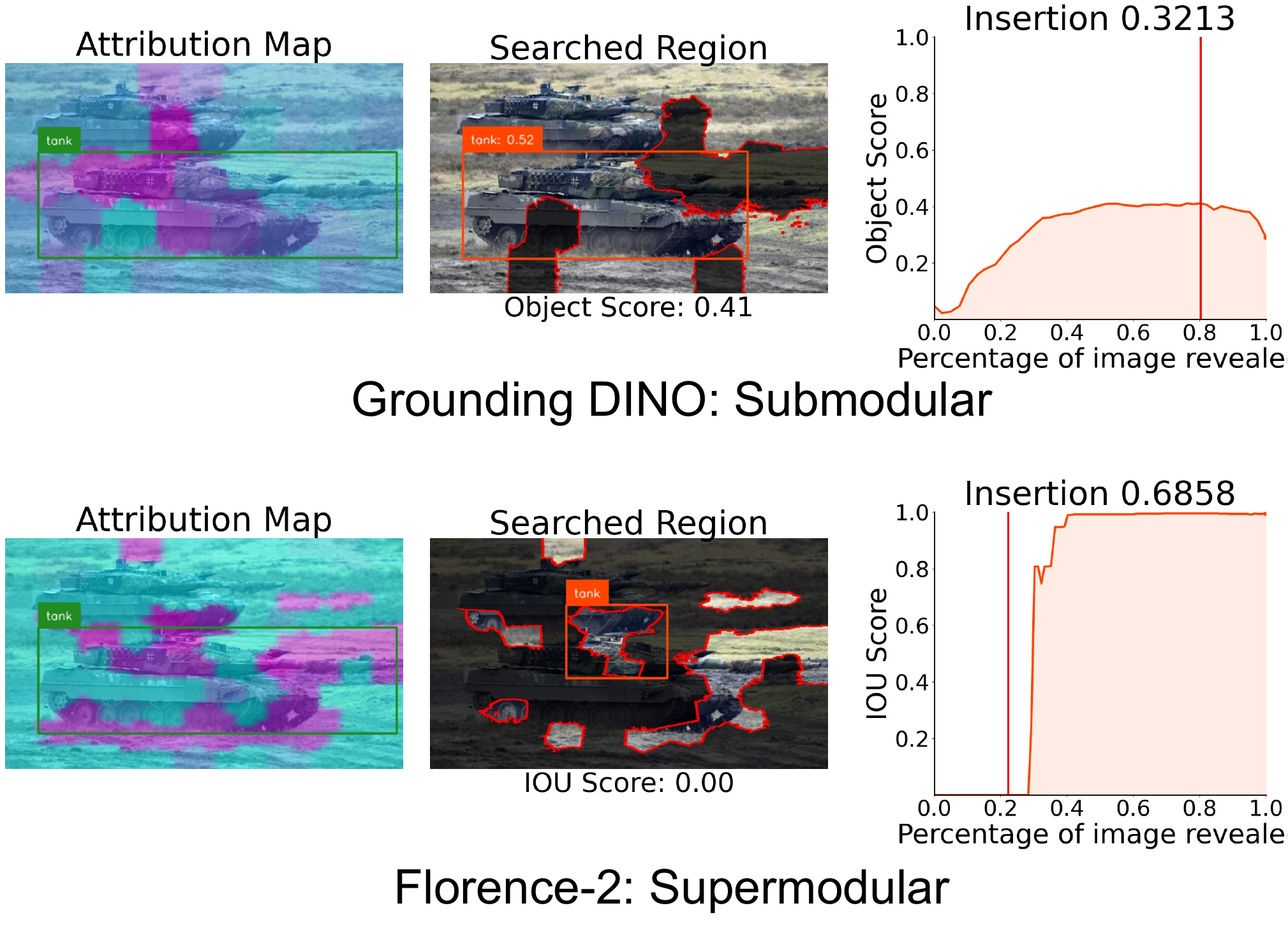}
    \caption{Insertion response curves under greedy subset search. Grounding DINO shows an approximately concave trend with local violations, whereas Florence-2 exhibits a convex, synergy-dominated trend. These curves indicate that the attribution score is not a globally submodular objective.}
    \label{sub}
\end{figure}

\clearpage
\onecolumn
\section{Notation}
\label{sec:notation}
For ease of reading, we have compiled a list of symbols that may be used in the main text and the appendix of the proof before we begin, for your reference.

\begin{table*}[h]
\caption{Structural notation table.}
\centering
\small
\setlength{\tabcolsep}{4pt}
\begin{tabular}{p{0.18\textwidth} p{0.2\textwidth} p{0.52\textwidth}}
\toprule
\textbf{Symbol} & \textbf{Scope} & \textbf{Meaning} \\
\midrule
$U$ & Problem setup & Finite ground set of candidate regions. \\
$n$ & Problem setup & Ground-set size, $n=|U|$. \\
$\mathcal H=\{H_1,\dots,H_q\}$ & Structure & Partition of $U$ into $q$ semantic blocks. \\
$H_j$ & Structure & The $j$-th block in $\mathcal H$. \\
$q$ & Structure & Number of semantic blocks. \\
$X\subseteq U$ & Structure & Subset of regions. \\
$B(X)$ & Structure & Activated block set: $\{j\in[q]:X\cap H_j\neq\emptyset\}$. \\
$\Phi$ & Structure & Block-level monotone submodular function on $2^{[q]}$. \\
$R$ & Structure & Residual function in decomposition $F(X)=\Phi(B(X))+R(X)$. \\
$\varepsilon_R$ & Structure & Uniform upper bound on residual marginal gains. \\
$\underline{\Delta}_\Phi$ & Structure & Minimum positive marginal gain of $\Phi$. \\
\midrule
$\Pi=(v_1,\dots,v_k)$ & Output & PhaseWin ordered output. \\
$S_i^{\mathrm{PW}}$ & Output & Prefix set $\{v_1,\dots,v_i\}$. \\
$\mathcal R_i$ & Algorithm & Live candidate pool before step $i$. \\
$\mathcal D_i$ & Algorithm & Candidates deleted at step $i$. \\
$a_i$ & Algorithm & Best live marginal before step $i$. \\
$\rho_{\mathrm{sel}}$ & Parameter & Selection threshold ratio. \\
$\rho_{\mathrm{del}}$ & Parameter & Deletion threshold ratio. \\
$\beta_i^\psi$ & Policy & Approximation ratio of window policy at step $i$. \\
$\kappa_i$ & Derived & Effective selection ratio $\kappa_i=\rho_{\mathrm{sel}}\beta_i^\psi$. \\
$\kappa_\psi$ & Derived & Uniform lower bound $\min_i \kappa_i$. \\
$C_k(\kappa_\psi,\rho_{\mathrm{del}})$ & Theorem~\ref{thm:finiteF} & Finite-cardinality approximation coefficient. \\
\midrule
$G:2^U\to\mathbb R_+$ & Objective & Attribution/evaluation set function. \\
$F:2^U\to\mathbb R$ & Objective & Search objective $F(X)=G(X)+G(U)-G(U\setminus X)$. \\
$\Delta_H(e\mid X)$ & General & Marginal gain $H(X\cup\{e\})-H(X)$. \\
$\pi$ & Ordering & Full ordering of $U$. \\
$\pi^{\mathrm{PW}}$ & Ordering & Full ordering induced by PhaseWin. \\
$P_t^\pi$ & Ordering & Prefix $\{\pi_1,\dots,\pi_t\}$. \\
$M_G(\pi)$ & Metric & Prefix maximum $\max_t G(P_t^\pi)$. \\
$\operatorname{AUC}_G(\pi)$ & Metric & Insertion AUC along ordering $\pi$. \\
$\operatorname{AUC}_G^\star$ & Metric & Optimal AUC over all orderings. \\
$\mathrm{OPT}_G$ & Metric & $\max_{X\subseteq U} G(X)$. \\
$\mathrm{OPT}_{G,k}$ & Metric & $\max_{|X|\le k} G(X)$. \\
\midrule
$\lambda, b$ & One-sided alignment & Constants in Assumption~\ref{ass:align}. \\
$\mu, b_+$ & Two-sided alignment & Constants in Assumption~\ref{ass:align}. \\
$\bar\delta$ & Weak submodularity & Submodularity violation bound of $G$. \\
\midrule
$d_i$ & Proof & Accepted marginal gain at step $i$. \\
$J_t$ & Proof & Activated block set $B(S_t^{\mathrm{PW}})$. \\
$j_t$ & Proof & Block activated at step $t$. \\
$j_t^\star$ & Proof & Best inactive block at step $t$. \\
$J_t^\star$ & Proof & Optimal block subset of size $\le t$. \\
$\mathrm{OPT}_t^\Phi$ & Proof & $\max_{|J|\le t}\Phi(J)$. \\
$X^\star$ & Proof & Optimal subset for $G$. \\
$r^\star$ & Proof & Number of activated blocks in $X^\star$. \\
$G_t^\pi$ & Proof & Prefix value $G(P_t^\pi)$. \\
$B_t^\pi$ & Proof & Prefix block set $B(P_t^\pi)$. \\
\bottomrule
\end{tabular}
\end{table*}

\clearpage
\twocolumn
\section{Necessity of the Alignment Assumption}
\label{app:necessity}

The structural assumptions on $F$ alone do not imply any nontrivial guarantee for the induced ordering on $G$.

\begin{proposition}
\label{prop:necessity}
Even if $G$ is exactly submodular, i.e., $\delta_G=0$, there is no universal constant $c>0$ such that
\begin{equation}
M_G(\pi^{\mathrm{PW}})
\ge
c\max_{X\subseteq U}G(X)
\end{equation}
or, in the equal-area setting $a_e\equiv 1$,
\begin{equation}
\operatorname{AUC}_G(\pi^{\mathrm{PW}})
\ge
c\,\operatorname{AUC}_G^\star
\end{equation}
can be guaranteed under Assumptions~\ref{ass:partition}, \ref{ass:windowfaithful}, and \ref{ass:align} alone.
\end{proposition}

\begin{proof}
Let $U=\{1,\dots,n\}$ and let $G$ be the cut function of the path graph $P_n$:
\begin{equation}
G(X)
:=
\#\bigl\{(u,v)\in E(P_n): |\{u,v\}\cap X|=1\bigr\}.
\end{equation}
This $G$ is nonnegative, normalized, and submodular. Since graph cut functions are symmetric,
\begin{equation}
G(U\setminus X)=G(X),
\;
G(U)=0.
\end{equation}
Therefore
\begin{equation}
F(X)
=
G(X)+G(U)-G(U\setminus X)
=
0,
\;
\forall X\subseteq U.
\end{equation}
Thus the search objective carries no ordering information; any tie-breaking order may be returned.

Consider the natural order $\pi_{\mathrm{nat}}=(1,2,\dots,n)$. Its prefixes are contiguous intervals, so
\begin{equation}
G(P_t^{\pi_{\mathrm{nat}}})=1,
\;
t=1,\dots,n-1,
\end{equation}
and $G(P_n^{\pi_{\mathrm{nat}}})=0$. By contrast,
\begin{equation}
\max_{X\subseteq U}G(X)=n-1,
\end{equation}
as witnessed by an alternating subset of vertices. Hence
\begin{equation}
\frac{
M_G(\pi_{\mathrm{nat}})
}{
\max_XG(X)
}
=
\frac{1}{n-1}
\to 0.
\end{equation}

For the AUC,
\begin{equation}
\operatorname{AUC}_G(\pi_{\mathrm{nat}})
=
\frac{1}{n}\sum_{t=1}^nG(P_t^{\pi_{\mathrm{nat}}})
=
\frac{n-1}{n}.
\end{equation}
On the other hand, the alternating order
\begin{equation}
\pi_{\mathrm{alt}}=(1,3,5,\dots,2,4,6,\dots)
\end{equation}
keeps many path edges crossing the prefix boundary for a linear number of prefixes, yielding
\begin{equation}
\operatorname{AUC}_G(\pi_{\mathrm{alt}})=\Omega(n).
\end{equation}
Therefore
\begin{equation}
\frac{
\operatorname{AUC}_G(\pi_{\mathrm{nat}})
}{
\operatorname{AUC}_G^\star
}
\to 0.
\end{equation}
This proves that without an explicit alignment assumption between $F$ and $G$, no constant-factor quality guarantee for $G$ can hold in general.
\end{proof}

\section{Proofs of the Main Theorems}
\label{app:proofs}

We prove the main theorems in \ref{sec:theory} in this section.

\begin{lemma}[Residual accumulation]
\label{lem:residual}
Under Assumption~\ref{ass:partition}, for any $X\subseteq U$,
\begin{equation}
0 \le R(X) \le |X|\varepsilon_R.
\end{equation}
\end{lemma}

\begin{proof}
The lower bound follows from $R(\varnothing)=0$ and
$\Delta_R(e\mid X)\ge 0$ for every feasible pair $(X,e)$.
For the upper bound, write $X=\{e_1,\dots,e_m\}$ and telescope:
\begin{equation}
R(X)
=
\sum_{i=1}^m \Delta_R(e_i\mid \{e_1,\dots,e_{i-1}\})
\le m\varepsilon_R.
\end{equation}
\end{proof}
\begin{lemma}[Two-sided transfer]
\label{lem:twosided-transfer}
Assume Assumption~\ref{ass:align} and $\lambda_2>0$.
If $S,T\subseteq U$ and $c,\eta\ge 0$ satisfy
\begin{equation}
F(S)\ge c\,F(T)-\eta,
\end{equation}
then
\begin{equation}
\label{eq:generic-transfer-clean}
G(S)
\ge
\frac{c\lambda_1}{\lambda_2}\,G(T)
+
\frac{cb_1-\eta-b_2}{\lambda_2}.
\end{equation}
\end{lemma}

\begin{proof}
The upper side of Assumption~\ref{ass:align} gives
$F(S)\le \lambda_2G(S)+b_2$, while the lower side gives
$F(T)\ge \lambda_1G(T)+b_1$.  Therefore
\begin{equation}
\lambda_2 G(S)+b_2
\ge
F(S)
\ge
c\,F(T)-\eta
\ge
c\bigl(\lambda_1G(T)+b_1\bigr)-\eta.
\end{equation}
Moving $b_2$ to the right-hand side and dividing by $\lambda_2$
proves the claim.
\end{proof}
\begin{proof}[Proof of Theorem~\ref{thm:finiteF}]
Let the first $k$ accepted elements be $(v_1,\ldots,v_k)$, and write
$S_i^{\mathrm{PW}}=\{v_1,\ldots,v_i\}$ with
$S_0^{\mathrm{PW}}=\varnothing$.  Let
$S^\star=S_{F,k}^\star$ be an optimal set for the search objective under the
cardinality constraint $|X|\le k$.  We use the shorthand
\begin{equation}
F_i := F(S_i^{\mathrm{PW}}),
\quad
d_i := \Delta_F(v_i\mid S_{i-1}^{\mathrm{PW}})=F_i-F_{i-1}.
\end{equation}
Since $F(\varnothing)=\Phi(\varnothing)+R(\varnothing)=0$, we have
$F_i=\sum_{m=1}^i d_m$.
Set
\begin{equation}
c_F := F(S^\star)-k\varepsilon_R.
\end{equation}
The proof first lower-bounds $F(S_k^{\mathrm{PW}})$, and then transfers the
result to $G$ by Lemma~\ref{lem:twosided-transfer}.

By Lemma~\ref{lem:residual},
\begin{equation}
F(S^\star)
=
\Phi(B(S^\star))+R(S^\star)
\le
\Phi(B(S^\star))+k\varepsilon_R,
\end{equation}
and hence
\begin{equation}
\label{eq:cF-block-bound}
c_F \le \Phi(B(S^\star)).
\end{equation}
Thus, up to the additive residual loss $k\varepsilon_R$, it is enough to compare
PhaseWin against the block-level value of the optimal set.

Fix a step $\ell\in\{1,\dots,k\}$.  For an element $e$, let
$b(e)$ denote the unique block index satisfying $e\in H_{b(e)}$, and recall
$\Delta_\Phi(b(e)\mid J):=\Phi(J\cup\{b(e)\})-\Phi(J)$. Note that this quantity is zero
when $b(e)\in J$. By monotonicity and submodularity of $\Phi$,
\begin{equation}
\begin{aligned}
\Phi(B(S^\star))
&\le
\Phi\bigl(B(S^\star)\cup B(S_{\ell-1}^{\mathrm{PW}})\bigr) \\
&\le
\Phi\bigl(B(S_{\ell-1}^{\mathrm{PW}})\bigr)
+
\sum_{e\in S^\star\setminus S_{\ell-1}^{\mathrm{PW}}}
\Delta_\Phi\!\bigl(B(e)\mid B(S_{\ell-1}^{\mathrm{PW}})\bigr).
\end{aligned}
\end{equation}
The sum may contain several elements from the same block; this only makes the
upper bound looser, because all block marginal gains are nonnegative.
Since $R\ge 0$, we also have $\Phi(J_{\ell-1})\le F_{\ell-1}$.  Combining this with
\eqref{eq:cF-block-bound} gives
\begin{equation}
\label{eq:finite-opt-bound-clean}
c_F
\le
F_{\ell-1}
+
\sum_{e\in S^\star\setminus S_{\ell-1}^{\mathrm{PW}}}
\Delta_\Phi\!\bigl(B(e)\mid B(S_{\ell-1}^{\mathrm{PW}})\bigr).
\end{equation}

Now we fix $e\in S^\star\setminus S_{\ell-1}^{\mathrm{PW}}$.
By the PhaseWin algorithm, such an element is either still live in $\mathcal R_\ell$
or was deleted at some earlier decision point $i<\ell$.

If $e\in \mathcal R_\ell$, then
\begin{equation}
\begin{aligned}
\Delta_\Phi\!\bigl(B(e)\mid B(S_{\ell-1}^{\mathrm{PW}})\bigr)
&\le
\Delta_F(e\mid S_{\ell-1}^{\mathrm{PW}}) \\
&\le
a_\ell
\le
\frac{d_\ell}{\kappa_\ell}
\le
\frac{d_\ell}{\kappa_1},
\end{aligned}
\end{equation}
where the first inequality uses $\Delta_R(e\mid S_{\ell-1}^{\mathrm{PW}})\ge 0$,
the second is the definition of $a_\ell$, the third is from the fact that
PhaseWin selection guarantee $d_\ell\ge \kappa_\ell a_\ell$, and the last uses
the monotonicity $\kappa_\ell\ge \kappa_1$.

If instead $e\in \mathcal D_i$ for some $i<\ell$, then
$B(S_{i-1}^{\mathrm{PW}})\subseteq B(S_{\ell-1}^{\mathrm{PW}})$, so by
submodularity of $\Phi$ and the deletion criterion,
\begin{equation}
\begin{aligned}
\Delta_\Phi\!\bigl(B(e)\mid B(S_{\ell-1}^{\mathrm{PW}})\bigr)
&\le
\Delta_\Phi\!\bigl(B(e)\mid B(S_{i-1}^{\mathrm{PW}})\bigr) \\
&\le
\Delta_F(e\mid S_{i-1}^{\mathrm{PW}}) \\
&\le
\rho_{\mathrm{del}}\,
\Delta_F(v_i\mid S_{i-1}^{\mathrm{PW}}) \\
&=
\rho_{\mathrm{del}}\,d_i.
\end{aligned}
\end{equation}

Substituting these bounds into \eqref{eq:finite-opt-bound-clean} gives a single-step progress inequality.
There are at most $k$ elements of $S^\star$ in the live part, and for each
earlier decision $i<\ell$ there are also at most $k$ deleted elements from
$S^\star$ to consider.  Therefore,
\begin{equation}
\begin{aligned}
c_F
&\le
F_{\ell-1}
+
\frac{k}{\kappa_1}d_\ell
+
k\rho_{\mathrm{del}}\sum_{i=1}^{\ell-1} d_i \\
&=
F_{\ell-1}
+
\frac{k}{\kappa_1}(F_\ell-F_{\ell-1})
+
k\rho_{\mathrm{del}}F_{\ell-1} \\
&=
\frac{k}{\kappa_1}F_\ell
-
\left(
\frac{k}{\kappa_1}-1-k\rho_{\mathrm{del}}
\right)F_{\ell-1}.
\end{aligned}
\end{equation}
Rearranging gives
\begin{equation}
\label{eq:finite-recurrence-clean}
F_\ell
\ge
\frac{\kappa_1}{k}c_F
+
\left(
1-\frac{\kappa_1}{k}-\kappa_1\rho_{\mathrm{del}}
\right)F_{\ell-1}.
\end{equation}
This is the usual approximate recurrence for greedy algorithm, with an additional
$\kappa_1\rho_{\mathrm{del}}$ loss accounting for deleted
candidates.

Define
\begin{equation}
\mu_k := 1-\frac{\kappa_1(1+k\rho_{\mathrm{del}})}{k}.
\end{equation}
Under the theorem hypothesis, $\mu_k\in[0,1]$.
Iterating \eqref{eq:finite-recurrence-clean} from $\ell=1$ to $\ell=k$ and using
$F_0=0$, we get
\begin{equation}
\begin{aligned}
F_k
&\ge
\frac{\kappa_1}{k}c_F
\sum_{j=0}^{k-1}\mu_k^j \\
&=
\frac{\kappa_1}{k}
\cdot
\frac{1-\mu_k^k}{1-\mu_k}
\cdot c_F \\
&=
\frac{
1-\left(
1-\frac{\kappa_1(1+k\rho_{\mathrm{del}})}{k}
\right)^k
}{
1+k\rho_{\mathrm{del}}
}
\left(F(S_{F,k}^\star)-k\varepsilon_R\right).
\end{aligned}
\end{equation}
This is exactly \eqref{eq:Ck}.

It remains to transfer the guarantee from $F$ to $G$.
Let
\begin{equation}
c_k := C_k(\kappa_1,\rho_{\mathrm{del}}).
\end{equation}
By Theorem~\ref{thm:finiteF},
\begin{equation}
\begin{aligned}
F(S_k^{\mathrm{PW}})
&\ge
c_k\bigl(F(S_{F,k}^\star)-k\varepsilon_R\bigr) \\
&\ge
c_k F(S_{G,k}^\star)-c_k k\varepsilon_R.
\end{aligned}
\end{equation}
Applying Lemma~\ref{lem:twosided-transfer} with
$S=S_k^{\mathrm{PW}}$, $T=S_{G,k}^\star$, $c=c_k$, and
$\eta=c_k k\varepsilon_R$, we obtain
\begin{equation}
\label{eq:finiteG-clean}
G(S_k^{\mathrm{PW}})
\ge
\frac{c_k\lambda_1}{\lambda_2}\,\mathrm{OPT}_{G,k}
+
\frac{c_k b_1-c_k k\varepsilon_R-b_2}{\lambda_2}.
\end{equation}

If $\kappa_1=1-o(1)$ and $\rho_{\mathrm{del}}=o(1/k)$, then
\begin{equation}
\mu_k
=
1-\frac{1}{k}+o\!\left(\frac{1}{k}\right),
\qquad
1+k\rho_{\mathrm{del}}=1+o(1),
\end{equation}
and therefore
\begin{equation}
C_k(\kappa_1,\rho_{\mathrm{del}})
=
1-e^{-1}-o(1).
\end{equation}

\end{proof}

\begin{proof}[Proof of Theorem~\ref{thm:linear}]
At the beginning of a phase, PhaseWin scans all candidates that are still in the
live pool.  The candidate with the largest marginal gain is selected as the
anchor of this phase.  We show that, as long as there is still an inactive block
visible in the live pool, this anchor must come from an inactive block.

Let $S$ be the current selected set.  Consider first an element $e\in H_j$ whose
block has not been activated. That is, $j\notin B(S)$.  Adding $e$ activates
block $j$, so its marginal gain can be written as
\begin{equation}
\Delta_F(e\mid S)
=
\Delta_\Phi(j\mid B(S))+\Delta_R(e\mid S)
\ge
\underline{\Delta}_\Phi,
\end{equation}
where the inequality uses the definition of $\underline{\Delta}_\Phi$ and the
nonnegativity of the residual marginal.  Thus every live representative of an
inactive block has gain at least $\underline{\Delta}_\Phi$.

Now consider an element $e$ from a block that is already active.  Adding such an
element does not activate any new block, so the block-level term does not change.
Only the residual term can increase:
\begin{equation}
\Delta_F(e\mid S)=\Delta_R(e\mid S)\le \varepsilon_R.
\end{equation}
By Assumption~\ref{ass:partition}, we have
$\varepsilon_R<\kappa_1\underline{\Delta}_\Phi$, and since
$\kappa_1\le 1$, this implies $\varepsilon_R<\underline{\Delta}_\Phi$.  Hence
any live element from an inactive block has strictly larger gain than any live
element from an already active block.  Therefore, if the live pool still contains
a representative of some inactive block, the maximum-gain anchor selected by the
global scan must activate a new block.

This gives a direct bound on the number of phases.  Each effective
non-terminal phase can be associated with the new block activated by its anchor.
The same block cannot be activated twice, and there are only $q$ blocks in total.
Therefore, there are at most $q$ effective non-terminal phases that activate new
blocks.  After all blocks that remain visible have already been activated, by the exit criterion there
may be one last phase in which only residual gains are left.  This possible last
phase accounts for the additional $+1$.  Hence the number of effective phases is
at most $q+1$.

It remains to count the evaluations within each phase.  The global scan evaluates
the marginal gain of every live candidate once, so it uses at most $n$
evaluations.  The windowed refinement only processes candidates that enter the
phase window.  By the definition of the window policy, the number of exact
reevaluations needed for one window scan is bounded by $f_\psi(\omega)$.  Since
at most $n$ candidates can be charged to the windows of one phase, the local
refinement costs $O(n f_\psi(\omega))$ evaluations in one phase.  Thus one phase
costs
\[
O\bigl(n(f_\psi(\omega)+1)\bigr)
\]
evaluations.  Multiplying this by the phase bound $q+1$ gives
\begin{equation}
O\!\bigl((q+1)n(f_\psi(\omega)+1)\bigr).
\end{equation}
When the number of blocks $q$ is independent of $n$, this becomes
$O(n(f_\psi(\omega)+1))$.  If the window size $\omega$ is also fixed, then
$f_\psi(\omega)$ is a constant, and the total number of evaluations is linear in
$n$.
\end{proof}

\begin{proof}[Proof of Corollary~\ref{thm:max}]
Let $X^\star$ be an optimal set for $G$, and put
\begin{equation}
k^\star:=|X^\star|.
\end{equation}
Since PhaseWin is run until it produces a full ordering, the prefix
$P_{k^\star}^{\pi^{\mathrm{PW}}}=S_{k^\star}^{\mathrm{PW}}$ is available.  The
first $k^\star$ selected elements are exactly the output that PhaseWin would
return if we stopped the same run after $k^\star$ selections.  Therefore we may
apply Theorem~\ref{thm:finiteF} with target cardinality $k^\star$.

The set $X^\star$ is feasible for the problem
$\max_{|X|\le k^\star}G(X)$.  Hence
\begin{equation}
\mathrm{OPT}_{G,k^\star}
=
G(X^\star)
=
\mathrm{OPT}_G .
\end{equation}
Using the $G$-version of Theorem~\ref{thm:finiteF}, we obtain
\begin{equation}
\begin{aligned}
G(P_{k^\star}^{\pi^{\mathrm{PW}}})
\ge\;&
\frac{C_{k^\star}(\kappa_1,\rho_{\mathrm{del}})\lambda_1}{\lambda_2}
\,\mathrm{OPT}_G \\
&+
\frac{
C_{k^\star}(\kappa_1,\rho_{\mathrm{del}})b_1
-
C_{k^\star}(\kappa_1,\rho_{\mathrm{del}})k^\star\varepsilon_R
-
b_2
}{\lambda_2}.
\end{aligned}
\end{equation}
Finally, the prefix maximum is at least the value of this particular prefix:
\begin{equation}
M_G(\pi^{\mathrm{PW}})
\ge
G(P_{k^\star}^{\pi^{\mathrm{PW}}}).
\end{equation}
Combining the last two displays proves the desired prefix-maximum bound.  If the
statement uses the generic symbol $k$, it corresponds here to the prefix length
$k^\star=|X^\star|$.
\end{proof}
\begin{proof}[Proof of Corollary~\ref{thm:auc}]
In the equal-area setting,
\begin{equation}
\operatorname{AUC}_G(\pi)
=
\frac1n\sum_{t=1}^nG(P_t^\pi).
\end{equation}
Let $\pi^\star$ be an ordering that maximizes this quantity.  We compare the two
orderings prefix by prefix.

Fix a prefix length $t$ and write
\[
c_t:=C_t(\kappa_1,\rho_{\mathrm{del}}).
\]
The set $P_t^{\pi^\star}$ has size $t$, so it is a feasible competitor for the
cardinality-$t$ problem.  Applying Theorem~\ref{thm:finiteF} to the first $t$
elements of the PhaseWin ordering gives
\begin{equation}
\label{eq:auc-prefix-transfer}
\begin{aligned}
G(P_t^{\pi^{\mathrm{PW}}})
\ge\;&
\frac{c_t\lambda_1}{\lambda_2}
G(P_t^{\pi^\star}) \\
&+
\frac{
c_t b_1
-
c_t t\varepsilon_R
-
b_2
}{\lambda_2}.
\end{aligned}
\end{equation}
This inequality says that the $t$-th PhaseWin prefix is comparable with the
$t$-th prefix of the best AUC ordering.  Summing
\eqref{eq:auc-prefix-transfer} over all $t=1,\ldots,n$ and dividing by $n$ gives
\begin{equation}
\label{eq:auc-exact-average}
\begin{aligned}
\operatorname{AUC}_G(\pi^{\mathrm{PW}})
\ge\;&
\frac{\lambda_1}{\lambda_2}
\cdot
\frac1n
\sum_{t=1}^n
c_t\,G(P_t^{\pi^\star}) \\
&+
\frac{1}{n\lambda_2}
\sum_{t=1}^n
\Bigl(
c_t b_1
-
c_t t\varepsilon_R
-
b_2
\Bigr).
\end{aligned}
\end{equation}

For a shorter one-line bound, define
\begin{equation}
C_{\min}:=\min_{1\le t\le n} c_t
\end{equation}
and
\begin{equation}
\Gamma_{\min}:=
\min_{1\le t\le n}
\bigl(c_t b_1-c_t t\varepsilon_R\bigr).
\end{equation}
Since $G$ is nonnegative in the attribution setting,
\eqref{eq:auc-exact-average} implies
\begin{equation}
\operatorname{AUC}_G(\pi^{\mathrm{PW}})
\ge
\frac{C_{\min}\lambda_1}{\lambda_2}
\operatorname{AUC}_G^\star
+
\frac{\Gamma_{\min}-b_2}{\lambda_2}.
\end{equation}
And we complete the proof.
\end{proof}
\section{Additional Support Experiments}
\subsection{Linear Complexity Verification with Fixed Window Size}

We first demonstrated in the main text that window size is the most significant factor affecting the number of forward passes required by PhaseWin. Experimental results across different settings also demonstrate PhaseWin's general speedup performance. Next, we showcase the linear complexity of PhaseWin under a specific set of experimental settings with fixed window size and exit threshold.

We consider the attribution task of Grounding DINO on MS COCO for correctly detected samples, using SLICO superpixel segmentation, considering the number of regions 50, 64, and 100, fixing the Phasewin window size to 16, and setting the exit threshold to 0.025, and calculating the average number of forward passes.

\begin{figure}[ht]
    \centering
    \includegraphics[width=0.5\linewidth]{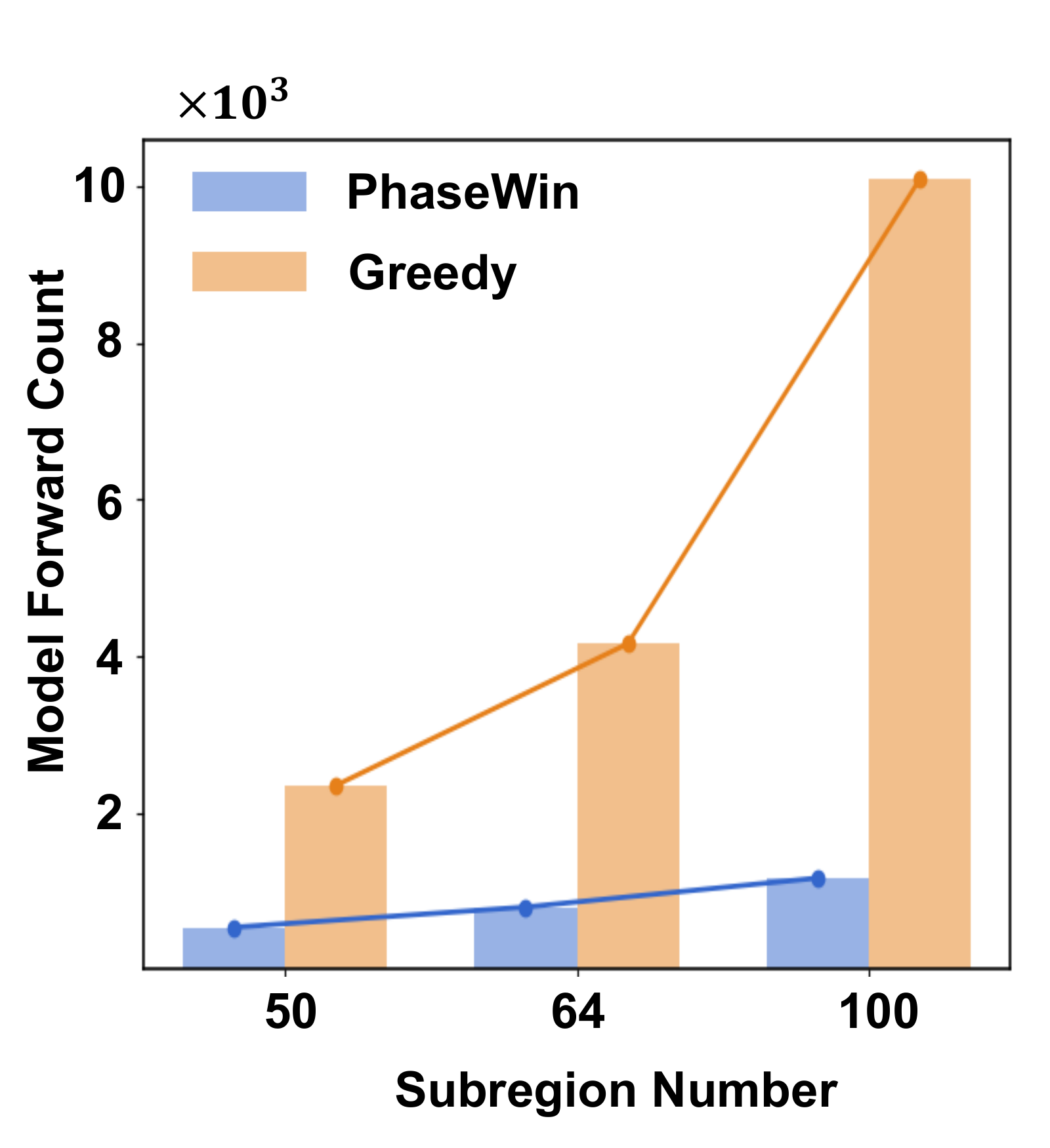}
    \caption{
Linear-complexity verification under a fixed window configuration. 
We evaluate Grounding DINO on correctly detected MS COCO samples with SLICO superpixels and vary the number of segmented regions from 50 to 64 and 100. 
PhaseWin uses the same window size and exit threshold across all settings. 
As the number of regions increases, the model-forward count of Greedy grows quadratically, whereas PhaseWin increases approximately linearly, confirming the expected near-linear behavior under fixed hyperparameters.
}
    \label{fig:intro1}
\end{figure}

The results are shown in Figure~\ref{fig:intro1}. The value of phasewin is 536.8-706.5-1192.7, and the value of greedy is 2548.8-4132.7-10100, which shows that under fixed hyperparameters, phasewin changes linearly with the number of regions.

\subsection{Greedy's Advantage is Amplified by a Right-Tail of Samples}
\label{app:greedy_tail}
\begin{table}[t]
\centering
\small
\setlength{\tabcolsep}{4.5pt}
\renewcommand{\arraystretch}{1.12}
\caption{Right-tail analysis of Greedy's insertion-AUC advantage over PhaseWin on classification.}
\label{tab:greedy_tail_model}
\begin{tabular}{lrrrr}
\toprule
Model
& \(n\)
& Mean gap
& \(d_i>10\)
& Tail contrib. \\
\midrule
CLIP RN101
& 9000
& 5.51
& 20.1\%
& 62.9\% \\
CLIP ViT-L/14
& 9000
& 3.50
& 11.4\%
& 52.8\% \\
ResNet-101
& 9000
& 3.30
& 7.7\%
& 39.4\% \\
\midrule
All
& 27000
& 4.10
& 13.1\%
& 53.7\% \\
\bottomrule
\end{tabular}
\end{table}

We further examine whether the insertion-AUC gap between Greedy and PhaseWin reflects a uniform sample-wise advantage or is mainly amplified by a small set of large-gap samples. This analysis is conducted on the classification task under three local evaluation settings, namely \texttt{true}, \texttt{cause}, and \texttt{repair}.

For each paired sample, we define
\begin{equation}
    d_i
    =
    100\cdot
    \left(
    \mathrm{InsAUC}^{\mathrm{Greedy}}_i
    -
    \mathrm{InsAUC}^{\mathrm{PhaseWin}}_i
    \right).
\end{equation}
Thus, \(d_i>0\) means that Greedy obtains a higher insertion AUC than PhaseWin on sample \(i\). We define the extreme Greedy-advantage set as
\begin{equation}
    \mathcal{E}_{10}=\{i:d_i>10\},
\end{equation}
where the threshold corresponds to a ten-point insertion-AUC advantage. We also compute the tail contribution ratio
\begin{equation}
    R_{10}
    =
    \frac{
        \sum_{i\in \mathcal{E}_{10}} d_i
    }{
        \sum_{i:d_i>0} d_i
    },
\end{equation}
which measures how much of Greedy's positive advantage is contributed by these extreme samples.

Table~\ref{tab:greedy_tail_model} shows a clear right-tail effect. Across all \(27{,}000\) paired samples, the average gap between Greedy and PhaseWin is \(4.10\) insertion-AUC points. However, only \(13.1\%\) of samples fall into the extreme region \(d_i>10\), while these samples account for \(53.7\%\) of Greedy's total positive advantage. This indicates that the aggregate gap is not a uniform sample-wise separation, but is substantially enlarged by a minority of large-gap cases.

This tail effect is especially pronounced for CLIP RN101, where \(20.1\%\) of samples contribute \(62.9\%\) of Greedy's positive advantage. CLIP ViT-L/14 shows a similar pattern, with \(11.4\%\) extreme samples explaining \(52.8\%\) of the positive advantage. ResNet-101 has a weaker but still visible tail concentration, suggesting that the degree of right-tail amplification depends on the backbone.

After excluding the extreme Greedy-advantage samples, the remaining gap between PhaseWin and Greedy becomes small, staying within \(2\%\) insertion AUC under the non-extreme classification protocol. Therefore, the main empirical difference between the two methods is not that Greedy consistently dominates PhaseWin on ordinary samples. Rather, Greedy occasionally obtains much larger gains on a small subset of samples, and these tail cases disproportionately increase the reported mean gap.

This observation is important for interpreting the efficiency--faithfulness trade-off. PhaseWin is designed to approximate Greedy's search behavior with substantially fewer model evaluations. The tail analysis shows that PhaseWin remains close to Greedy on the non-extreme majority of classification samples, while the remaining average gap is mainly driven by a small right tail. Hence, the aggregate Greedy advantage should be understood as a tail-amplified effect rather than a broad degradation of PhaseWin's attribution fidelity.

\textbf{Case Study and Fix}
Finally, we present three examples to illustrate the causes and solutions for such extreme cases.

\begin{figure}[ht]
    \centering
    \includegraphics[width=0.9\linewidth]{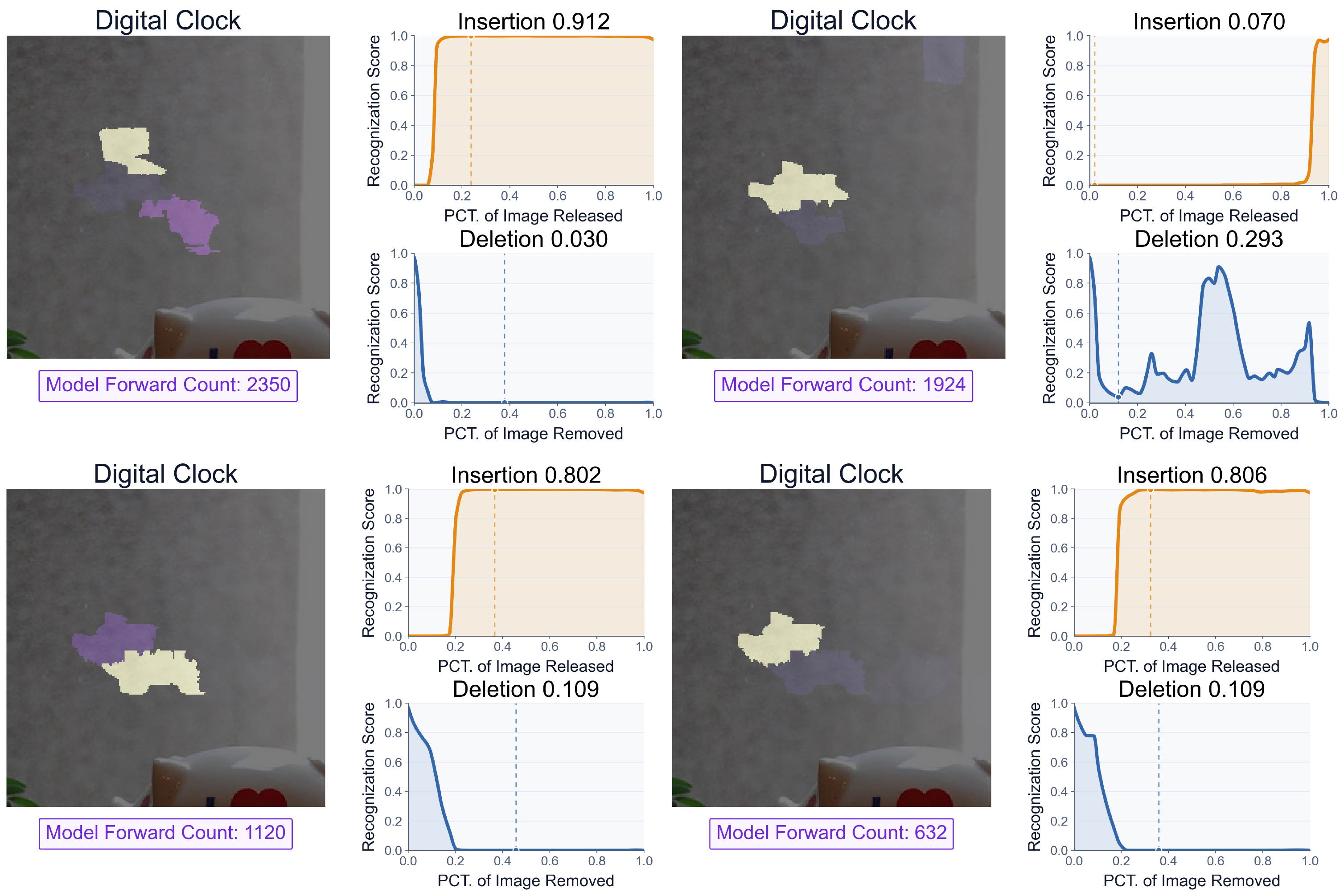}
    \caption{
Case study of an extreme Greedy-advantage sample on CLIP ViT-L/14 under the failure-attribution setting toward the model's wrong prediction. 
The original fine superpixel partition fragments the model-preferred noisy evidence, allowing exhaustive Greedy search to recover a high insertion trajectory while PhaseWin may postpone the decisive fragment. 
When the segmentation is coarsened, the unstable fragments are merged into more reliable attribution units, and PhaseWin recovers Greedy-level insertion and deletion behavior with fewer model evaluations.
}
    \label{fig:fail1}
\end{figure}

\begin{figure}[ht]
    \centering
    \includegraphics[width=0.9\linewidth]{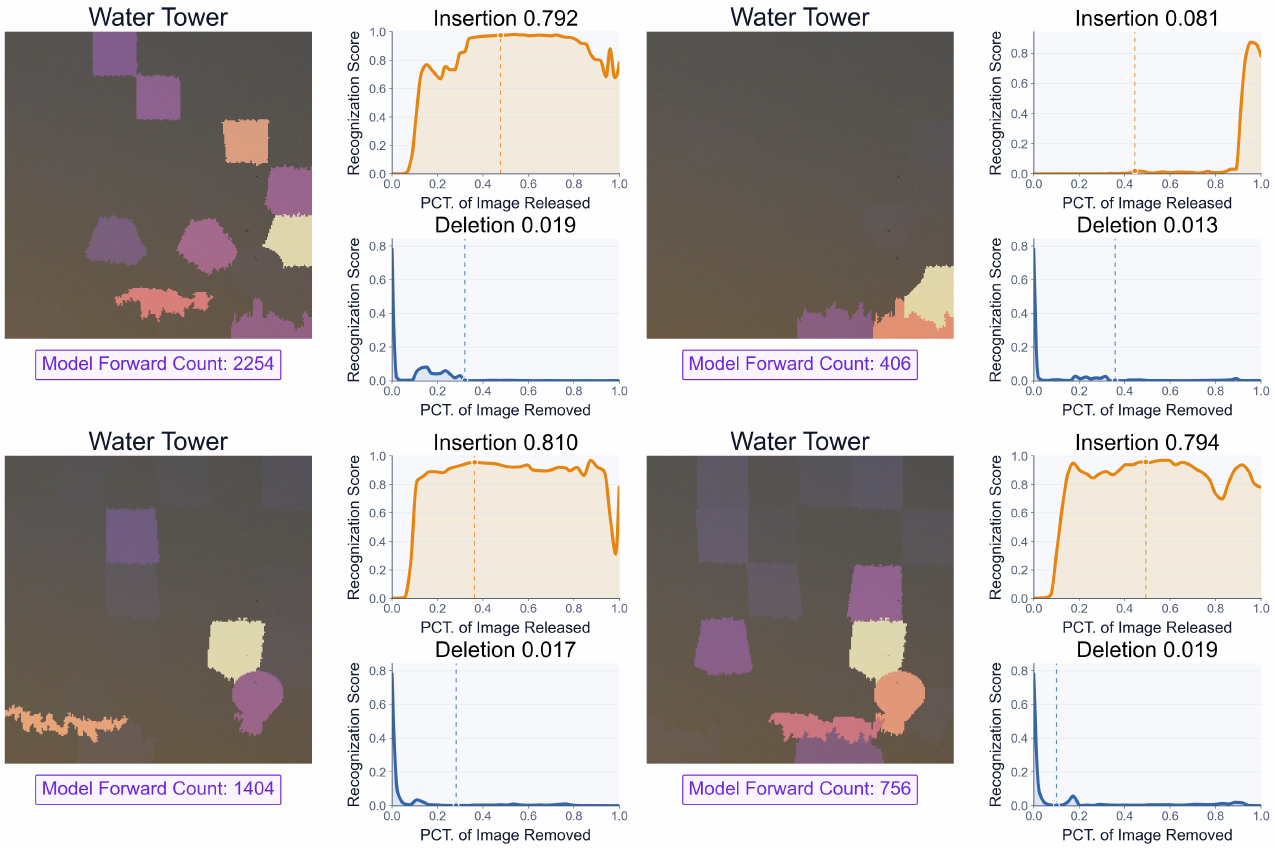}
    \caption{
Case study of an extreme Greedy-advantage sample on CLIP RN101 under the correctly classified setting. 
With the original fine partition, the model response is highly sensitive to small background or nuisance regions, producing unstable local marginal rankings. 
Greedy benefits from repeated global rescoring, whereas PhaseWin is affected by the partition-induced instability. 
After reducing the number of segmentation regions, the attribution units better match the model's response scale, and the gap between Greedy and PhaseWin is substantially reduced.
}
    \label{fig:fail2}
\end{figure}

\begin{figure}[ht]
    \centering
    \includegraphics[width=0.9\linewidth]{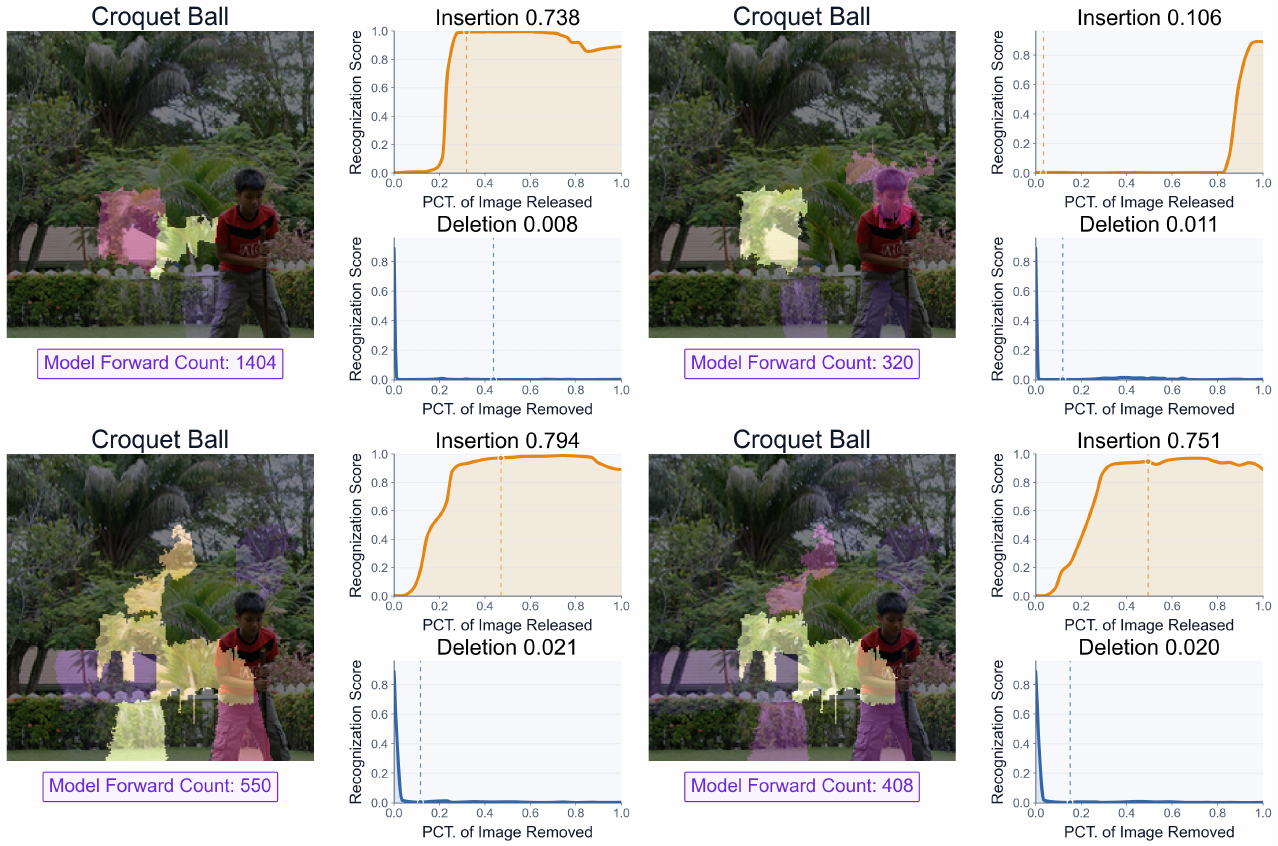}
    \caption{
Case study of an extreme Greedy-advantage sample on CLIP RN101 under the correctly classified setting. 
The large insertion-AUC gap is caused by an unfavorable over-partitioning of the visual evidence rather than by a systematic limitation of PhaseWin. 
A coarser segmentation merges correlated fragments into more stable regions, enabling PhaseWin to recover a response curve close to Greedy while using substantially fewer forward passes.
}
    \label{fig:fail3}
\end{figure}

Figures~\ref{fig:fail1}--\ref{fig:fail3} show the three samples with the largest insertion-AUC gap \(d_i\). These cases clarify that the apparent lead of Greedy does not mainly come from a stronger search capability. Rather, it is caused by a mismatch between the predefined superpixel partition and the model's intrinsic preference on noisy or background-dominated samples. Under the original fine partition, the evidence favored by the model is fragmented into unstable nuisance regions. Greedy can still recover these fragments because it repeatedly performs exhaustive global rescoring, whereas PhaseWin relies on a locally stable candidate ordering and can therefore postpone or miss the decisive noisy fragment. This is a partition--model mismatch rather than a systematic failure of the accelerated search.

The lower rows in Figures~\ref{fig:fail1}--\ref{fig:fail3} show that this mismatch can be directly mitigated by reducing the number of segmentation regions. A coarser partition merges correlated noisy fragments into more stable attribution units, making the search space better aligned with the model's response scale. After this adjustment, PhaseWin recovers Greedy-level insertion and deletion behavior on all three pathological samples, while still requiring substantially fewer model forward passes. Therefore, the large right-tail gap should be interpreted as an artifact induced by unfavorable over-partitioning on rare noisy samples, not as evidence that Greedy has a broad algorithmic advantage over PhaseWin.
\clearpage
\section{Additional Visualization Results}
\label{appvis}

This section provides additional qualitative examples for the three attribution settings studied in the main paper: image classification, object detection and visual grounding, and image captioning. These examples are intended to complement the quantitative comparisons by showing how the attribution maps behave across different backbones, tasks, and failure modes. Unless otherwise specified, the panels in each figure follow the left-to-right order stated in the corresponding caption.

\subsection{Additional Visualization Results for Classification}

We first provide additional classification visualizations on a shared input image across three representative backbones: CLIP-RN101, ResNet-101, and CLIP ViT-L/14. This controlled presentation separates the effect of attribution method from sample variation, and shows how gradient-based, perturbation-based, and subset-search methods behave under the same visual evidence.

\begin{figure}[t]
    \centering
    \includegraphics[width=0.95\linewidth]{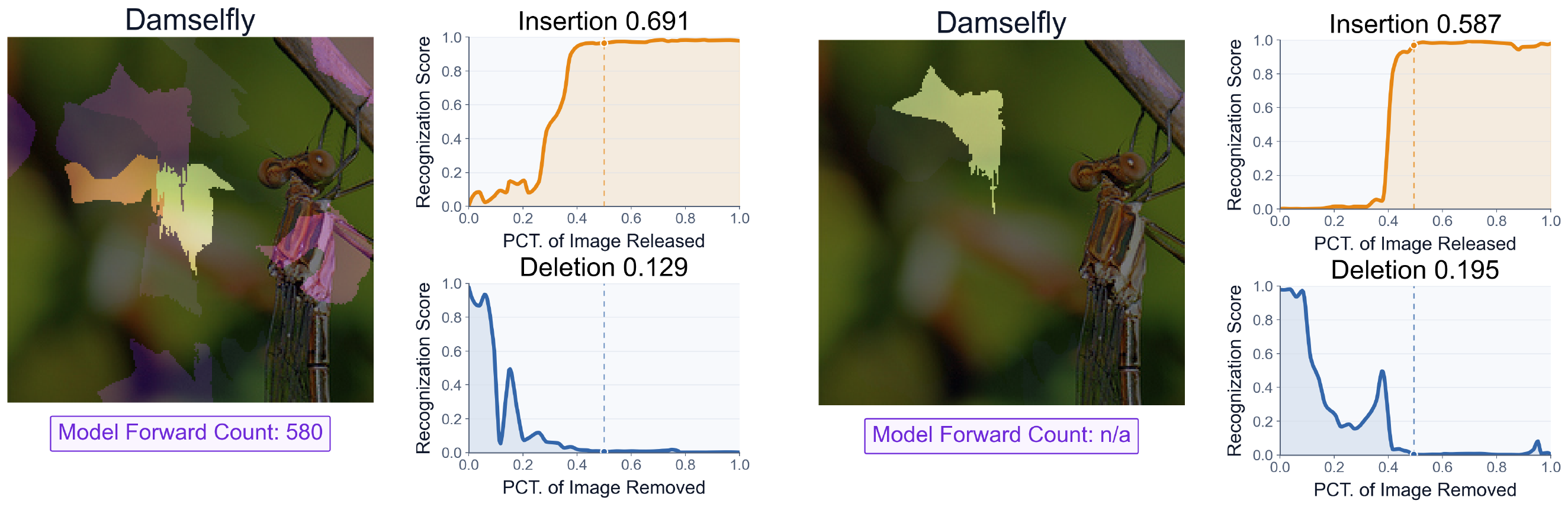}
    \caption{Additional classification attribution results on CLIP-RN101. From left to right: D-HSIC and D-RISE.}
    \label{fig:app-cls-clip-rn101-dhsic-drise}
\end{figure}

\begin{figure}[t]
    \centering
    \includegraphics[width=0.95\linewidth]{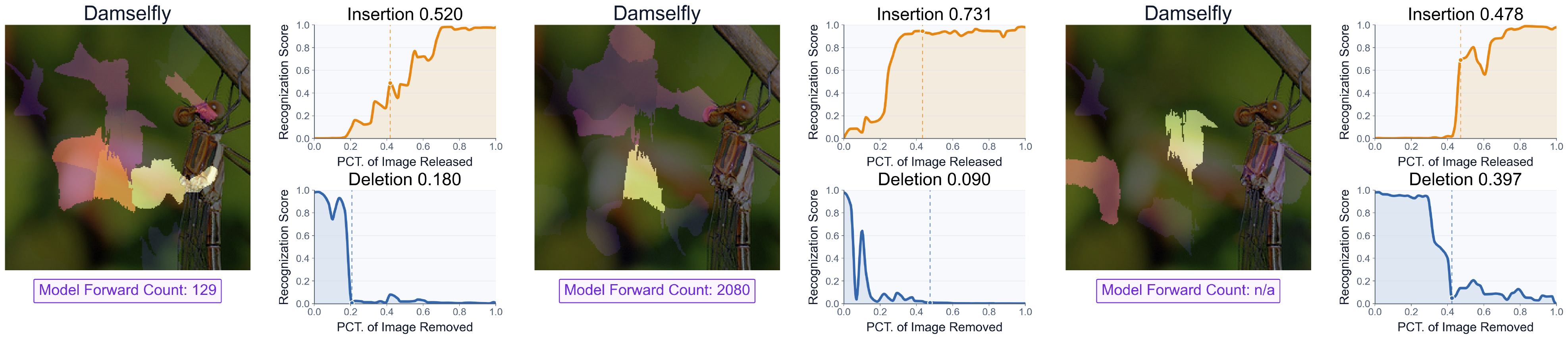}
    \caption{Additional classification attribution results on CLIP-RN101. From left to right: Gradient, Integrated Gradients, and IGOS++.}
    \label{fig:app-cls-clip-rn101-grad-ig-igos}
\end{figure}

\begin{figure}[t]
    \centering
    \includegraphics[width=0.95\linewidth]{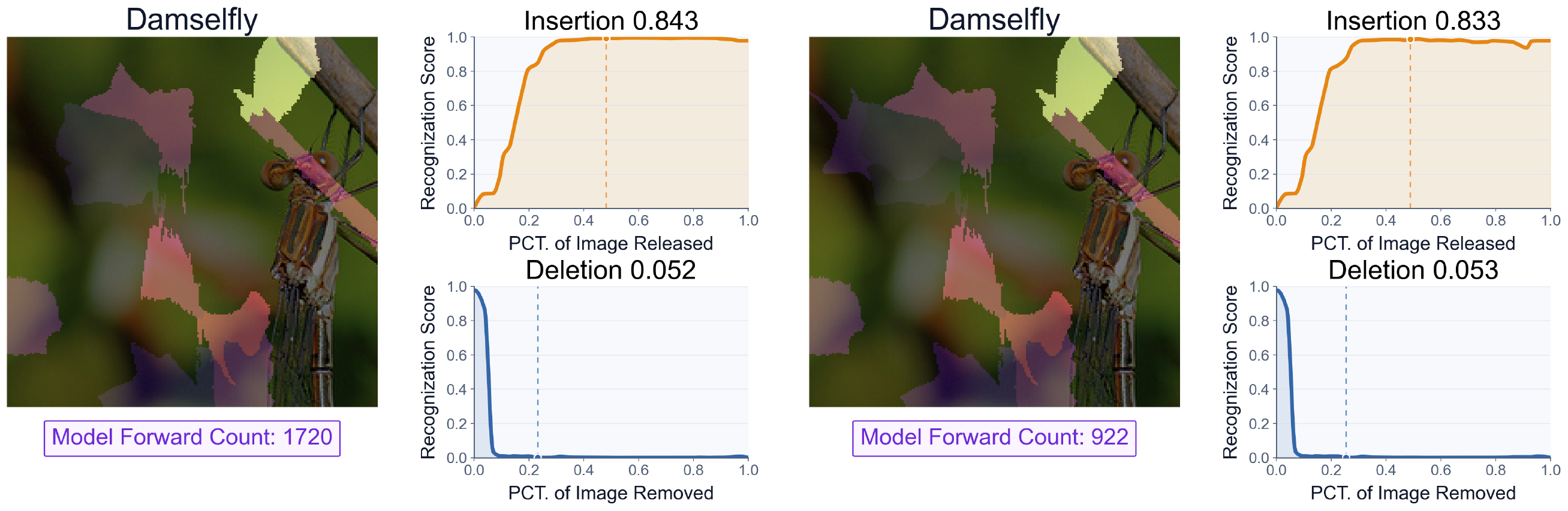}
    \caption{Additional classification attribution results on CLIP-RN101. From left to right: Greedy and PhaseWin.}
    \label{fig:app-cls-clip-rn101-greedy-phasewin}
\end{figure}

\begin{figure}[t]
    \centering
    \includegraphics[width=0.95\linewidth]{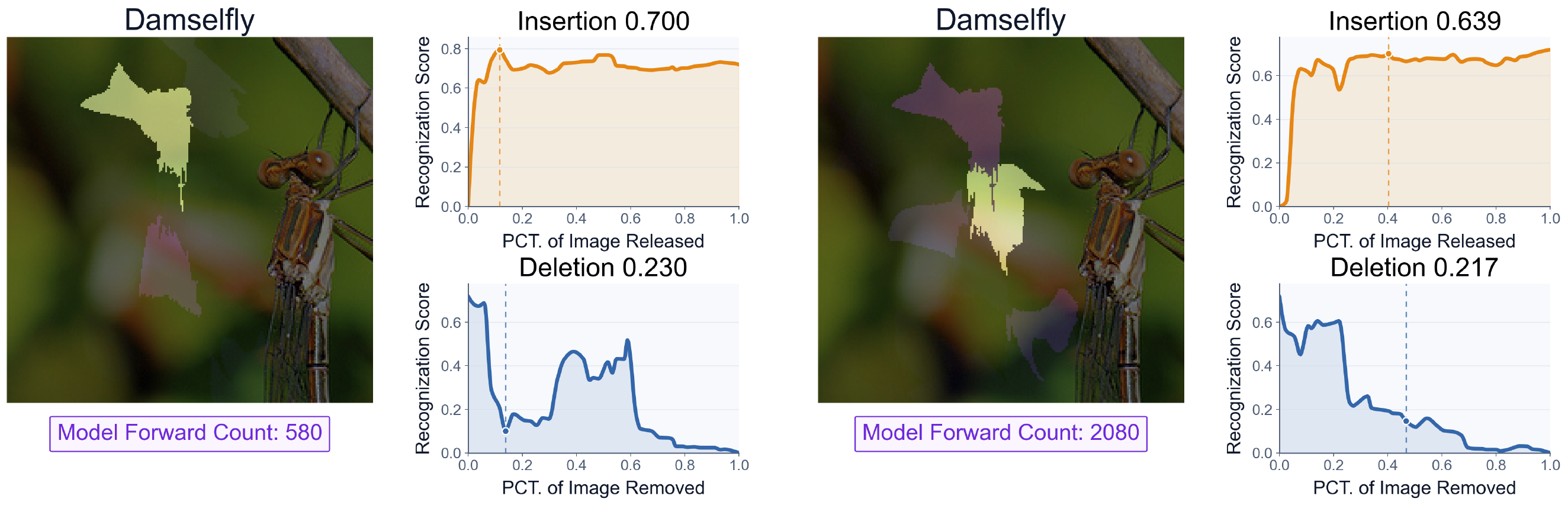}
    \caption{Additional classification attribution results on ResNet-101. From left to right: D-HSIC and D-RISE.}
    \label{fig:app-cls-resnet101-dhsic-drise}
\end{figure}

\begin{figure}[t]
    \centering
    \includegraphics[width=0.95\linewidth]{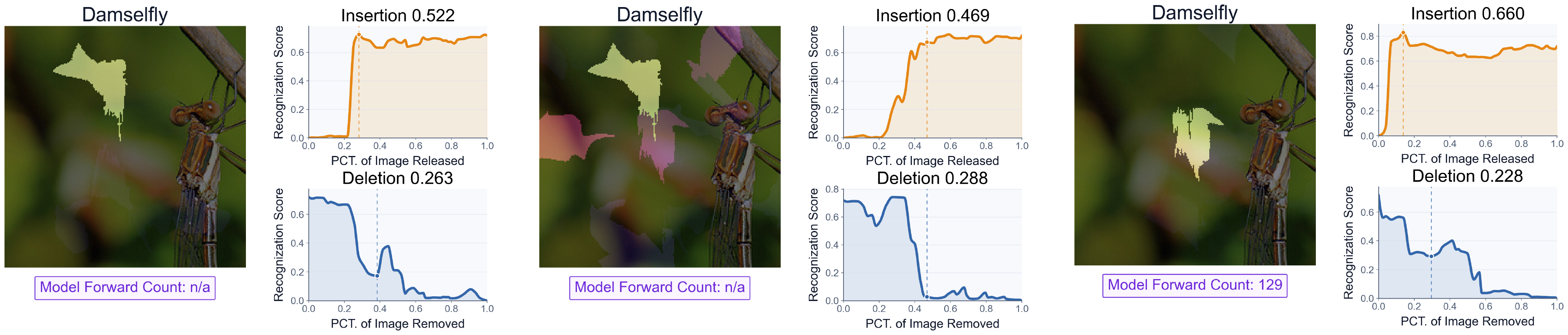}
    \caption{Additional classification attribution results on ResNet-101. From left to right: Gradient, Integrated Gradients, and IGOS++.}
    \label{fig:app-cls-resnet101-grad-ig-igos}
\end{figure}

\begin{figure}[t]
    \centering
    \includegraphics[width=0.95\linewidth]{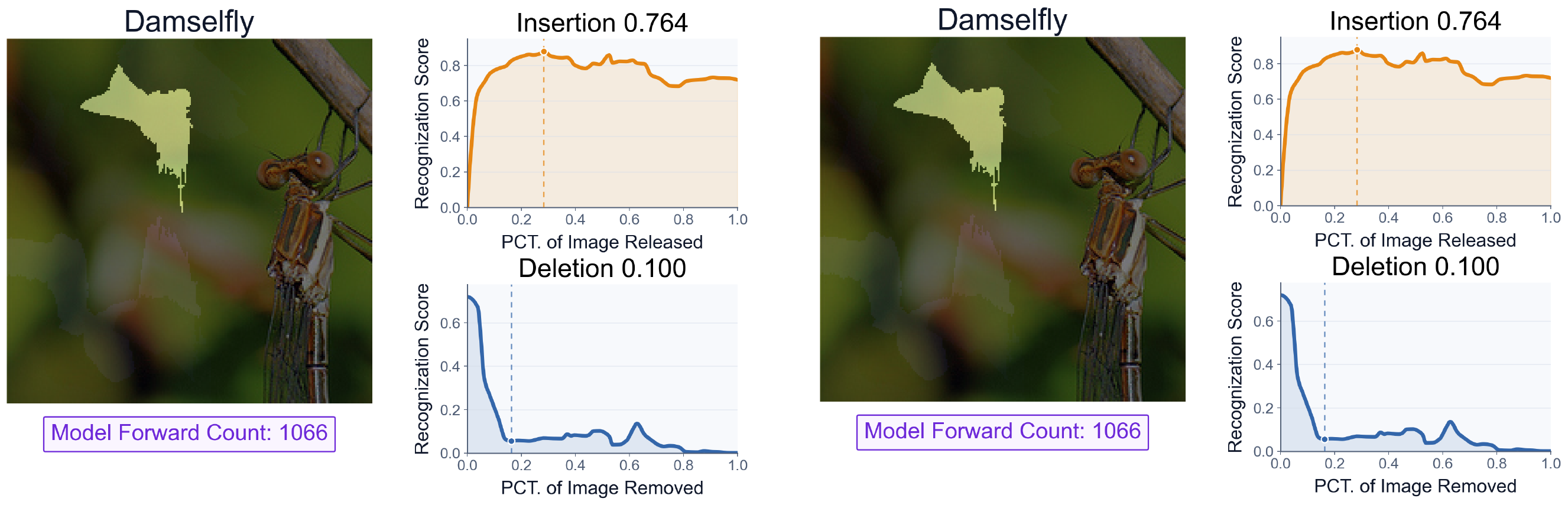}
    \caption{Additional classification attribution results on ResNet-101. From left to right: Greedy and PhaseWin.}
    \label{fig:app-cls-resnet101-greedy-phasewin}
\end{figure}

\begin{figure}[t]
    \centering
    \includegraphics[width=0.95\linewidth]{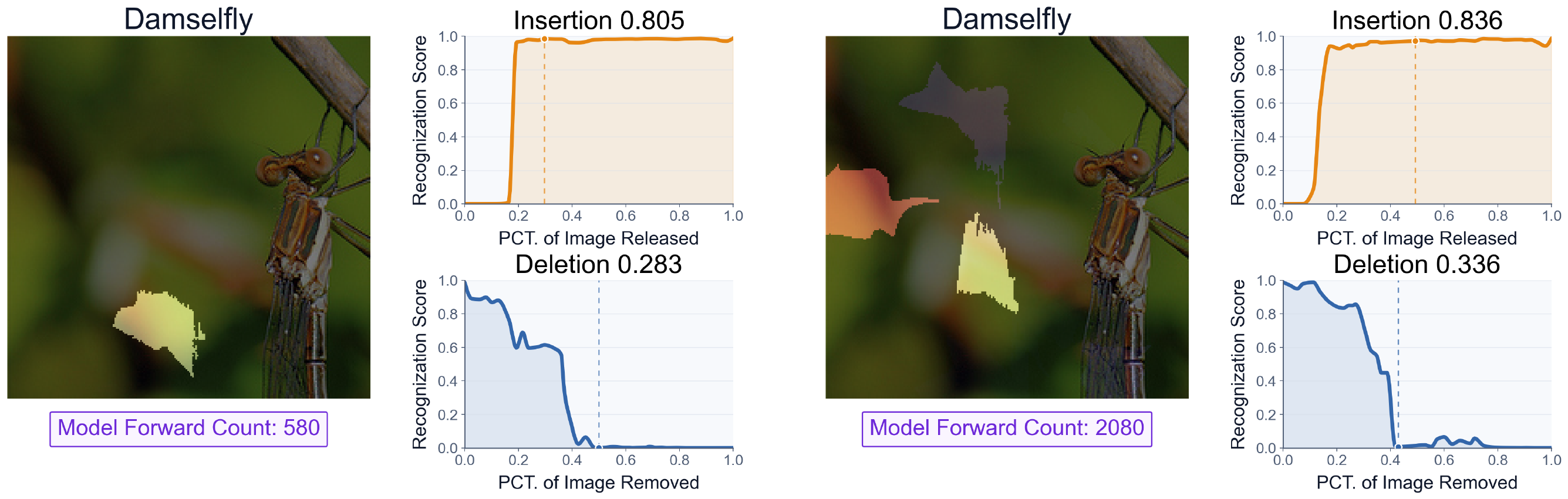}
    \caption{Additional classification attribution results on CLIP ViT-L/14. From left to right: D-HSIC and D-RISE.}
    \label{fig:app-cls-clip-vitl-dhsic-drise}
\end{figure}

\begin{figure}[t]
    \centering
    \includegraphics[width=0.95\linewidth]{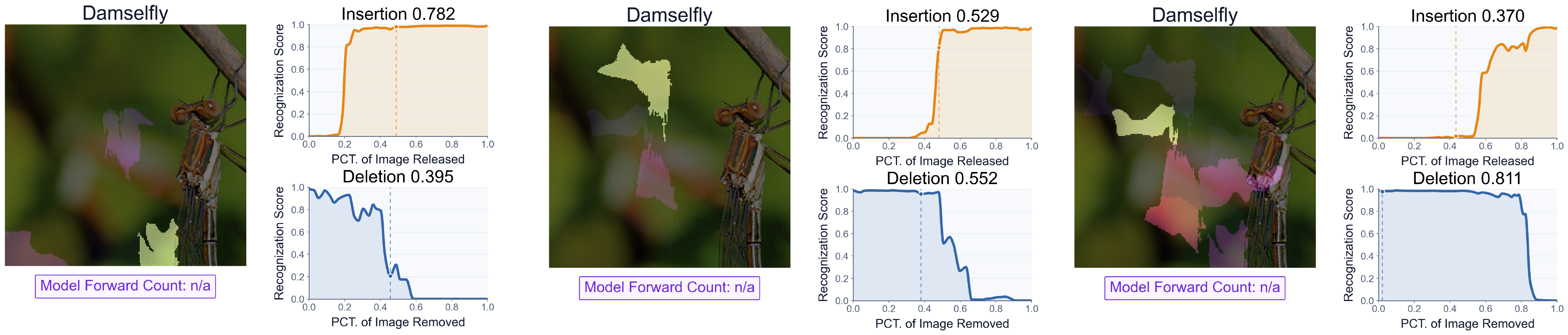}
    \caption{Additional classification attribution results on CLIP ViT-L/14. From left to right: Grad-ECLIP, Gradient, and Integrated Gradients.}
    \label{fig:app-cls-clip-vitl-gradeclip-grad-ig}
\end{figure}

\begin{figure}[t]
    \centering
    \includegraphics[width=0.95\linewidth]{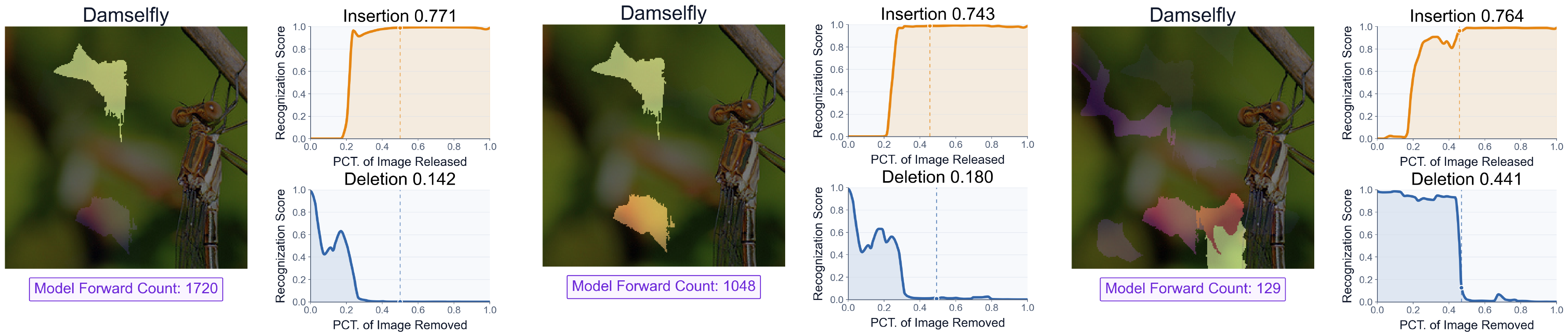}
    \caption{Additional classification attribution results on CLIP ViT-L/14. From left to right: Greedy, PhaseWin, and IGOS++.}
    \label{fig:app-cls-clip-vitl-greedy-phasewin-igos}
\end{figure}

\subsection{Additional Visualization Results for Detection and Visual Grounding}

We next provide additional examples for object detection and visual grounding. The visualization set covers both correct predictions and failure cases, including misclassification, missed detection, and incorrect grounding. These cases are used to examine whether PhaseWin preserves the diagnostic behavior of Greedy when the model prediction is either correct or erroneous.

\begin{figure}[t]
    \centering
    \includegraphics[width=0.95\linewidth]{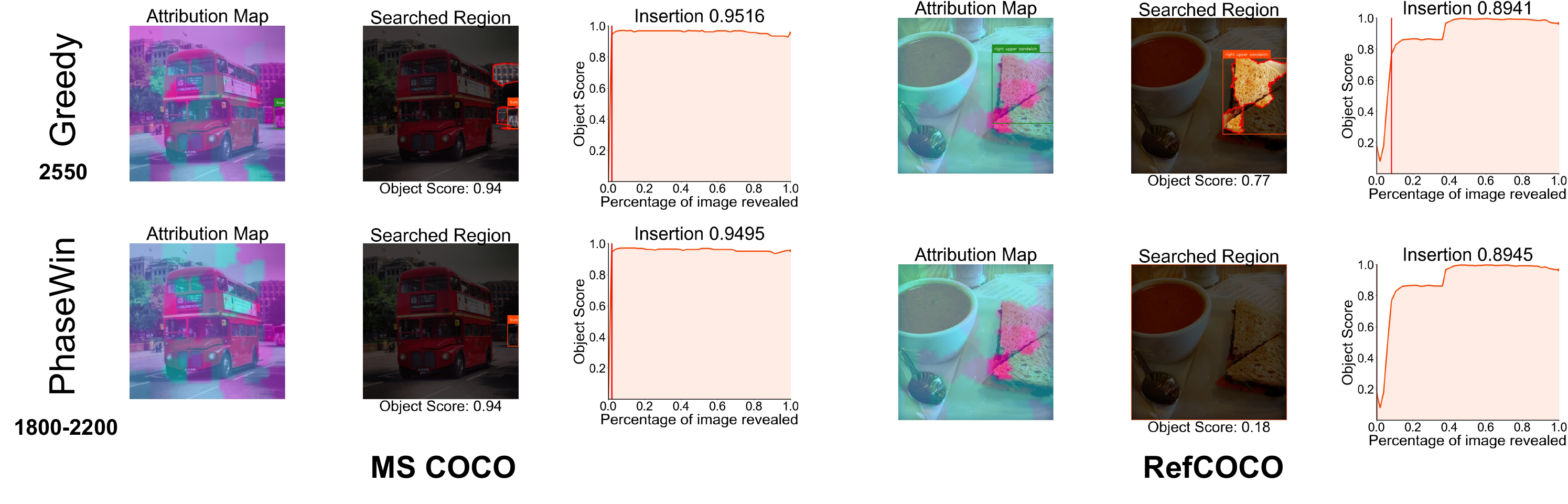}
    \caption{Additional correct-case visualizations on Florence-2 for MS COCO detection and RefCOCO grounding. From left to right, the panels compare Greedy and PhaseWin.}
    \label{fig:app-det-florence-correct}
\end{figure}

\begin{figure}[t]
    \centering
    \includegraphics[width=0.95\linewidth]{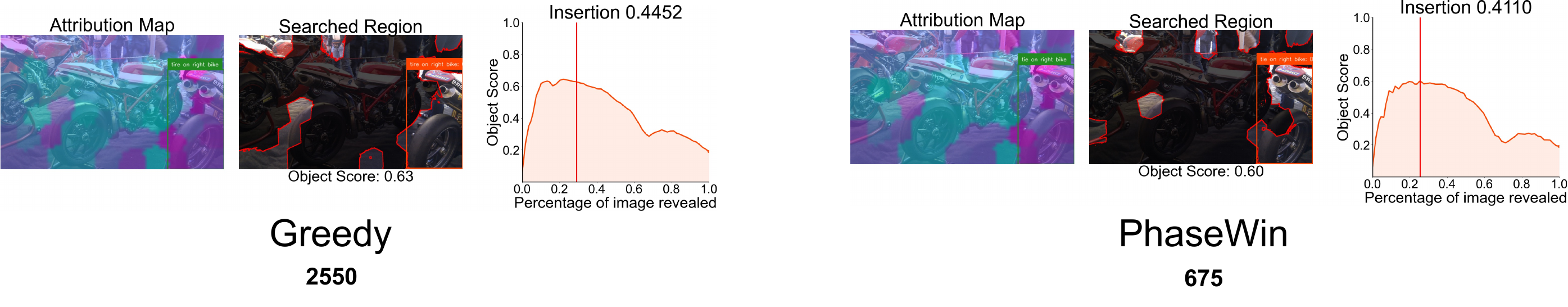}
    \caption{Additional failure-case visualizations on Grounding DINO for MS COCO misclassification. From left to right: Greedy and PhaseWin.}
    \label{fig:app-det-gdino-coco-miscls}
\end{figure}

\begin{figure}[t]
    \centering
    \includegraphics[width=0.95\linewidth]{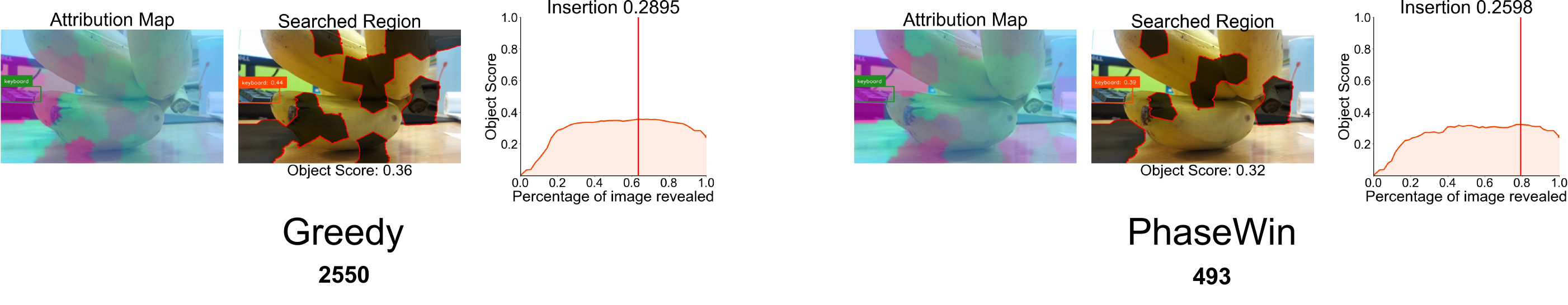}
    \caption{Additional failure-case visualizations on Grounding DINO for LVIS misclassification. From left to right: Greedy and PhaseWin.}
    \label{fig:app-det-gdino-lvis-miscls}
\end{figure}

\begin{figure}[t]
    \centering
    \includegraphics[width=0.95\linewidth]{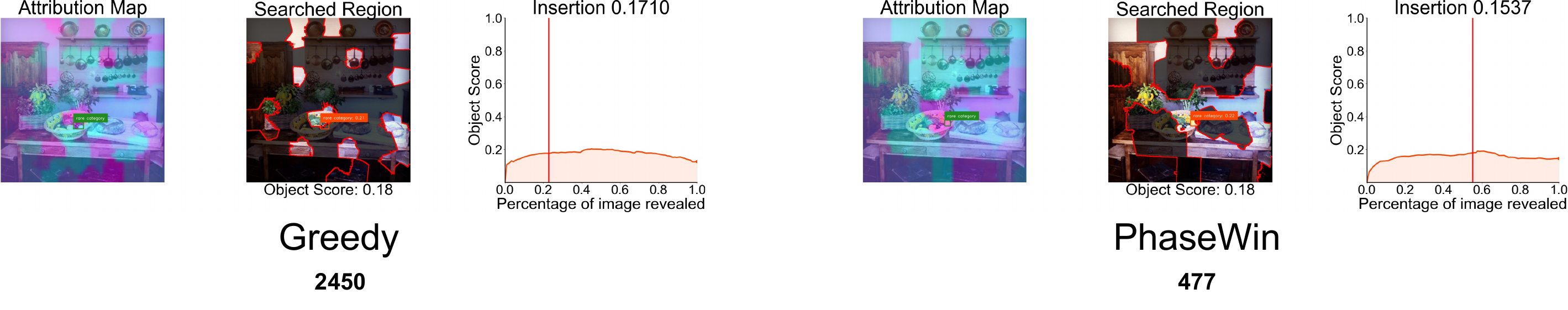}
    \caption{Additional failure-case visualizations on Grounding DINO for MS COCO missed detection. From left to right: Greedy and PhaseWin.}
    \label{fig:app-det-gdino-coco-missed}
\end{figure}

\begin{figure}[t]
    \centering
    \includegraphics[width=0.95\linewidth]{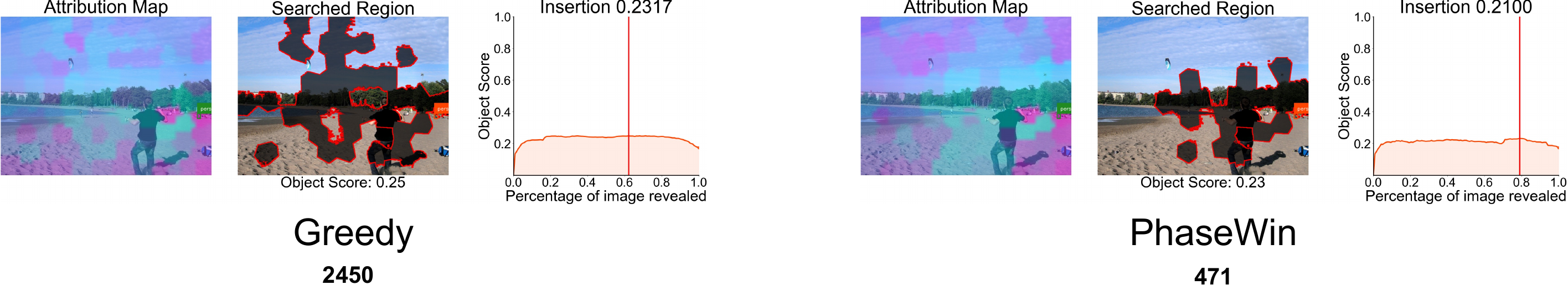}
    \caption{Additional failure-case visualizations on Grounding DINO for LVIS missed detection. From left to right: Greedy and PhaseWin.}
    \label{fig:app-det-gdino-lvis-missed}
\end{figure}

\begin{figure}[t]
    \centering
    \includegraphics[width=0.95\linewidth]{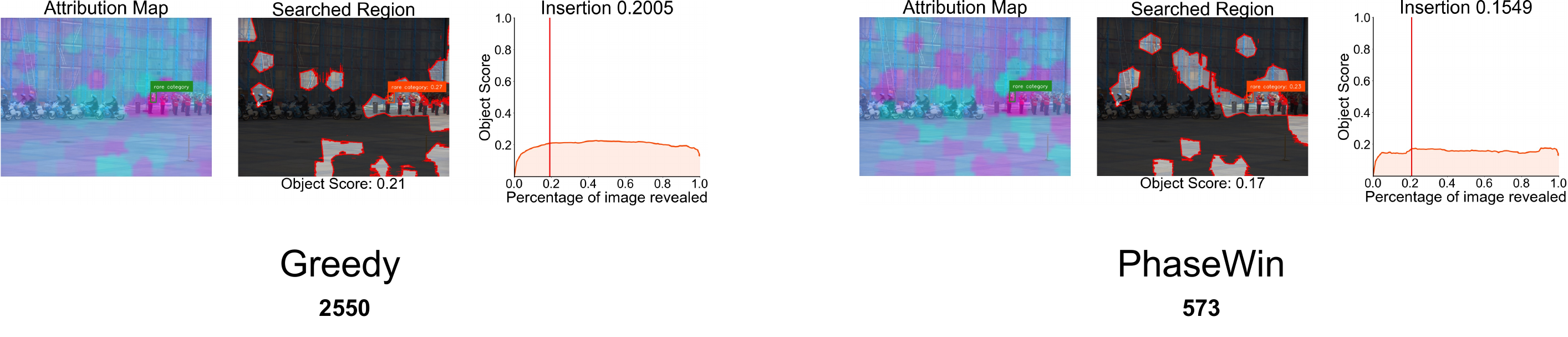}
    \caption{Additional failure-case visualizations on Grounding DINO for RefCOCO grounding errors. From left to right: Greedy and PhaseWin.}
    \label{fig:app-det-gdino-refcoco-failure}
\end{figure}

\subsection{Additional Visualization Results for Image Captioning}

Finally, we provide additional image-captioning visualizations on Qwen2.5-VL-3B. The main paper reports the corresponding qualitative examples on the larger 7B model; the examples here show that the same comparison protocol can also be applied to the 3B-scale captioning model.

\begin{figure}[t]
    \centering
    \includegraphics[width=0.95\linewidth]{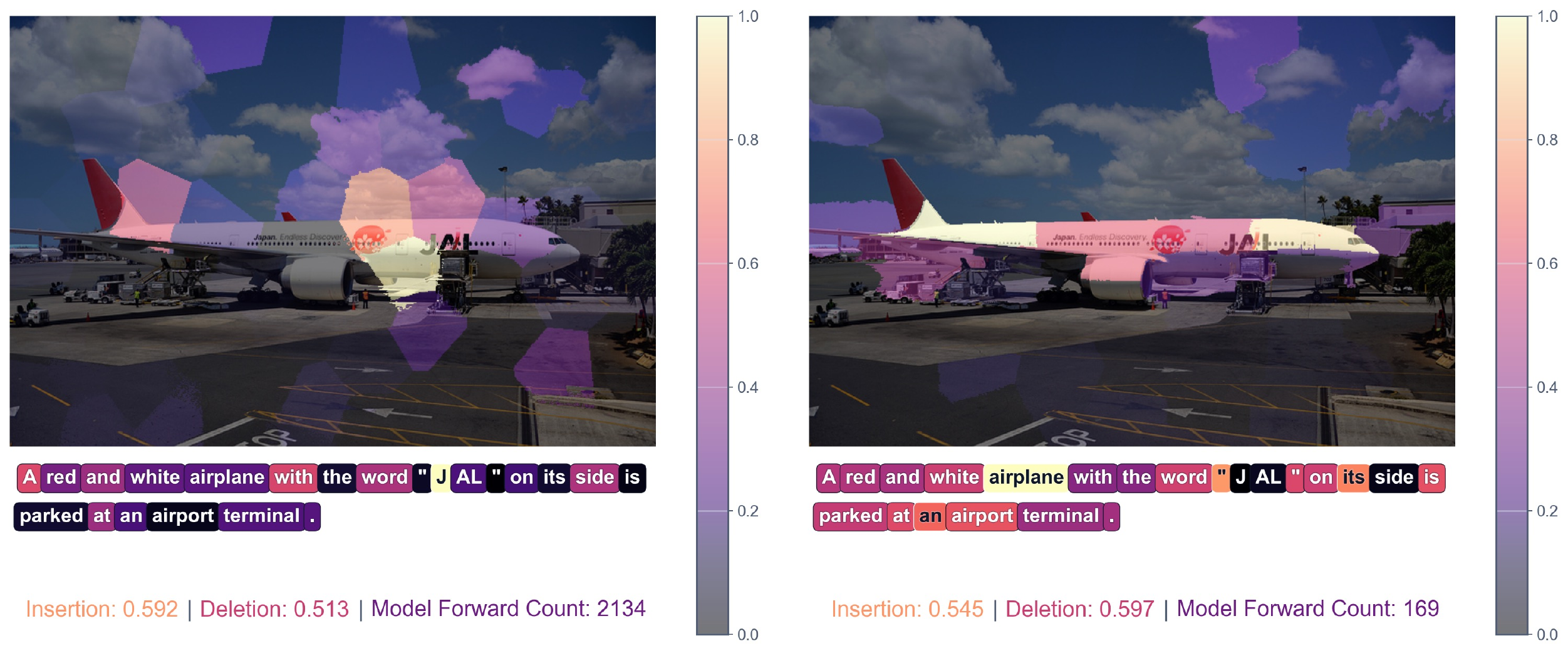}
    \caption{Additional image-captioning attribution results on Qwen2.5-VL-3B. From left to right: D-RISE and IGOS++.}
    \label{fig:app-cap-qwen25vl3b-drise-igos}
\end{figure}

\begin{figure}[t]
    \centering
    \includegraphics[width=0.95\linewidth]{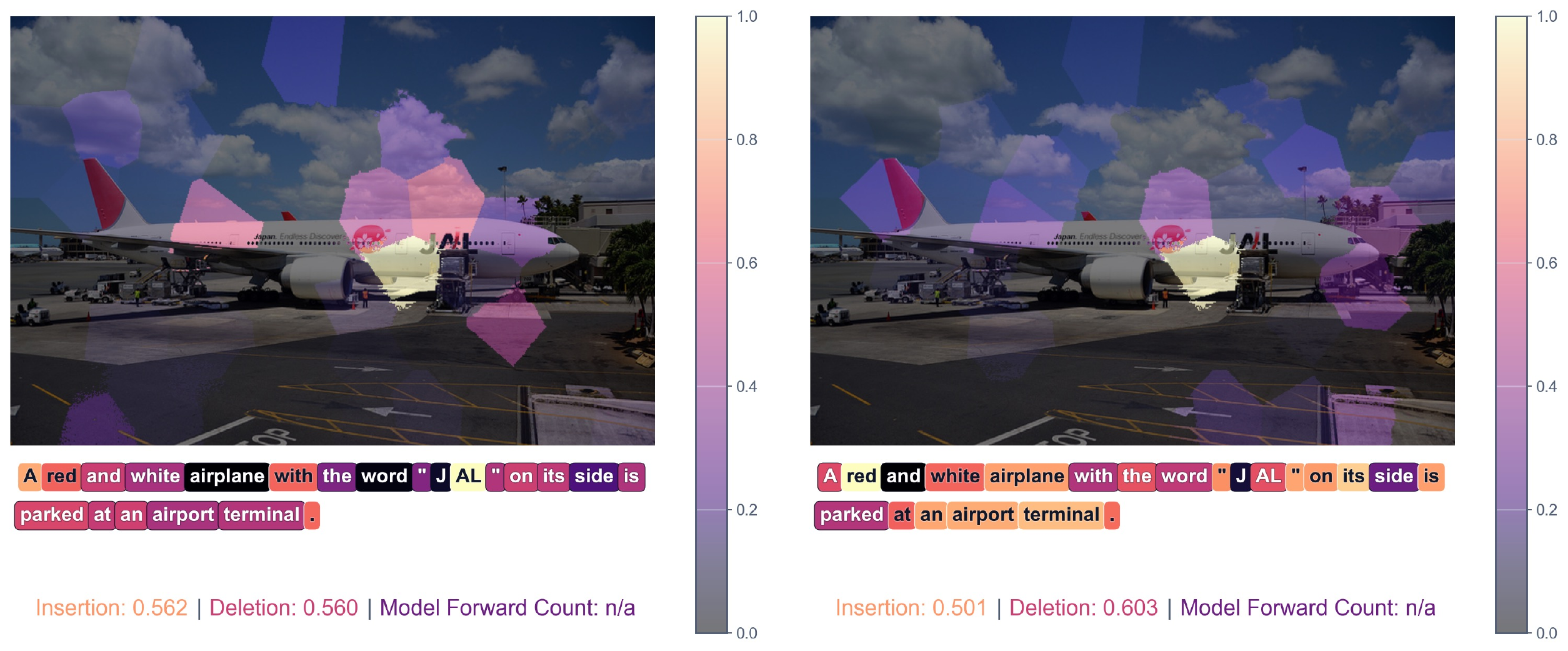}
    \caption{Additional image-captioning attribution results on Qwen2.5-VL-3B. From left to right: Gradient and LLaVA-CAM.}
    \label{fig:app-cap-qwen25vl3b-grad-llavacam}
\end{figure}

\begin{figure}[t]
    \centering
    \includegraphics[width=0.95\linewidth]{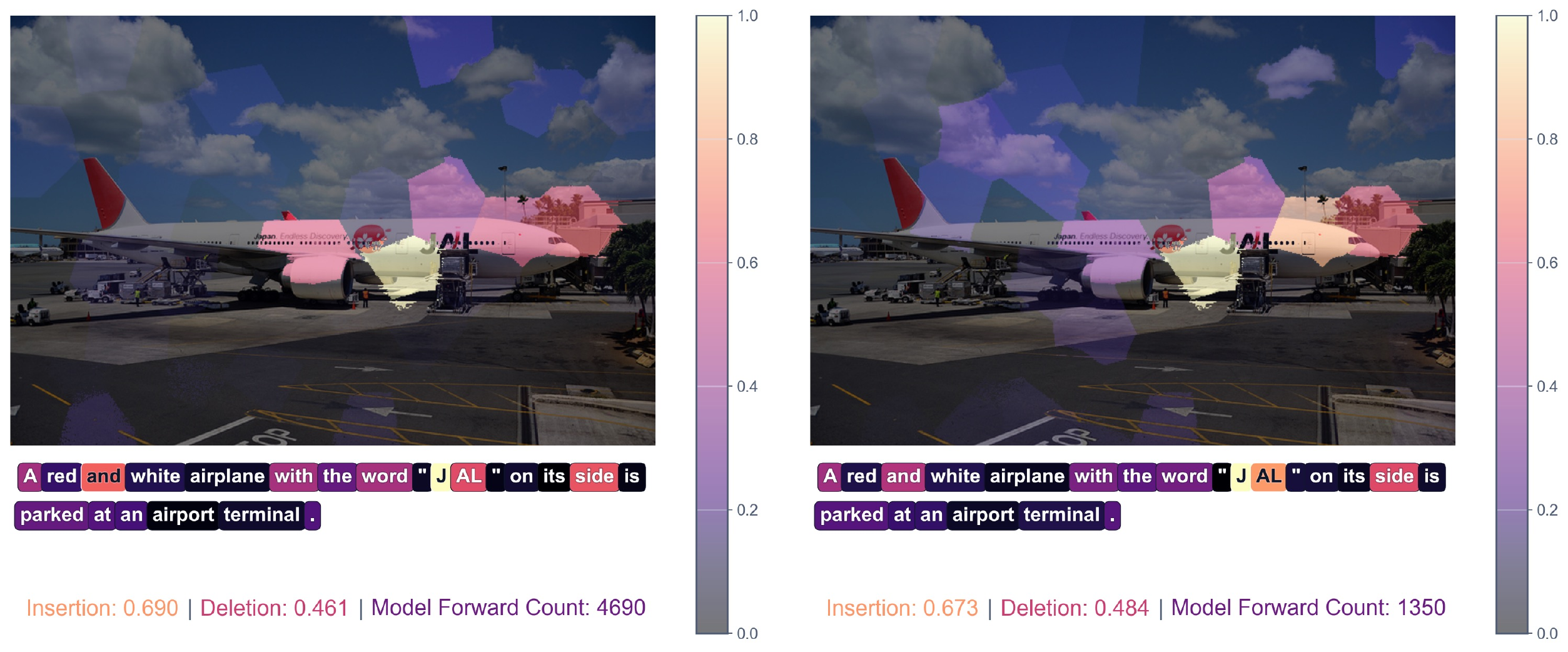}
    \caption{Additional image-captioning attribution results on Qwen2.5-VL-3B. From left to right: Greedy and PhaseWin.}
    \label{fig:app-cap-qwen25vl3b-greedy-phasewin}
\end{figure}
\end{document}